\documentclass[a4,11pt]{article}

\usepackage{times}
\usepackage{mathptmx}
\usepackage{latexsym}
\usepackage{amsmath} 
\usepackage{amssymb}  

\usepackage{dsfont} 

\usepackage{url} 
\usepackage{epsf}
\usepackage{xspace}	
\usepackage{pgf}

\usepackage{amsthm} 
\newtheorem{theorem}{Theorem}

\newtheorem{corollary}[theorem]{Corollary}
\newtheorem{lemma}[theorem]{Lemma}

\newtheorem{definition}{Definition}

\usepackage{amsopn}

\DeclareMathOperator*{\argmin}{arg\,min}

\newcommand{\m}[1]{\ensuremath{\overline{#1}}}

\definecolor{red}{rgb}{0.8,0.2,0.2}
\definecolor{blue}{rgb}{0,0,0.5}
\definecolor{green}{rgb}{0,0.7,0}
\definecolor{violet}{rgb}{0.5,0.2,0.5}
\definecolor{orange}{rgb}{0.8,0.5,0.2}
\definecolor{meu}{rgb}{0.6,0.8,0.6}

\usepackage[colorlinks,pagebackref,linktocpage]{hyperref}

\newsavebox{\fmbox}
\newenvironment{fmpage}[1]
{\begin{lrbox}{\fmbox}\begin{minipage}{#1}}
{\end{minipage}\end{lrbox}\fbox{\usebox{\fmbox}}}

\newcommand{\xaxis}{$x$-axis\xspace}
\newcommand{\yaxis}{$y$-axis\xspace}
\newcommand{\auc}{\ensuremath{\mathit{AUC}}\xspace}

\newcommand{\auch}{\ensuremath{\mathit{AUCH}}\xspace}
\newcommand{\bs}{\ensuremath{\mathit{BS}}\xspace}
\newcommand{\mae}{\ensuremath{\mathit{MAE}}\xspace}
\newcommand{\acc}{\ensuremath{\mathit{Acc}}\xspace}
\newcommand{\mbs}{\ensuremath{\mathit{MBS}}\xspace}
\newcommand{\mmae}{\ensuremath{\mathit{MMAE}}\xspace}
\newcommand{\micro}{\ensuremath{\mu}}
\newcommand{\macro}{\ensuremath{M}}
\newcommand{\macc}{\ensuremath{\mathit{MAcc}}\xspace}

\newcommand{\sk}{z}

\newcommand{\Tcost}{T_c}
\newcommand{\Tsk}{T_\sk}

\newcommand{\Qcost}{Q_c}
\newcommand{\Qsk}{Q_\sk}

\newcommand{\wcost}{w_c}
\newcommand{\wsk}{w_\sk}

\newcommand{\Thresholdchoice}{Threshold choice\xspace}  
\newcommand{\thresholdchoice}{threshold choice\xspace}  
\newcommand{\thresholdchoices}{threshold choices\xspace}  

\newcommand{\Thresholdchoicemethod}{\Thresholdchoice method\xspace}  
\newcommand{\Thresholdchoicemethods}{\Thresholdchoice methods\xspace}  
\newcommand{\thresholdchoicemethod}{\thresholdchoice method\xspace}  
\newcommand{\thresholdchoicemethods}{\thresholdchoice methods\xspace}  

\newcommand{\optimal}{optimal\xspace}  
\newcommand{\scorefixed}{score-fixed\xspace}  
\newcommand{\scoredriven}{score-driven\xspace}  
\newcommand{\scoreuniform}{score-uniform\xspace}  %
\newcommand{\ratefixed}{rate-fixed\xspace}  
\newcommand{\ratedriven}{rate-driven\xspace}  
\newcommand{\rateuniform}{rate-uniform\xspace}  

\newcommand{\Scoredriven}{Score-driven\xspace}  
\newcommand{\Ratedriven}{Rate-driven\xspace}  

\newcommand{\performancemetric}{performance metric\xspace}

\newcommand{\performancemetrics}{performance metrics\xspace}

\newcommand{\model}{model\xspace}
\newcommand{\classifier}{classifier\xspace}
\renewcommand{\models}{models\xspace}
\newcommand{\classifiers}{classifiers\xspace}

\newcommand{\To}{{T^o}}  
\newcommand{\Tsf}{{T^{sf}}}  
\newcommand{\Tsd}{{T^{sd}}}  
\newcommand{\Tsu}{{T^{su}}}  
\newcommand{\Trf}{{T^{rf}}}  
\newcommand{\Trd}{{T^{rd}}}  
\newcommand{\Tru}{{T^{ru}}}  
\newcommand{\Tcosto}{{T^o_c}}  
\newcommand{\Tcostsf}{{T^{sf}_c}}  
\newcommand{\Tcostsd}{{T^{sd}_c}}  
\newcommand{\Tcostsu}{{T^{su}_c}}  %
\newcommand{\Tcostrf}{{T^{rf}_c}}  
\newcommand{\Tcostrd}{{T^{rd}_c}}  
\newcommand{\Tcostru}{{T^{ru}_c}}  
\newcommand{\Tsko}{{T^o_\sk}}  
\newcommand{\Tsksf}{{T^{sf}_\sk}}  
\newcommand{\Tsksd}{{T^{sd}_\sk}}  
\newcommand{\Tsksu}{{T^{su}_\sk}}  %
\newcommand{\Tskrf}{{T^{rf}_\sk}}  
\newcommand{\Tskrd}{{T^{rd}_\sk}}  
\newcommand{\Tskru}{{T^{ru}_\sk}}  
\newcommand{\Lcost}{L_c}
\newcommand{\Lcosto}{L^o_c}
\newcommand{\LcostoU}{L^o_{U(c)}}

\newcommand{\Lcostsf}{{L^{sf}_c}}  
\newcommand{\LcostsfU}{{L^{sf}_{U(c)}}}  
\newcommand{\Lcostsd}{{L^{sd}_c}}  
\newcommand{\LcostsdU}{{L^{sd}_{U(c)}}}  
\newcommand{\LcostsuU}{{L^{su}_{U(c)}}}  
\newcommand{\Lcostrf}{{L^{rf}_c}}  
\newcommand{\LcostrfU}{{L^{rf}_{U(c)}}}  %
\newcommand{\Lcostrd}{{L^{rd}_c}}  
\newcommand{\LcostrdU}{{L^{rd}_{U(c)}}}  
\newcommand{\Lcostru}{{L^{ru}_c}}  
\newcommand{\LcostruU}{{L^{ru}_{U(c)}}}  
\newcommand{\Lsk}{L_\sk}

\newcommand{\LskoU}{L^o_{U(\sk)}}

\newcommand{\Lsksf}{{L^{sf}_\sk}}  
\newcommand{\LsksfU}{{L^{sf}_{U(\sk)}}}  
\newcommand{\LsksdU}{{L^{sd}_{U(\sk)}}}  
\newcommand{\LsksuU}{{L^{su}_{U(\sk)}}}  
\newcommand{\LskrdU}{{L^{rd}_{U(\sk)}}}  
\newcommand{\LskruU}{{L^{ru}_{U(\sk)}}}  

\newcommand{\ra}{r}  
\newcommand{\rate}{rate\xspace}  
\newcommand{\rates}{rates\xspace}  

\long\def\comment#1{}

\newcommand{\Conv}{\ensuremath{\mathrm{Conv}}\xspace}
\newcommand{\Cal}{\ensuremath{\mathrm{Cal}}\xspace}
\newcommand{\Even}{\ensuremath{\mathrm{Even}}\xspace}
\newcommand{\Evend}{\ensuremath{\mathrm{EST}}\xspace}

\newcommand{\rl}{\ensuremath{\mathit{RL}}\xspace}
\newcommand{\cl}{\ensuremath{\mathit{CL}}\xspace}

\begin{document}


\begin{titlepage}

\begin{center}

\vspace*{1cm}

\newcommand{\HRule}{\rule{\linewidth}{0.5mm}}

\HRule \\[0.7cm]
{ \huge \bfseries Threshold Choice Methods: the Missing Link}\\[0.4cm]

\HRule \\[1.2cm]

\begin{minipage}{0.75\textwidth}
\begin{flushleft} 
Jos\'e Hern\'andez-Orallo\hfill(jorallo@dsic.upv.es)\\
Departament de Sistemes Inform\`atics i Computaci\'o\\
Universitat Polit\`ecnica de Val\`encia, Spain\\[12pt]

Peter Flach\hfill(Peter.Flach@bristol.ac.uk)\\
Intelligent Systems Laboratory\\
University of Bristol, United Kingdom\\[12pt]

C\`esar Ferri\hfill(cferri@dsic.upv.es)\\
Departament de Sistemes Inform\`atics i Computaci\'o\\
Universitat Polit\`ecnica de Val\`encia, Spain\\[12pt]

\end{flushleft}
\end{minipage}

\vspace{1cm}

{\large \today}

\vfill
\end{center}


\begin{abstract}%
Many \performancemetrics have been introduced in the literature for the evaluation of classification performance, each of them with different origins and areas of application. These metrics include accuracy, macro-accuracy, area under the ROC curve or the ROC convex hull, the mean absolute error and the Brier score or mean squared error (with its decomposition into refinement and calibration). 
One way of understanding the relation among these metrics is by means of variable operating conditions (in the form of misclassification costs and/or class distributions). Thus, a metric may correspond to some expected loss over different operating conditions. One dimension for the analysis has been the distribution for this range of operating conditions, leading to some important connections in the area of proper scoring rules. 
We demonstrate in this paper that there is an equally important dimension which has so far not received attention in the analysis of \performancemetrics. This new dimension is given by the decision rule, which is typically implemented as a {\em \thresholdchoicemethod} when using scoring models. In this paper, we explore many old and new \thresholdchoicemethods: fixed, \scoreuniform, \scoredriven,  \ratedriven and  \optimal, among others. By calculating the expected loss obtained with these \thresholdchoicemethods for a uniform range of operating conditions we give clear interpretations of the 0-1 loss, the absolute error, the Brier score, the \auc and the refinement loss respectively. 
Our analysis provides a comprehensive view of \performancemetrics as well as a systematic approach to loss minimisation which can be summarised as follows: given a model, apply the \thresholdchoicemethods that correspond with the available information  about the operating condition, and compare their expected losses. In order to assist in this procedure we also derive several connections between the aforementioned \performancemetrics, and we highlight the role of calibration in choosing the \thresholdchoicemethod.

\vspace{0.3cm}

\end{abstract}

\end{titlepage}



\newpage
\tableofcontents

\newpage

\section{Introduction}\label{sec:intro}


The choice of a proper \performancemetric for evaluating classification \cite{hand1997construction} is an old but still lively debate which has incorporated many different \performancemetrics along the way. Besides accuracy (\acc, or, equivalently, the error rate or 0-1 loss), many other \performancemetrics have been studied. The most prominent and well-known metrics are the Brier Score (\bs, also known as Mean Squared Error) \cite{brier1950verification} and its decomposition in terms of refinement and calibration \cite{murphy1973new}, the absolute error (\mae), the log(arithmic) loss (or cross-entropy) \cite{Good52} and the area under the ROC curve (\auc, also known as the Wilcoxon-Mann-Whitney statistic, proportional to the Gini coefficient and to the Kendall's tau distance to a perfect \model) \cite{SDM00,Fawcett06}. There are also many graphical representations and tools for \model evaluation, such as ROC curves \cite{SDM00,Fawcett06}, ROC isometrics \cite{Fla03}, cost curves \cite{DH00,drummond-and-Holte2006}, DET curves \cite{martin1997det}, lift charts \cite{piatetsky1999estimating}, calibration maps \cite{cohen2004properties}, etc. A survey of graphical methods for classification predictive performance evaluation can be found in \cite{Prati2011}.

When we have a clear operating condition which establishes the misclassification costs and the class distributions, there are effective tools such as ROC analysis \cite{SDM00,Fawcett06} to establish which \model is best and what its expected loss will be. However, the question is more difficult in the general case when we do not have information about the operating condition where the \model will be applied. In this case, we want our \models to perform well in a wide range of operating conditions.
In this context, the notion of `proper scoring rule', see e.g. \cite{murphy1970scoring}, sheds some light on some \performancemetrics. 
Some proper scoring rules, such as the Brier Score (MSE loss), the logloss, boosting loss and error rate (0-1 loss) have been shown in \cite{BSS05} to be special cases of an integral over a Beta density of costs,  
see e.g. \cite{gneiting2007strictly,reid2010composite,reid2011information,brummer2010thesis}. Each performance metric is derived as a special case of the Beta distribution.
However, this analysis focusses on scoring rules which are `proper', i.e., metrics that are minimised for well-calibrated probability assessments or, in other words, get the best (lowest) score by forecasting the true beliefs.
Much less is known (in terms of expected loss for varying distributions) about other \performancemetrics which are non-proper scoring rules, such as \auc. Moreover, even its role as a classification \performancemetric has been put into question \cite{hand2009measuring}.


All these approaches make some (generally implicit and poorly understood) assumptions on how the \model will work for each operating condition. In particular, it is generally assumed that the threshold which is used to discriminate between the classes will be set according to the operating condition. In addition, it is assumed that the threshold will be set in such a way that the estimated probability where the threshold is set is made equal to the operating condition. This is natural if we focus on proper scoring rules. Once all this is settled and fixed, different performance metrics represent different expected losses by using the distribution over the operating condition as a parameter. However, this {\em threshold choice} is only one of the many possibilities.

In our work we make these assumptions explicit through the concept of a \emph{\thresholdchoicemethod,} which we argue forms the `missing link' between a \performancemetric and expected loss. 
A \thresholdchoicemethod sets a single threshold on the scores of a \model in order to arrive at classifications, possibly taking circumstances in the deployment context into account, such as the operating condition (the class or cost distribution) or the intended proportion of positive predictions (the predicted positive rate). 
Building on this new notion of \thresholdchoicemethod, we are able to systematically explore how known \performancemetrics are linked to expected loss, resulting in a range of results that are not only theoretically well-founded but also practically relevant.

The basic insight is the realisation that there are many ways of converting a \model (understood throughout this paper as a function assigning scores to instances) into a \classifier that maps instances to classes (we assume binary classification throughout). 
Put differently, there are many ways of setting the threshold given a \model and an operating point. We illustrate this with an example concerning a very common scenario in machine learning research. Consider two models $A$ and $B$, a naive Bayes \model and a decision tree respectively (induced from a training dataset), which are evaluated against a test dataset, producing a score distribution for the positive and negative classes as shown in Figure \ref{fig:models}. ROC curves of both models are shown in Figure \ref{fig:ROCmodels}. We will assume that at this {\em evaluation time} we do not have information about the operating condition, but we expect that this information will be available at {\em deployment time}.


\begin{figure}
\centering
\includegraphics[width=0.45\textwidth]{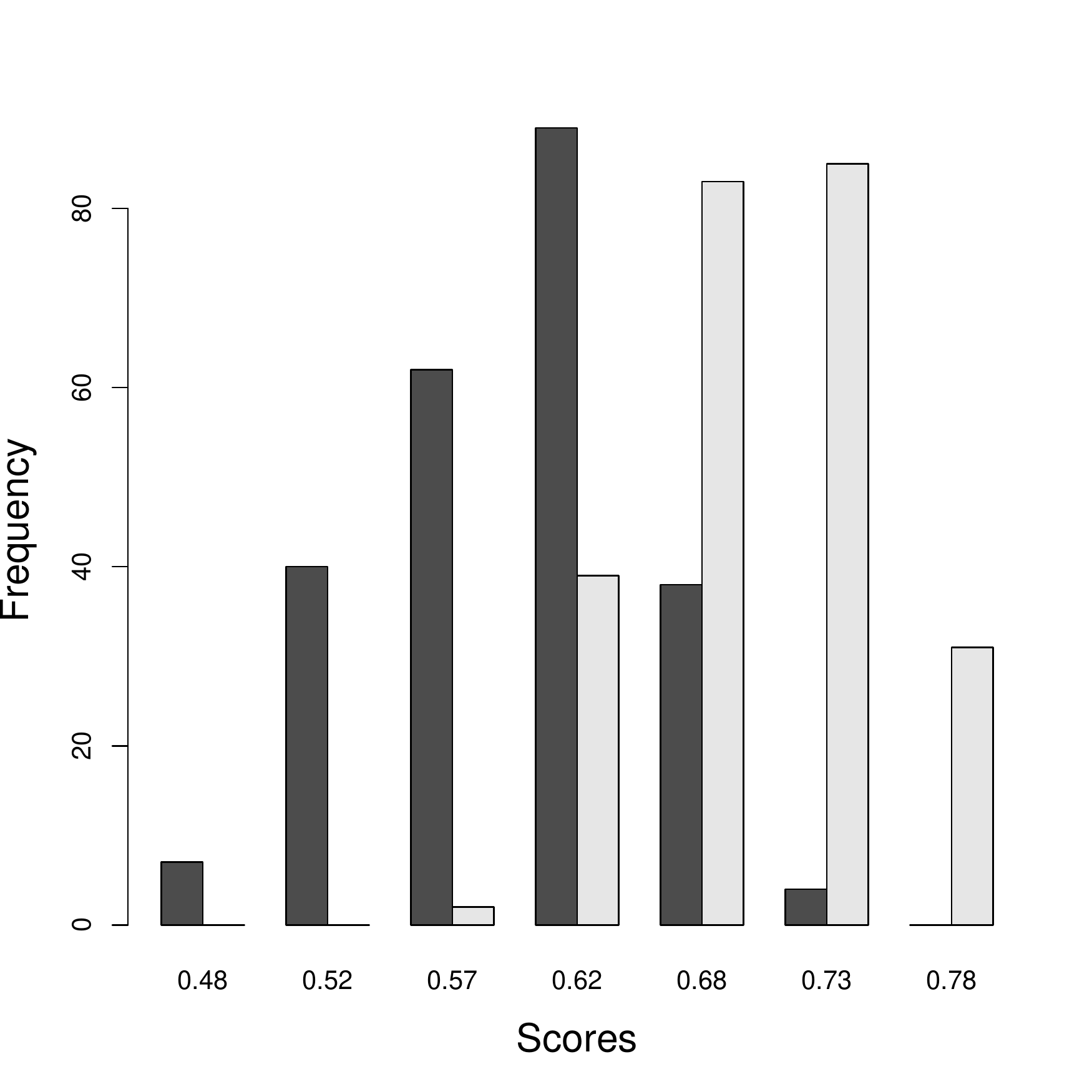} \hfill
\includegraphics[width=0.45\textwidth]{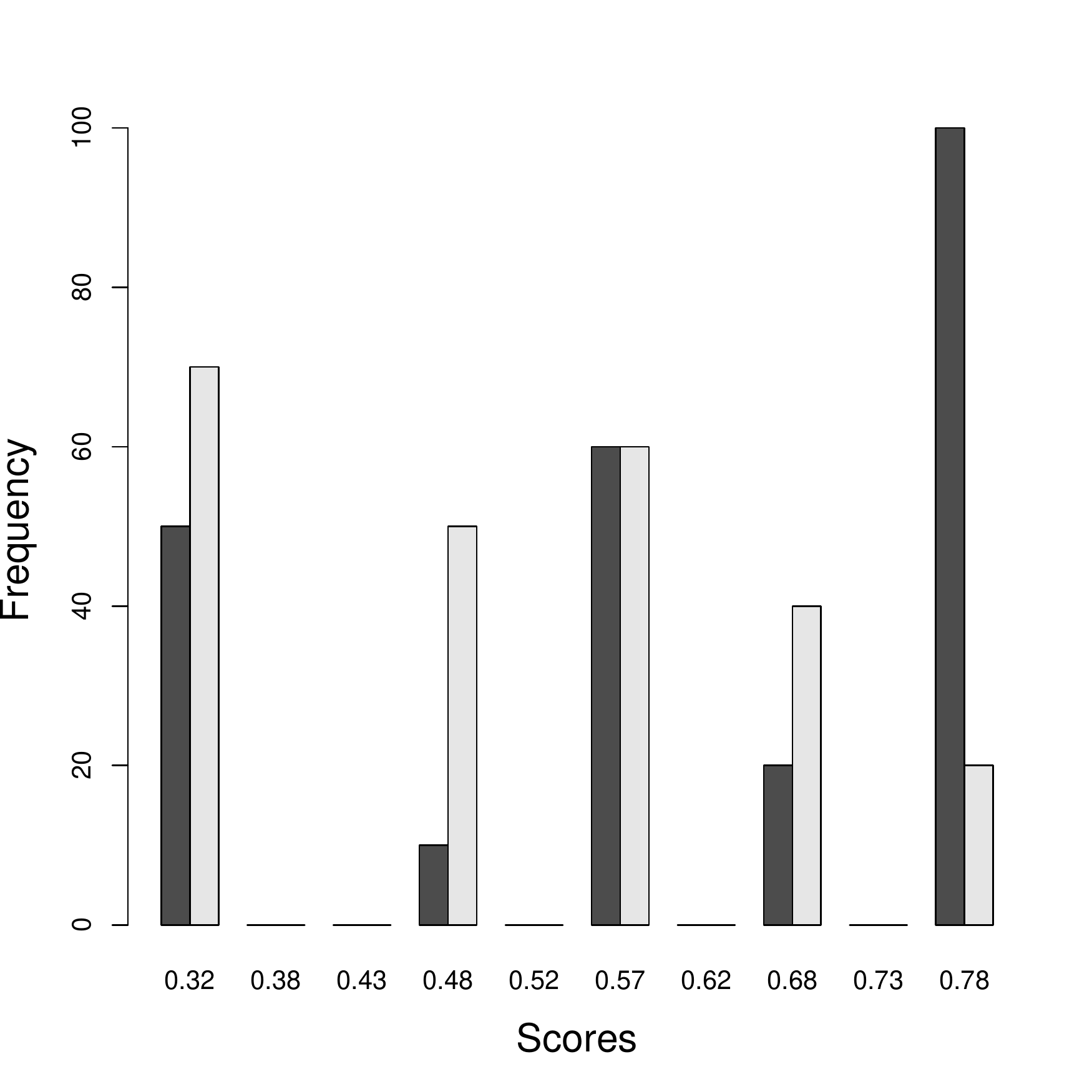}
\caption{Histograms of the score distribution for model $A$ (left) and model $B$ (right).}
\label{fig:models}
\end{figure}

\begin{figure}
\centering
\includegraphics[width=0.45\textwidth]{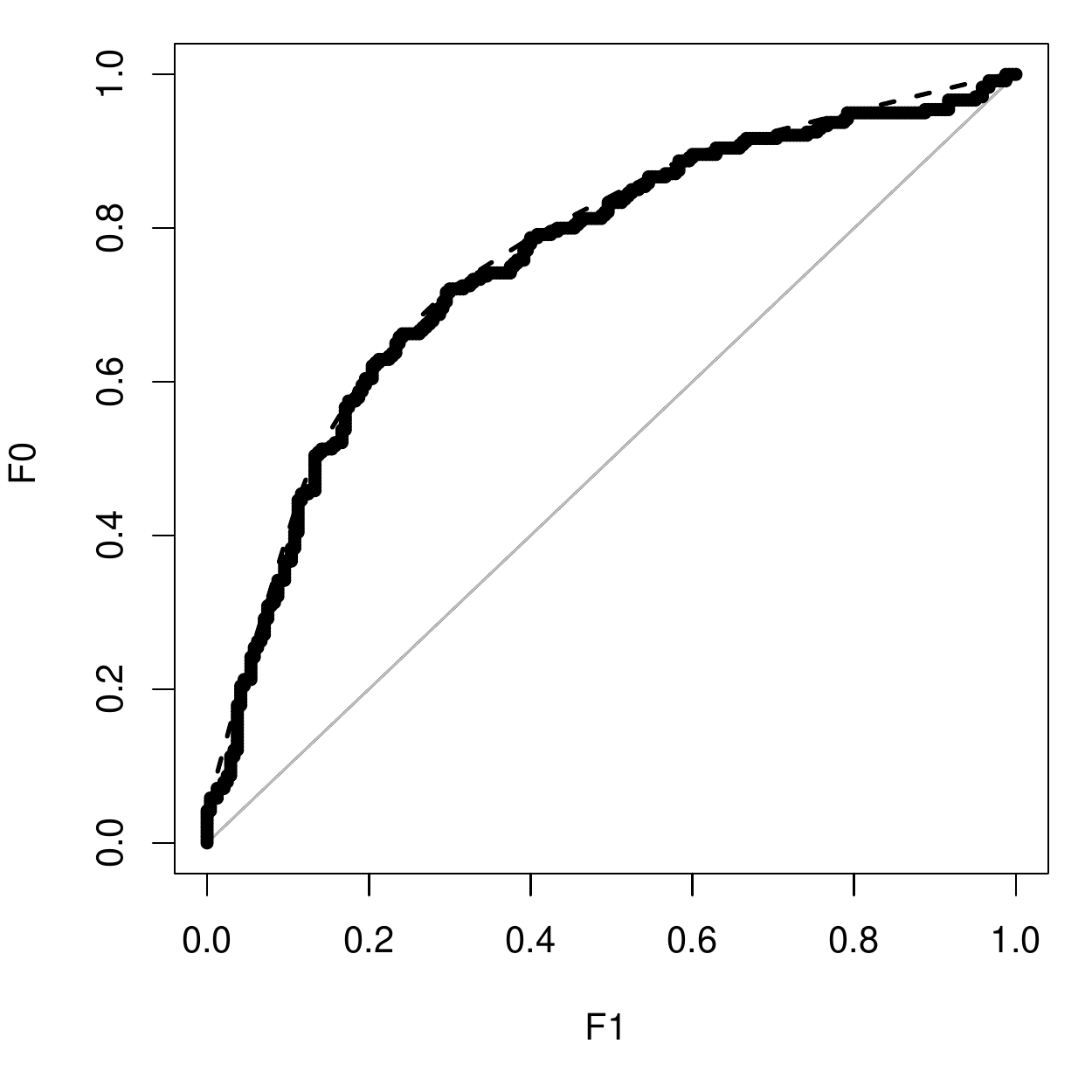} \hfill
\includegraphics[width=0.45\textwidth]{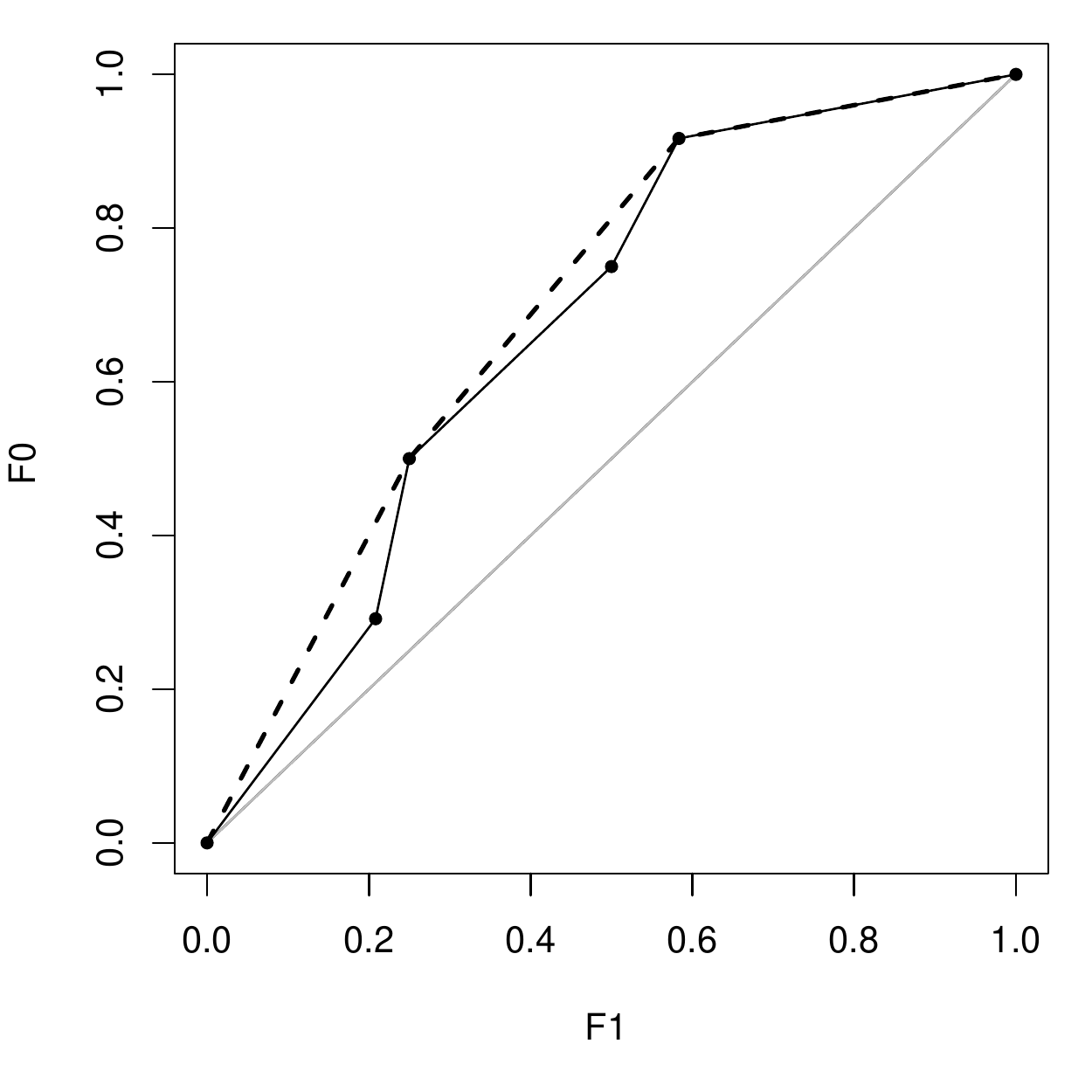}
\caption{ROC Curves for model $A$ (left) and model $B$ (right).}
\label{fig:ROCmodels}
\end{figure}
If we ask the question of which model is best we may rush to calculate its \auc and \bs (and perhaps other metrics), as given by Table \ref{tab:models1}.
However, we cannot give an answer because the question is {\em underspecified}. First, we need to know the range of operating conditions the model will work with. Second, we need to know how we will make the classifications, or in other words, we need a {\em decision rule}, which can be implemented as a {\em \thresholdchoicemethod} when the model outputs scores. For the first dimension (already considered by the work on proper scoring rules), if we have no knowledge about the operating conditions, we can assume a distribution, e.g., a uniform distribution, which considers all operating conditions equally likely. For the second (new) dimension, we have many options.


\begin{table}[htdp]
\begin{center}
\begin{tabular}{cccc}\hline performance metric & model $A$ & model $B$ \\ \hline 
\auc & 0.79 & 0.67  \\
Brier score & 0.33 & 0.24  \\
\hline\end{tabular}
\end{center}
\caption{Results from two models on a data set.}
\label{tab:models1}
\end{table}%

For instance, we can just set a fixed threshold at 0.5. This is what naive Bayes and decision trees do by default. This decision rule works as follows: if the score is greater than 0.5 then predict positive, otherwise predict negative. With this precise decision rule, we can now ask the question about the expected loss. Assuming a uniform distribution for operating conditions, we can effectively calculate the answer on the dataset: $0.51$.

But we can use better decision rules. We can use decision rules which adapt to the operating condition. One of these decision rules is the \scoredriven \thresholdchoicemethod, which sets the threshold equal to the operating condition or, more precisely, to a cost proportion $c$. Another decision rule is the \ratedriven \thresholdchoicemethod, which sets the threshold in such a way that the proportion of predicted positives (or predicted positive rate), simply known as `rate' and denoted by $r$, equals the operating condition.
Using these three different \thresholdchoicemethods for the models $A$ and $B$ we get the expected losses shown in Table \ref{tab:models2}.

\begin{table}[htdp]
\begin{center}
\begin{tabular}{cccc}\hline \thresholdchoicemethod & expected loss model $A$ & expected loss model $B$ \\ \hline 
Fixed ($T=0.5$)              & 0.510       & 0.375 \\
\Scoredriven ($T=c$)        & 0.328     & 0.231 \\   
\Ratedriven ($T$ s.t. $r=c$) & 0.188 & 0.248 \\   
\hline\end{tabular}
\end{center}
\caption{Extension of Table \ref{tab:models1} where two models are applied with three different \thresholdchoicemethods each, leading to six different \classifiers and corresponding expected losses. In all cases, the expected loss is calculated over a range of cost proportions (operating conditions), which is assumed to be uniformly distributed. We denote the threshold by $T$, the cost proportion by $c$ and the predicted positive rate by $r$). 
 }
\label{tab:models2}
\end{table}%

In other words, only when we specify or assume a \thresholdchoicemethod can we convert a \model into a \classifier for which it makes sense to consider its expected loss. In fact, as we can see in Table \ref{tab:models2}, very different expected losses are obtained for the same model with different \thresholdchoicemethods. And this is the case even assuming the same uniform cost distribution for all of them.

Once we have made this (new) dimension explicit, we are ready to ask new questions. How many \thresholdchoicemethods are there? Table \ref{tab:thresholdchoicemethods} shows six of the \thresholdchoicemethods we will analyse in this work, along with their notation. Only the \scorefixed and the \scoredriven methods have been analysed in previous works in the area of proper scoring rules. In addition, a seventh \thresholdchoicemethod, known as \optimal \thresholdchoicemethod, denoted by $\To$, has been (implicitly) used in a few works \cite{DH00,drummond-and-Holte2006,hand2009measuring}. 

\begin{table}[htdp]
\begin{center}
		\begin{tabular}{c|ccc}	 
			\Thresholdchoicemethod              & Fixed & Chosen uniformly & Driven by o.c.  \\\hline
			
			Using scores                          & \scorefixed ($\Tsf$)                              & \scoreuniform ($\Tsu$)         & \scoredriven ($\Tsd$)          \\
      Using \rates                          & \ratefixed ($\Trf$)                               & \rateuniform ($\Tru$)          & \ratedriven ($\Trd$)           \\		
			\hline
		\end{tabular}
\end{center}
	\caption{Non-\optimal \thresholdchoicemethods. The first family uses scores (as they were probabilities) and the second family uses rates (using scores as rank indicators). For both families we can fix a threshold or assume them ranging uniformly, which makes the \thresholdchoicemethod independent from the operating condition. Only the last column takes the operating condition (o.c.) into account, and hence are the most interesting \thresholdchoicemethods.  }
	\label{tab:thresholdchoicemethods}
\end{table}

We will see that each \thresholdchoicemethod is linked to a specific \performancemetric.
This means that if we decide (or are forced) to use a \thresholdchoicemethod then there is a recommended \performancemetric for it. The results in this paper show that accuracy is the appropriate \performancemetric for the \scorefixed method, \mae fits the \scoreuniform method, \bs is the appropriate \performancemetric for the \scoredriven method, and \auc fits both the \rateuniform and the \ratedriven methods. All these results assume a uniform cost distribution. 

The good news is that inter-comparisons are still possible: given a \thresholdchoicemethod we can calculate expected loss from the relevant \performancemetric. 
The results in Table \ref{tab:models2} allow us to conclude that model $A$ achieves the lowest expected loss for uniformly sampled cost proportions, {\em if} we are wise enough to choose the appropriate \thresholdchoicemethod (in this case the \ratedriven method) to turn \model $A$ into a successful \classifier. Notice that this cannot be said by just looking at Table \ref{tab:models1} because the metrics in this table are not comparable to each other. In fact, there is no single \performancemetric that ranks the models in the correct order, because, as already said, expected loss cannot be calculated for \models, only for \classifiers.

\subsection{Contributions and structure of the paper}

The contributions of this paper to the subject of \model evaluation for classification can be summarised as follows.  

\begin{enumerate}
\item The expected loss of a model can only be determined if we select a distribution of operating conditions and a \thresholdchoicemethod. We need to set a point in this two-dimensional space. Along the second (usually neglected) dimension, several new \thresholdchoicemethods are introduced in this paper.
\item We answer the question: ``if one is choosing thresholds in a particular way, which \performancemetric is appropriate?" by giving an explicit expression for the expected loss for each \thresholdchoicemethod. We derive linear relationships between expected loss and many common \performancemetrics. The most remarkable one is the vindication of \auc as a measure of expected classification loss for both the \rateuniform and \ratedriven methods, contrary to recent claims in the literature \cite{hand2009measuring}. 
\item One fundamental and novel result shows that the refinement loss of the convex hull of a ROC curve is equal to expected \emph{optimal} loss as measured by the area under the optimal cost curve. This sets an optimistic (but also unrealistic) bound for the expected loss.
\item Conversely, from the usual calculation of several well-known \performancemetrics we can derive expected loss. Thus, \classifiers and performance metrics become easily comparable. With this we do not choose the best \model (a concept that does not make sense) but we choose the best \classifier (a \model with a particular \thresholdchoicemethod).
\item By cleverly manipulating scores we can connect several of these \performancemetrics, either by the notion of evenly-spaced scores or perfectly calibrated scores. This provides an additional way of analysing the relation between \performancemetrics and, of course, \thresholdchoicemethods.
\item We use all these connections to better understand which \thresholdchoicemethod should be used, and in which cases some are better than others. The analysis of calibration plays a central role in this understanding, and also shows that non-proper scoring rules do have their role and can lead to lower expected loss than proper scoring rules, which are, as expected, more appropriate when the model is well-calibrated.
\end{enumerate}

\noindent This set of contributions provides an integrated perspective on \performancemetrics for classification around the `missing link' which we develop in this paper: the notion of \thresholdchoicemethod. 

The remainder of the paper is structured as follows.
Section \ref{sec:background} 
introduces some notation, the basic definitions for operating condition, threshold, expected loss, and particularly the notion of \thresholdchoicemethod, which we will use throughout the paper.
Section \ref{sec:fixed} investigates expected loss for fixed \thresholdchoicemethods (\scorefixed and
\ratefixed), which are the base for the rest. We show that, not surprisingly, the expected loss for these \thresholdchoicemethod are the 0-1 loss (accuracy or macro-accuracy depending on whether we use cost proportions or skews). 
%
Section \ref{sec:scores} presents the results that the \scoreuniform \thresholdchoicemethod has \mae as associate performance metric and the \scoredriven \thresholdchoicemethod leads to the Brier score. We also show that one dominates over the other.
Section \ref{sec:rates} analyses the non-fixed methods based on \rates. Somewhat surprisingly, both the \rateuniform \thresholdchoicemethod and the \ratedriven \thresholdchoicemethod lead to linear functions of \auc, with the latter always been better than the former. All this vindicates the \ratedriven  \thresholdchoicemethod but also \auc as a performance metric for classification. 
Section \ref{sec:optimal} uses the \optimal \thresholdchoicemethod, connects the expected loss in this case with the area under the \optimal cost curve, and derives its corresponding metric, which is refinement loss, one of the components of the Brier score decomposition.
%
%
Section \ref{sec:calibration} analyses the connections between the previous \thresholdchoicemethods and metrics by considering several properties of the scores: evenly-spaced scores and perfectly calibrated scores. This also helps to understand which \thresholdchoicemethod should be used depending on how good scores are. 
Finally, Section \ref{sec:conclusion} closes the paper with a thorough discussion of results, related work, and an overall conclusion with future work and open questions.
There is an appendix which includes some technical results for the \optimal \thresholdchoicemethod and some examples.



\section{Background}\label{sec:background}



In this section we introduce some basic notation and definitions we will need throughout the paper. Some other definitions will be delayed and introduced when needed. The most important definitions we will need are introduced below: the notion of \thresholdchoicemethod and the expression of expected loss.

\subsection{Notation and basic definitions}\label{sec:notation}


A \emph{\classifier} is a function that maps instances $x$ from an instance space $X$ to classes $y$ from an output space $Y$. 
For this paper we will assume binary classifiers, i.e., $Y = \{0, 1\}$. 
A \emph{\model} is a function $m:X \rightarrow \mathbb{R}$ that maps examples to real numbers (scores) on an unspecified scale. 
We use the convention that higher scores express a stronger belief that the instance is of class 1. 
A \emph{probabilistic \model} is a function $m:X \rightarrow [0,1]$ that maps examples to estimates $\hat{p}(1|x)$ of the probability of example $x$ to be of class $1$.
In order to make predictions in the $Y$ domain, a \model can be converted to a \classifier by fixing a decision threshold $t$ on the scores. Given a predicted score $s=m(x)$, the instance $x$ is classified in class $1$ if $s > t$, and in class $0$ otherwise. 

For a given, unspecified \model and population from which data are drawn, we denote the score density for class $k$ by $f_k$ and the cumulative distribution function by $F_k$. Thus, 
$F_{0}(t) = \int_{-\infty}^{t} f_{0}(s) ds = P(s\leq t|0)$ 
is the proportion of class 0 points correctly classified if the decision threshold is $t$, which is the sensitivity or true positive rate at $t$. Similarly, 
$F_{1}(t) = \int_{-\infty}^{t} f_{1}(s) ds = P(s\leq t|1)$ 
is the proportion of class 1 points incorrectly classified as 0 or the false positive rate at threshold $t$; $1-F_{1}(t)$ is the true negative rate or specificity.
Note that we use 0 for the positive class and 1 for the negative class, but scores increase with $\hat{p}(1|x)$. That is, $F_{0}(t)$ and $F_{1}(t)$ are monotonically non-decreasing with $t$. This has some notational advantages and is the same convention as used by, e.g., Hand \cite{hand2009measuring}.

Given a dataset $D \subset \langle X, Y\rangle$ of size $n = |D|$, we denote by $D_k$ the subset of examples in class $k \in \{0,1\}$,
and set $n_k = |D_k|$ and $\pi_k = n_k / n$. Clearly $\pi_0 + \pi_1 = 1$. We will use the term \emph{class proportion} for $\pi_0$ (other terms such as `class ratio' or `class prior' have been used in the literature).
Given a \model and a threshold $t$, we denote by $R(t)$ the predicted positive rate, i.e., the proportion of examples that will be predicted positive (class 0) is threshold is set at $t$. 
This can also be defined as $R(t) = \pi_0F_0(t) + \pi_1F_1(t)$.
The average score of actual class $k$ is 
$\m{s}_k= \int_{0}^{1}  s f_k(s)ds$.
Given any strict order for a dataset of $n$ examples we will use the index $i$ on that order to refer to the $i$-th example. Thus, $s_i$ denotes the score of the $i$-th example and $y_i$ its true class.  


We define partial class accuracies as $Acc_0(t) = F_0(t)$ and $Acc_1(t) = 1-F_1(t)$. From here, (micro-average) accuracy is defined as $Acc(t) = \pi_0 Acc_0(t) + \pi_1 Acc_1(t)$ and macro-average accuracy ${\macro}Acc(t) = (Acc_0(t) + Acc_1(t))/2$.

We denote by $U_S(x)$ the continuous uniform distribution of variable $x$ over an interval $S \subset {\mathds{R}}$. If this interval $S$ is $[0,1]$ then $S$ can be omitted. 
The family of continuous distributions Beta is denoted by $B_{\alpha,\beta}$. The Beta distributions are always defined in the interval $[0,1]$. Note that the continuous distribution is a special case of the Beta family, i.e., $B_{1,1}=U$. 

\subsection{Operating conditions and expected loss}\label{sec:operating}

When a \model is deployed for classification, the conditions might be different to those during training. In fact, a \model can be used in several deployment contexts, with different results. A context can entail different class distributions, different classification-related costs (either for the attributes, for the class or any other kind of cost), or some other details about the effects that the application of a model might entail and the severity of its errors.
In practice, a deployment context or \emph{operating condition} is usually defined by a misclassification cost function and a class distribution. 
Clearly, there is a difference between operating when the cost of misclassifying $0$ into $1$ is equal to the cost of misclassifying $1$ into $0$ and doing so when the former is ten times the latter. Similarly, operating when classes are balanced is different from when there is an overwhelming majority of instances of one class.

One general approach to cost-sensitive learning assumes that the cost does not depend on the example but only on its class. In this way, misclassification costs are usually simplified by means of cost matrices, where we can express that some misclassification costs are higher than others \cite{Elk01}. Typically, the costs of correct classifications are assumed
to be 0. This means that for binary \models we can describe the cost matrix by two values $c_k \geq 0$, representing the misclassification cost of an example of class $k$. Additionally, we can normalise the costs by setting $b = c_0 + c_1$ and $c = c_0 / b$; 
we will refer to $c$ as the \emph{cost proportion}. 
Since this can also be expressed as $c=(1 + c_1/c_0)^{-1}$, it is often called `cost ratio' even though, technically, it is a proportion ranging between $0$ and $1$. 

The loss which is produced at a decision threshold $t$ and a cost proportion $c$ is then given by the formula:
\begin{align}\label{eqQcost}
\Qcost(t; c) & \triangleq c_{0} \pi_0 (1 -F_0(t)) + c_{1} \pi_1 F_1(t) \\
 & = b\{c \pi_0 (1 -F_0(t)) + (1-c) \pi_1 F_1(t)\} \nonumber
\end{align}
%
This notation assumes the class distribution to be fixed. In order to take both class proportion and cost proportion into account we introduce the notion of \emph{skew}, which is a normalisation of their product:
\begin{align}\label{eqSkew}
\sk & \triangleq \frac{c_0\pi_0}{c_0\pi_0 + c_1\pi_1} = \frac{c\pi_0}{c\pi_0 + (1-c)(1-\pi_0)}
\end{align}
From equation~(\ref{eqQcost}) we obtain
\begin{align}\label{eqQsk}
\frac{\Qcost(t; c)}{c_0\pi_0 + c_1\pi_1} = \sk (1 -F_0(t)) + (1-\sk) F_1(t) & \triangleq \Qsk(t; \sk)
\end{align}
This gives an expression for loss at a threshold $t$ and a skew $\sk$. 
We will assume that the operating condition is either defined by the cost proportion (using a fixed class distribution) or by the skew.
We then have the following simple but useful result.
\begin{lemma}\label{lemma-balance}
If $\pi_0 = \pi_1$ then $z=c$ and $\Qsk(t; \sk) = {2 \over b} \Qcost(t; c)$. 
\end{lemma}
\begin{proof}
If classes are balanced we have $c_0\pi_0 + c_1\pi_1 = b/2$, and the result follows from Equation~(\ref{eqSkew}) and Equation~(\ref{eqQsk}). 
\end{proof}
\noindent This justifies taking $b=2$, which means that $\Qsk$ and $\Qcost$ are expressed on the same 0-1 scale, and are also commensurate with error rate which assumes $c_{0}=c_{1}=1$. 
The upshot of Lemma~\ref{lemma-balance} is that we can transfer any expression for loss in terms of cost proportion to an equivalent expression in terms of skew by just setting $\pi_0 = \pi_1=1/2$ and $z=c$. 
Notice that if $c_{0}=c_{1}=1$ then $\sk=\pi_{0}$, so in that case skew denotes the class distribution as operating condition. 



It is important to distinguish the information we may have available at each stage of the process. At {evaluation time} we may not have some information that is available later, at {deployment time}. In many real-world problems, when we have to evaluate or compare \models, we do not know the cost proportion or skew that will apply during deployment. One general approach is to evaluate the \model on a range of possible operating points. In order to do this, we have to set a weight or distribution on cost proportions or skews. In this paper, we will mostly consider the continuous uniform distribution $U$ (but other distribution families, such as the Beta distribution could be used).

A key issue when applying a \model under different operating conditions is how the threshold is chosen in each of them. If we work with a \classifier, this question vanishes, since the threshold is already settled. However, in the general case when we work with a \model, we have to decide how to establish the threshold. The key idea proposed in this paper is the notion of a \thresholdchoicemethod, a function which converts an operating condition 
into an appropriate threshold for the classifier. 
\begin{definition}
{\bf \Thresholdchoicemethod}. 
A \thresholdchoicemethod is a (possibly non-deterministic) function $T:[0,1] \rightarrow \mathds{R}$ such that given an operating condition it returns a decision threshold.
The operating condition can be either a skew $\sk$ or a cost proportion $c$; to differentiate these we use the subscript $\sk$ or $c$ on $T$. 
Superscripts are used to identify particular threshold choice methods. 
Some \thresholdchoicemethods we consider in this paper take additional information into account, such as a default threshold or a target predicted positive rate; such information is indicated by square brackets. 
So, for example, the score-fixed \thresholdchoicemethod for cost proportions considered in the next section is indicated thus: 
$\Tcostsf[t](c)$. 
\end{definition}
\noindent When we say that $T$ may be non-deterministic, it means that the result may depend on a random variable and hence may itself be a random variable according to some distribution. 

We introduce the \thresholdchoicemethod as an abstract concept since there are several reasonable options for the function $T$,
essentially because there may be different degrees of information about the model and the operating conditions at evaluation time. We can set a fixed threshold ignoring the operating condition; we can set the threshold by looking at the ROC curve (or its convex hull) and using the cost proportion or the skew to intersect the ROC curve (as ROC analysis does); we can set a threshold looking at the estimated scores, especially when they represent probabilities; or we can set a threshold independently from the rank or the scores. The way in which we set the threshold may dramatically affect performance. But, not less importantly, the \performancemetric used for evaluation must be in accordance with the \thresholdchoicemethod.

In the rest of this paper, we explore a range of different methods to choose the threshold (some deterministic and some non-deterministic).
We will give proper definitions of all these \thresholdchoicemethods in its due section.

Given a threshold choice function $\Tcost$, the loss for a particular cost proportion is given by $\Qcost(\Tcost(c); c)$. 
Following Adams and Hand \cite{AdamsHand1999} we define expected loss as a weighted average over operating conditions. 
\begin{definition}\label{def:Lcost}
Given a \thresholdchoicemethod for cost proportions $\Tcost$ and a probability density function over cost proportions $\wcost$, \emph{expected loss} $\Lcost$ is defined as 
\begin{equation}\label{eqLcost}
\Lcost \triangleq \int^1_0{} \Qcost(\Tcost(c); c) \wcost(c) dc 
\end{equation}
%
Incorporating the class distribution into the operating condition we obtain expected loss over a distribution of skews:
\begin{equation}\label{eqLsk}
\Lsk \triangleq \int^1_0{} \Qsk(\Tsk(\sk); \sk) \wsk(\sk) d\sk
\end{equation}
\end{definition}
It is worth noting that if we plot $\Qcost$ or $\Qsk$ against $c$ and $\sk$, respectively, we obtain \emph{cost curves} as defined by \cite{DH00,drummond-and-Holte2006}. Cost curves are also known as risk curves (see, e.g. \cite{reid2011information}, where the plot can also be shown in terms of {\em priors}, i.e., class proportions).

Equations (\ref{eqLcost}) and (\ref{eqLsk}) illustrate the space we explore in this paper. Two parameters determine the expected loss: $\wcost(c)$ and $\Tcost(c)$ (respectively $\wsk(\sk)$ and $\Tsk(\sk)$). While much work has been done on a first dimension, by changing $\wcost(c)$ or $\wsk(\sk)$, particularly in the area of proper scoring rules, no work has systematically analysed what happens when changing the second dimension, $\Tcost(c)$ or $\Tsk(\sk)$.

%

\section{Expected loss for fixed-threshold \classifiers}\label{sec:fixed}





The easiest way to choose the threshold is to set it to a pre-defined value $t_{fixed}$, independently from the \model and also from the operating condition. 
%
%
This is, in fact, what many classifiers do (e.g. Naive Bayes chooses $t_{fixed} = 0.5$ independently from the \model and independently from the operating condition). 
We will see the straightforward result that this \thresholdchoicemethod corresponds to 0-1 loss (either micro-average accuracy, \acc, or macro-average accuracy, \macro\acc). Part of these results will be useful to better understand some other \thresholdchoicemethods.

\begin{definition}\label{def:scorefixed}
The {\em \scorefixed \thresholdchoicemethod} is defined as follows:
\begin{equation}\label{eqTf}
\Tcostsf[t](c) \triangleq \Tsksf[t](\sk) \triangleq  t
\end{equation}
\end{definition}


This choice has been criticised in two ways, but is still frequently used. Firstly, choosing 0.5 as a threshold is not generally the best choice even for balanced datasets or for applications where the test distribution is equal to the training distribution (see, e.g. \cite{lachiche2003improving} on how to get much more from a Bayes classifier by simply changing the threshold). Secondly, even if we are able to find a better value than 0.5, this does not mean that this value is best for every skew or cost proportion ---this is precisely one of the reasons why ROC analysis is used \cite{provost2001robust}. 
Only when we know the deployment operating condition at evaluation time is it reasonable to fix the threshold according to this information.
So either by common choice or because we have this latter case, consider then that we are going to use the same threshold $t$ independently of skews or cost proportions. Given this  \thresholdchoicemethod, then the question is: {\em if we must evaluate a \model before application for a wide range of skews and cost proportions, which \performancemetric should be used?}
This is what we answer below.

If we plug $\Tcostsf$ (Equation (\ref{eqTf})) into the general formula of the expected loss for a range of cost proportions (Equation (\ref{eqLcost})) we have:
\begin{equation}\label{eqLcostf}
\Lcostsf(t) \triangleq \int^1_0{} \Qcost(\Tcostsf[t](c); c) \wcost(c) dc 
\end{equation}
We obtain the following straightforward result.

\begin{theorem}\label{thm:LcostsfU}
If a \classifier sets the decision threshold at a fixed value irrespective of the operating condition or the \model, then expected loss under a uniform distribution of cost proportions is equal to the error rate at that decision threshold. 
\end{theorem}

\begin{proof}
\begin{eqnarray*}
\LcostsfU(t) & = & \int^1_0{} \Qcost(\Tcostsf[t](c); c) U(c) dc = \int^1_0{} \Qcost(t; c) dc  \\
            & = &  \int^1_0{} 2 \{c \pi_0 (1 -F_0(t)) + (1-c)\pi_1 F_1(t)\} dc \\
%
%
& = & 2\pi_0 (1 -F_0(t)) \int^1_0{} cdc  + 2\pi_1 F_1(t) \int^1_0{}(1-c)dc  \\
            & = & 2\pi_0 (1 -F_0(t)) (1/2)  + 2\pi_1 F_1(t) (1/2)  
             =  \pi_0 (1-F_0(t)) + \pi_1 F_1(t)  = 1 - \acc(t)
\end{eqnarray*}
In words, the expected loss is equal to the class-weighted average of false positive rate and false negative rate, which is the (micro-average) error rate. 
\end{proof}

So, the expected loss under a uniform distribution of cost proportions for the {\em \scorefixed \thresholdchoicemethod} is the error rate of the classifier at that threshold. That means that accuracy can be seen as a measure of classification performance in a range of costs proportions when we choose a fixed threshold. This interpretation is reasonable, since accuracy is a \performancemetric which is typically applied to \classifiers (where the threshold is fixed) and not to \models outputting scores.  This is exactly what we did in Table \ref{tab:models2}. We calculated the expected loss for the fixed threshold at 0.5 for a uniform distribution of cost proportions, and we got $1 - \acc$ = $0.51$ and $0.375$ for models $A$ and $B$ respectively.

Similarly, if we plug $\Tsksf$ (Equation (\ref{eqTf})) into the general formula of the expected loss for a range of skews (Equation (\ref{eqLsk})) we have:
\begin{equation}\label{eqLskf}
\Lsksf(t) \triangleq \int^1_0{} \Qsk(\Tsksf[t](\sk); \sk) \wsk(\sk) d\sk 
\end{equation}
Using Lemma \ref{lemma-balance} we obtain the equivalent result for skews: 

\begin{corollary}\label{thm:LsksfU}
If a \classifier sets the decision threshold at a fixed value irrespective of the operating condition or the \model, then expected loss under a uniform distribution of skews is equal to the macro-average error rate at that decision threshold: 
$\LsksfU(t) =  (1-F_0(t))/2 + F_1(t)/2 = 1 - \macro\acc(t)$. 
\end{corollary}

The previous results show that 0-1 losses are appropriate to evaluate \models in a range of operating conditions if the threshold is fixed for all of them. In other words, accuracy and macro-accuracy can be the right \performancemetrics for \classifiers even in a cost-sensitive learning scenario. The situation occurs when one assumes a particular operating condition at evaluation time while the classifier has to deal with a range of operating conditions in deployment time.




In order to prepare for later results we also define a particular way of setting a fixed classification threshold, namely to achieve a particular predicted positive rate. 
One could say that such a method {\em quantifies} the proportion of positive predictions made by the \classifier.
For example, we could say that our threshold is fixed to achieve a \rate of 30\% positive predictions and the rest negatives.
This of course involves ranking the examples by their scores and setting a cutting point at the appropriate position, something which is frequently known as `screening'.


\newcommand{\Tcostrfi}{R^{-1}}
\newcommand{\Tskrfi}{R^{-1}_{\sk}}

\begin{definition}\label{def:Tq2}
Define the predicted positive rate at threshold $t$ as 
$R(t) = \pi_0F_0(t) + \pi_1F_1(t)$,
and assume the cumulative distribution functions $F_{0}$ and $F_{1}$ are invertible, then we define the {\em \ratefixed \thresholdchoicemethod} for rate $r$ as: 
\begin{equation}\label{eqTq2}
\Tcostrf[\ra](c) \triangleq \Tcostrfi(\ra)
\end{equation}
If $F_{0}$ and $F_{1}$ are not invertible, they have plateaus and so does $R$. This can be handled by deriving $t$ from the centroid of a plateau. 

The \ratefixed \thresholdchoicemethod for skews is defined as:
\begin{equation}\label{eqTskewrf}
\Tskrf[\ra](\sk) \triangleq \Tskrfi(\ra)
\end{equation}
where $R_{\sk}(t) = F_0(t)/2 + F_1(t)/2$.
\end{definition}
%
%
%
The corresponding expected loss for cost proportions is
\begin{equation}\label{eqLcostrf}
\Lcostrf \triangleq \int^1_0{} \Qcost(\Tcostrf[\ra](c); c) \wcost(c) dc = \int^1_0{} \Qcost(\Tcostrfi(\ra); c) \wcost(c) dc 
\end{equation}

The notion of setting a threshold based on a rate is closely related to the problem of quantification \cite{DBLP:journals/datamine/Forman08,bella2010quantification} where the goal is to correctly estimate the proportion for each of the classes (in the binary case, the positive rate is sufficient). This \thresholdchoicemethod allows the user to set the quantity of positives, which can be known (from a sample of the test) or can be estimated using a quantification method. In fact, some quantification methods can be seen as methods to determine an absolute fixed threshold $t$ that ensures a correct proportion for the test set.

Note that Equation (\ref{eqLcostrf}) is closely related to Theorem \ref{thm:LcostsfU}. If we determine the threshold which produces a rate, i.e., if we determine $\Tcostrfi(\ra)$, we get the expected loss as an accuracy. Formally, we have:

\begin{equation}\label{eqLcostrfU}
\LcostrfU = 1 - \acc(\Tcostrfi(\ra))
\end{equation}

Fortunately, it is immediate to get the threshold which produces a rate; it can just be derived by sorting the examples by their scores and placing the cutpoint where the rate equals the rank divided by the number of examples (e.g. if we have $n$ examples, the cutpoint $i$ makes $\ra=i/n$).

\section{\Thresholdchoicemethods using scores} \label{sec:scores}


In the previous section we looked at accuracy and error rate as \performancemetrics for \classifiers and gave their interpretation as expected losses. In this and the following sections we consider \performancemetrics for \models that do not require fixing a \thresholdchoicemethod in advance. Such metrics include \auc which evaluates ranking performance and the Brier score or mean squared error which evaluates the quality of probability estimates. We will deal with the latter in this section. 
We will therefore assume that scores range between 0 and 1 and represent posterior probabilities for class 1. 
This means that we can sample thresholds uniformly or derive them from the operating condition. 
We first introduce two \performancemetrics that are applicable to probabilistic scores.


The Brier score is a well-known \performancemetric for probabilistic \models. It is an alternative name for the Mean Squared Error or MSE loss \cite{brier1950verification}, especially for binary classification. 
\begin{definition}\label{def:BS}
$\bs(m,D)$ denotes the Brier score of \model $m$ on data $D$; we will usually omit $m$ and $D$ when clear from the context. 
\bs is defined as follows:
\begin{align*}\label{eqBS}
\bs & \triangleq   
 \pi_0 \bs_0 + \pi_1 \bs_1 \\
\bs_{0} & \triangleq \int_{0}^{1} s^{2}f_{0}(s) ds \\
\bs_{1} & \triangleq \int_{0}^{1} (1-s)^{2}f_{1}(s) ds 
\end{align*}
From here, we can define a prior-independent version of the Brier score (or a macro-average Brier score) as follows:           
\begin{equation}\label{eqMBS}
\macro\bs \triangleq   \frac{ \bs_0 +  \bs_1}{2}
\end{equation}
\end{definition}

The Mean Absolute Error (\mae) is another simple \performancemetric which has been rediscovered many times under different names. 
\begin{definition}\label{def:MAE}
$\mae(m,D)$ denotes the Mean Absolute Error of \model $m$ on data $D$; we will again usually omit $m$ and $D$ when clear from the context. 
\mae is defined as follows:
\begin{align*}\label{eqMAE}
\mae & \triangleq    
\pi_0 \mae_0 + \pi_1 \mae_1 \\
\mae_{0} & \triangleq \int_{0}^{1} s  f_{0}(s) ds = \m{s}_0\\
\mae_{1} & \triangleq \int_{0}^{1} (1-s)  f_{1}(s) ds = 1 - \m{s}_1
\end{align*}
%
We can define a macro-average \mae as follows:
\begin{equation}\label{eqmacroMAE}
{\macro}\mae \triangleq    \frac{\mae_0 + \mae_1}{2} = \frac{\m{s}_0 + (1 - \m{s}_1)}{2}
\end{equation}
\end{definition}
\noindent
It can be shown that \mae is equivalent to the Mean Probability Rate (MPR) \cite{LL02} for discrete classification  \cite{PRL09}.
%



\subsection{The \scoreuniform \thresholdchoicemethod leads to \mae}\label{sec:scoreuniform}

We now demonstrate how varying a \model's threshold leads to an expected loss that is different from accuracy. 
First, we explore a \thresholdchoicemethod which considers that we have no information at all about the operating condition, neither at evaluation time nor at deployment time. We just  employ the interval between the maximum and minimum value of the scores, and we randomly select the threshold using a uniform distribution over this interval.

%

\begin{definition}\label{def:Tcostsu}
Assuming a \model's scores are expressed on a bounded scale $[l,u]$, the {\em \scoreuniform \thresholdchoicemethod} is defined as follows:
\begin{equation}\label{eqTs}
\Tcostsu(c) \triangleq \Tsksu(\sk) \triangleq  \Tcostsf[U_{l,u}](c)
\end{equation}
\end{definition}

Given this \thresholdchoicemethod, then the question is: {\em if we must evaluate a \model before application for a wide range of skews and cost proportions, which \performancemetric should be used?}

\begin{theorem}\label{thm:LcostsU}
Assuming probabilistic scores and the \scoreuniform \thresholdchoicemethod, expected loss under a uniform distribution of cost proportions is equal to the \model's mean absolute error. 
\end{theorem}
\begin{proof}
First of all we note that the \thresholdchoicemethod does not take the operating condition $c$ into account, and hence we can work with $c=1/2$. 
Then, 
\begin{align*}
\LcostsuU & = \Qcost(\Tcostsu(1/2); 1/2)  = \Qcost(\Tcostsf[U_{l,u}](1/2); 1/2)  = \int^u_l{} \Qcost(\Tcostsf[t](1/2); 1/2)\frac{1}{u-l} dt \\
 & = \frac{1}{u-l}\int^u_l{} \Qcost(t; 1/2) dt 
 = \frac{1}{u-l}\int^u_l{} \{\pi_0(1-F_0(t)) + \pi_1F_1(t)\} dt  =  \frac{\pi_0(\m{s}_0-l) + \pi_1(u-\m{s}_1)}{(u-l)} 
\end{align*}
%
%
%
%
The last step makes use of the following useful property. 
\begin{eqnarray*}
\int_{l}^{u} F_k(t)dt
 = [tF_k(t)]^{u}_{l} - \int_{l}^{u} tf_k(t)dt
 = uF_k(u) - lF_k(l) - \m{s}_k
 = u-\m{s}_k\label{eq:prop62}
\end{eqnarray*}
%
Setting $l=0$ and $u=1$ for probabilistic scores, we obtain the final result: 
\begin{eqnarray*}
\LcostsuU = {{\pi_0 \m{s}_0 + \pi_1 (1 - \m{s}_1)} = \mae }
\end{eqnarray*}
\end{proof}


This gives a baseline loss if we choose thresholds randomly and independently of the model. 
Using Lemma \ref{lemma-balance} we obtain the equivalent result for skews: 

\begin{corollary}\label{thm:LsksuU}
Assuming probabilistic scores and the \scoreuniform \thresholdchoicemethod, expected loss under a uniform distribution of skews is equal to the \model's macro-average mean absolute error: 
\begin{eqnarray*}
\LsksuU & = & {{\m{s}_0 + (1 - \m{s}_1)} \over 2} = \mmae \label{eq:baselineskews}
\end{eqnarray*}
\end{corollary}



\subsection{The \scoredriven \thresholdchoicemethod leads to the Brier score}\label{sec:scoredriven}

We will now consider the first \thresholdchoicemethod to take the operating condition into account. 
Since we are dealing with probabilistic scores, this method 
simply sets the threshold equal to
the operating condition (cost proportion or skew). 
This is a natural criterion as it has been used especially when the \model is a probability estimator and we expect to have perfect information about the operating condition at deployment time.
In fact, this is a direct choice when working with proper scoring rules, since when rules are proper, scores are assumed to be a probabilistic assessment. The use of this \thresholdchoicemethod can be traced back to Murphy \cite{murphy1966note} and, perhaps, implicitly, much earlier.
More recently, and in a different context from proper scoring rules, Drummond and Holte \cite{drummond-and-Holte2006} say it is a common example of a ``performance independence criterion''. Referring to figure 22 in their paper which uses the \scoredriven \thresholdchoice they say: ``the performance independent criterion, in this case, is to set the threshold to correspond to the operating
conditions. For example, if $PC(+) = 0.2$ the Naive Bayes threshold is set to 0.2". The term $PC(+)$ is equivalent to our `skew'.

\begin{definition}\label{def:Tcostp}
Assuming \model's scores are expressed on a probability scale $[0,1]$, the {\em \scoredriven \thresholdchoicemethod} is defined for cost proportions as follows:
\begin{equation}\label{eqTcostp}
\Tcostsd(c) \triangleq  c
\end{equation}
and for skews as
\begin{equation}\label{eqTskp}
\Tsksd(\sk) \triangleq \sk 
\end{equation}
\end{definition}


Given this \thresholdchoicemethod, then the question is: {\em if we must evaluate a \model before application for a wide range of skews and cost proportions, which \performancemetric should be used?} This is what we answer below.

\begin{theorem}[\cite{ICML11Brier}]\label{thm:LcostpUequalsBS2}
Assuming probabilistic scores and the \scoredriven \thresholdchoicemethod, expected loss under a uniform distribution of cost proportions is equal to the \model's Brier score. 
\end{theorem}
\begin{proof}
If we plug $\Tcostsd$ (Equation (\ref{eqTcostp})) into the general formula of the expected loss (Equation (\ref{eqLcost})) we have the expected \scoredriven loss:
\begin{equation}\label{eqLcostp}
\Lcostsd \triangleq \int^1_0{} \Qcost(\Tcostsd(c); c) \wcost(c) dc =  \int^1_0{} \Qcost(c; c) \wcost(c) dc
\end{equation}
And if we use the uniform distribution and the definition of $\Qcost$ (Equation (\ref{eqQcost})):
\begin{eqnarray}\label{eqLcostpU}
\LcostsdU & = & \int^1_0{} \Qcost(c; c) U(c) dc  
				  =           \int^1_0{} 2 \{ c \pi_0 (1 -F_0(c)) + (1-c)\pi_1 F_1(c) \} dc 
			 \label{eqLcostpUlast}
\end{eqnarray}
In order to show this is equal to the Brier score, we expand the definition of $\bs_0$ and $\bs_1$ using integration by parts: 
\begin{align*}
\bs_{0} & = \int_{0}^{1} s^{2}f_{0}(s) ds 
  = \left[ s^{2}F_{0}(s) \right]_{s=0}^{1} - \int_{0}^{1} 2sF_{0}(s) ds  = 1 - \int_{0}^{1} 2sF_{0}(s) ds\\
 & = \int_{0}^{1} 2s ds - \int_{0}^{1} 2sF_{0}(s) ds  = \int_{0}^{1} 2s(1-F_{0}(s)) ds \\
\bs_{1} & = \int_{0}^{1} (1-s)^{2}f_{1}(s) ds 
 = \left[ (1-s)^{2}F_{1}(s) \right]_{s=0}^{1} + \int_{0}^{1} 2(1-s)F_{1}(s) ds = \int_{0}^{1} 2(1-s)F_{1}(s) ds
\end{align*}
Taking their weighted average, we obtain
\begin{align}
\bs & = \pi_0 \bs_0 + \pi_1 \bs_1  = \int_{0}^{1} \{\pi_{0}2s(1-F_{0}(s)) + \pi_{1}2(1-s)F_{1}(s)\} ds \label{eq:bsloss}
\end{align}
which, after reordering of terms and change of variable, is the same expression as 
Equation~(\ref{eqLcostpUlast}).

\end{proof}

It is now clear why we just put the Brier score from Table \ref{tab:models1} as the expected loss in Table \ref{tab:models2}. We calculated the expected loss for the \scoredriven \thresholdchoicemethod for a uniform distribution of cost proportions as its Brier score.

Theorem \ref{thm:LcostpUequalsBS2} was obtained in \cite{ICML11Brier} (the \thresholdchoicemethod there was called `probabilistic') but it is not completely new in itself. 
In \cite{murphy1966note} we find a similar relation to expected utility (in our notation, $-(1/4)PS+(1/2)(1+\pi_0)$, where the so-called probability score $PS = 2\bs$). Apart from the sign (which is explained because Murphy works with utilities and we work with costs), the difference in the second constant term is explained because Murphy's utility (cost) model is based on a cost matrix where we have a cost for one of the classes (in meteorology the class `protect') independently of whether we have a right or wrong prediction (`adverse' or `good' weather). The only case in the matrix with a 0 cost is when we have `good' weather and `no protect'.
It is interesting to see that the result only differs by a constant term, which supports the idea that whenever we can express the operating condition with a cost proportion or skew, the results will be portable to each situation with the inclusion of some constant terms (which are the same for all classifiers).
In addition to this result, it is also worth mentioning another work by Murphy \cite{murphy1969measures} where he makes a general derivation for the Beta distribution. 


After Murphy, in the last four decades, there has been extensive work on the so-called proper scoring rules, where several utility (cost) models have been used and several distributions for the cost have been used. 
This has led to relating Brier score (square loss), logarithmic loss, 0-1 loss and other losses which take the scores into account. For instance, in \cite{BSS05} we have a comprehensive account of how all these losses can be obtained as special cases of the Beta distribution. The result given in Theorem \ref{thm:LcostpUequalsBS2} would be a particular case for the uniform distribution (which is a special case of the Beta distribution) and a variant of Murphy's results. Nonetheless, it is important to remark that the results we have just obtained in Section \ref{sec:scoreuniform} (and those we will get in Section \ref{sec:rates}) are new because they are not obtained by changing the cost distribution but rather by changing the \thresholdchoicemethod. The \thresholdchoicemethod used (the \scoredriven one) is not put into question in the area of proper scoring rules. But Theorem \ref{thm:LcostpUequalsBS2} can now be seen as a result which connects these two different dimensions: cost distribution and threshold choice method, so placing the Brier score at an even more predominant role.


We can derive an equivalent result using empirical distributions \cite{ICML11Brier}. In that paper we show how the loss can be plotted in cost space, leading to the \emph{Brier curve} whose area below is the Brier score. 

Finally, using skews we arrive at the prior-independent version of the Brier score.  

%
%

\begin{corollary}\label{thm:LskpUequalsBS2}
$\LsksdU = \mbs = {(\bs_0 + \bs_1) / 2}$. 
\end{corollary}


It is interesting to analyse the relation between $\LcostsuU$ and $\LcostsdU$  
(similarly between $\LsksuU$ and $\LsksdU$).
Since the former gives the \mae and the second gives the Brier score (which is the MSE), from the definitions of \mae and Brier score, we get that, assuming scores are between $0$ and $1$ we have:
$$\mae = \LcostsuU \geq \LcostsdU = \bs$$  
$$\mmae = \LsksuU \geq \LsksdU = \mbs$$  
Since \mae and \bs have the same terms but the second squares them, and all the values which are squared are between 0 and 1, then the \bs must be lower or equal.
This is natural, since the expected loss is lower if we get reliable information about the operating condition at deployment time. 
So, the difference between the Brier score and \mae is precisely the gain we can get by having (and using) the information about the operating condition at deployment time.
Notice that all this holds regardless of the quality of the probability estimates.

\section{\Thresholdchoicemethods using \rates} \label{sec:rates}


We show in this section that \auc can be translated into expected loss for varying operating conditions in more than one way, depending on the \thresholdchoicemethod used. We consider two \thresholdchoicemethods, where each of them sets the threshold to achieve a particular predicted positive rate: 
the \rateuniform method, which sets the rate in a uniform way; 
and the \ratedriven method, which sets the rate equal to the operating condition. 


We recall the definition of a ROC curve and its area first.
\begin{definition}\label{def:ROC}
The ROC curve \cite{SDM00,Fawcett06} is defined as a plot of $F_1(t)$ (i.e., false positive rate at decision threshold $t$) on the \xaxis against $F_0(t)$ (true positive rate at $t$) on the \yaxis, with both quantities monotonically non-decreasing with increasing $t$ (remember that scores increase with $\hat{p}(1|x)$ and 1 stands for the negative class).
The Area Under the ROC curve (\auc) is defined as:
\begin{eqnarray}
\auc & \triangleq & \int_{0}^{1} F_0(s) d F_1(s) = \int_{-\infty}^{+\infty} F_0(s) f_1(s) ds = \int_{-\infty}^{+\infty} \int_{-\infty}^{s} f_0(t) f_1(s) dt ds  \label{eq:AUC}\\
     & = & \int_{0}^{1} (1-F_1(s)) d F_0(s) = \int_{-\infty}^{+\infty} (1-F_1(s)) f_0(s) ds = \int_{-\infty}^{+\infty} \int_{s}^{+\infty} f_1(t) f_0(s) dt ds \nonumber
\end{eqnarray}
%
\end{definition}
%

\subsection{The \rateuniform \thresholdchoicemethod leads to \auc} \label{sec:rateuniform}

The \ratefixed \thresholdchoicemethod places the threshold in such a way that a given predictive positive rate is achieved.
However, if this proportion may change easily or we are not going to have (reliable) information about the operating condition at deployment time, an alternative idea is to consider a non-deterministic choice or a distribution for this quantity. One reasonable choice can be a uniform distribution. 

\begin{definition}\label{def:Ti-continuous}
The \rateuniform \thresholdchoicemethod non-deterministically sets the threshold to achieve a uniformly randomly selected rate:
\begin{align}\label{eqTi-continuous}
\Tcostru(c) & \triangleq \Tcostrf[U_{0,1}](c)  \\
\Tskru(\sk) & \triangleq \Tskrf[U_{0,1}](\sk) 
\end{align}
\end{definition}
\noindent
In other words, it sets a relative quantity (from 0\% positives to 100\% positives) in a uniform way, and obtains the threshold from this uniform distribution over \rates.
Note that for a large number of examples, this is the same as defining a uniform distribution over examples or, alternatively, over cutpoints (between examples), as explored in \cite{ICML11CoherentAUC}. 

There are reasons for considering this threshold a reasonable method.
It is a generalisation of the  \ratefixed \thresholdchoicemethod which considers  all the imbalances (class proportions) equally likely whenever we make a classification.
It assumes that we will not have any information about the operating condition at deployment time.

As done before for other \thresholdchoicemethods, we analyse the question: given this \thresholdchoicemethod, {\em if we must evaluate a \model before application for a wide range of skews and cost proportions, which \performancemetric should be used?}

The corresponding expected loss for cost proportions is
\begin{equation}\label{eqLcostru}
\Lcostru \triangleq \int^1_0 \Qcost(\Tcostru(c); c) \wcost(c) dc = \int^1_0 \int^1_0 \Qcost(\Tcostrfi(\ra); c)U(r) \wcost(c) dr\; dc 
\end{equation}
We then have the following result.

%
\begin{theorem}[\cite{ICML11CoherentAUC}]\label{th::LcostiUUU} 
Assuming the \ratefixed \thresholdchoicemethod, 
expected loss for uniform cost proportion and uniform rate decreases linearly with \auc as follows: 
$$\LcostruU = 
\pi_0\pi_1(1-2\auc) + 1/2$$
\end{theorem}
\begin{proof}
First of all we note that the \thresholdchoicemethod does not take the operating condition $c$ into account, and hence we can work with $c=1/2$. 
Furthermore, $r = R(t)$ and hence $dr = R'(t)dt = \{\pi_{0}f_{0}(t) + \pi_{1}f_{1}(t)\}dt$. 
Then, 
\begin{align*}
\LcostruU & = \int^1_0 \Qcost(\Tcostrfi(\ra); 1/2)U(r) dr\\
& = \int_{-\infty}^{\infty} \{\pi_0 (1 -F_0(t)) + \pi_1 F_1(t)\} \{\pi_{0}f_{0}(t) + \pi_{1}f_{1}(t)\} dt \\
& =  \pi_{0}\pi_{1} \int_{-\infty}^{\infty} \{(1 -F_0(t))f_{1}(t) + F_1(t)f_{0}(t)\} dt \\
& +  \pi_{0}^{2} \int_{-\infty}^{\infty} (1 -F_0(t))f_{0}(t)\ dt + \pi_{1}^{2} \int_{-\infty}^{\infty} F_1(t)f_{1}(t)\ dt 
\end{align*}
The first term can be related to \auc:
\begin{align*}
\int_{-\infty}^{\infty} \{(1 -F_0(t))f_{1}(t) + F_1(t)f_{0}(t)\} dt 
& = \int_{-\infty}^{\infty} f_{1}(t)dt - \int_{-\infty}^{\infty} F_0(t)f_{1}(t) dt + \int_{-\infty}^{\infty} F_1(t)f_{0}(t) dt \\
& = 1 - \auc + (1-\auc) = 2(1-\auc)
\end{align*}
The remaining two terms are easily solved:
\begin{align*}
\pi_{0}^{2} \int_{-\infty}^{\infty} (1 -F_0(t))f_{0}(t)\ dt
& =-\pi_{0}^{2} \int_{1}^{0} (1 -F_0(t))\ d(1-F_{0}(t))
  = \pi_{0}^{2}/2 \\
\pi_{1}^{2} \int_{-\infty}^{\infty} F_1(t)f_{1}(t)\ dt 
& = \pi_{1}^{2} \int_{0}^{1} F_1(t)\ dF_{1}(t) 
  = \pi_{1}^{2}/2 
\end{align*}
Putting everything together we obtain 
$\LcostruU = 2\pi_0\pi_1(1-\auc) + (\pi_{0}^{2}+\pi_{1}^{2})/2$. 
Since $\pi_0\pi_1+ (\pi_{0}^{2}+\pi_{1}^{2})/2 = (\pi_{0} + \pi_{1})^{2}/2 = 1/2$, 
this can be rewritten to
$\LcostruU = \pi_0\pi_1(1-2\auc) + 1/2$.%
\footnote{If we do not assume a uniform distribution for cost proportions $U(c)$ we would obtain a different integral, but expected loss would still be linear in AUC (David Hand, personal communication). }
\end{proof}

\begin{corollary}\label{th::LskiUUU}
Assuming the \ratefixed \thresholdchoicemethod, 
expected loss for uniform skew and uniform rate decreases linearly with \auc as follows: 
$$\LskruU = 
(1-2\auc)/4 + 1/2$$
\end{corollary}

\noindent We see that expected loss for uniform skew ranges from 1/4 for a perfect ranker that is harmed by sub-\optimal \thresholdchoices, to 3/4 for the worst possible ranker that puts positives and negatives the wrong way round, yet gains some performance by putting the threshold at or close to one of the extremes. 

Intuitively, these formulae can be understood as follows. Setting a randomly sampled rate is equivalent to setting the decision threshold to the score of a randomly sampled example. 
With probability $\pi_{0}$ we select a positive and with probability $\pi_{1}$ we select a negative. 
If we select a positive, then the expected true positive rate is $1/2$ (as on average we select the middle one); and the expected false positive rate is $1-\auc$ (as one interpretation of \auc\ is the expected proportion of negatives ranked correctly wrt.\ a random positive).   
Similarly, if we select a negative then the expected true positive rate is $\auc$ and the expected false positive rate is $1/2$.  
Put together, the expected true positive rate is
$\pi_{0}/2 + \pi_{1}\auc$
and the expected false positive rate is
$\pi_{1}/2 + \pi_{0}(1-\auc)$.
The proportion of true positives among all examples is thus
$$\pi_{0}\left(\pi_{0}/2 + \pi_{1}\auc\right) = {\pi_{0}^{2} \over 2} + \pi_{0}\pi_{1}\auc$$
and the proportion of false positives is
$$\pi_{1}\left(\pi_{1}/2 + \pi_{0}(1-\auc)\right) = {\pi_{1}^{2} \over 2} + \pi_{0}\pi_{1}(1-\auc)$$
We can summarise these expectations in the following contingency table (all numbers are proportions relative to the total number of examples): 
\begin{center}
\begin{tabular}{l|l|l|l}
 & Predicted $+$ & Predicted $-$ & \\
\hline
Actual $+$ & $\pi_{0}^{2} / 2 + \pi_{0}\pi_{1}\auc$ & $\pi_{0}^{2} / 2 + \pi_{0}\pi_{1}(1-\auc)$ & $\pi_{0}$ \\
\hline
Actual $-$ & $\pi_{1}^{2} / 2 + \pi_{0}\pi_{1}(1-\auc)$ & $\pi_{1}^{2} / 2 + \pi_{0}\pi_{1}\auc$ & $\pi_{1}$ \\
\hline
 & $1/2$ & $1/2$ & 1 \\
\end{tabular}
\end{center}
%
The column totals are, of course, as expected: if we randomly select an example to split on, then the expected split is in the middle.



While in this paper we concentrate on the case where we have access to population densities $f_{k}(s)$ and distribution functions $F_{k}(t)$, in practice we have to work with empirical estimates.
In \cite{ICML11CoherentAUC} we provide an alternative formulation of the main results in this section, relating empirical loss to the \auc of the empirical ROC curve.
For instance, the expected loss for uniform skew and uniform instance selection is
calculated in \cite{ICML11CoherentAUC} to be $ \left({n \over {n+1}}\right){{1-2\auc} \over 4} + {1 \over 2}$, showing that for smaller samples the reduction in loss due to \auc is somewhat smaller.

\subsection{The \ratedriven \thresholdchoicemethod leads to \auc} \label{sec:ratedriven}

Naturally, if we can have precise information of the operating condition at deployment time, we can use the information about the skew or cost to adjust the \rate of positives and negatives to that proportion.
This leads to a new threshold selection method: if we are given skew (or cost proportion) $\sk$ (or $c$), we choose the threshold $t$ in such a way that we get a proportion of $\sk$ (or $c$) positives.
This is an elaboration of the \ratefixed \thresholdchoicemethod which {\em does} take the operating condition into account.

\begin{definition}\label{def:Tcostrd}
The {\em \ratedriven \thresholdchoicemethod} for cost proportions is defined as
\begin{equation}\label{eqTcostrd}
\Tcostrd(c) \triangleq \Tcostrf[c](c) = \Tcostrfi(c) 
\end{equation}
The \ratedriven \thresholdchoicemethod for skews is defined as
\begin{equation}\label{eqTskrd}
\Tskrd(\sk) \triangleq \Tskrf[\sk](\sk) =  \Tskrfi(\sk) 
\end{equation}
\end{definition}

Given this \thresholdchoicemethod, the question is again: {\em if we must evaluate a \model before application for a wide range of skews and cost proportions, which \performancemetric should be used?}
This is what we answer below.

If we plug $\Tcostrd$ (Equation (\ref{eqTcostrd})) into the general formula of the expected loss for a range of cost proportions (Equation (\ref{eqLcost})) we have:
\begin{equation}\label{eqLcostn}
\Lcostrd \triangleq \int^1_0{} \Qcost(\Tcostrd(c); c) \wcost(c) dc 
\end{equation}

And now, from this definition, if we use the uniform distribution for $\wcost(c)$, we obtain this new result.

\begin{theorem}
Expected loss for uniform cost proportions using the \ratedriven \thresholdchoicemethod is linearly related to \auc as follows: 
\[ \LcostrdU
  =  \pi_1\pi_0  (1  - 2\auc) + 1/3   \]
\end{theorem}
\begin{proof}
\begin{eqnarray*}
\LcostrdU    & = & \int^1_0{} \Qcost(\Tcostrd(c); c) U(c) dc = \int^1_0{} \Qcost(\Tcostrfi(c); c) dc \\
\end{eqnarray*}
By a change of variable we have $c=R(t)$ and hence $dc = R'(t)dt = \{\pi_{0}f_{0}(t)+\pi_{1}f_{1}(t)\}dt = R'(t)dt$, and thus
\begin{eqnarray*}
\LcostrdU    & = &  \int_{-\infty}^{+\infty}{} \Qcost(t; c) R'(t) dt =   \int_{-\infty}^{+\infty}{} 2 \{c \pi_0 (1 -F_0(t)) + (1-c)\pi_1 F_1(t)\}  R'(t) dt \\
            & = &  \int_{-\infty}^{+\infty}{} 2 \{c \pi_0  - c ( \pi_0 F_0(t) + c\pi_1 F_1(t)) + \pi_1 F_1(t)  \} R'(t) dt \\
            & = &  \int_{-\infty}^{+\infty}{} 2 \{c \pi_0  - c ( \pi_0 F_0(t) + \pi_1 F_1(t)) \} R'(t) dt + \int_{-\infty}^{+\infty}{} 2 \{ \pi_1 F_1(t)  \} R'(t) dt 
\end{eqnarray*}
%
%
All terms in the first integral can be reduced to $R(t)=c$:
\begin{eqnarray*}
\int_{-\infty}^{+\infty}{} 2 \{c \pi_0  - c ( \pi_0 F_0(t) + \pi_1 F_1(t)) \} R'(t) dt    & = & \int^1_0{} 2 \{c \pi_0  - c^{2} \} dc  \\
            & = & [ c^2 \pi_0  - \frac{2 c^3}{3} ]^1_0  =  \pi_0  - \frac{2}{3}  \\
\end{eqnarray*}
%
%
The second integral provides the link to \auc: 
\begin{eqnarray*}
\int_{-\infty}^{+\infty}{} F_1(t) R'(t) dt & = & \int^\infty_{-\infty}{} F_1(t) \{\pi_0 f_0(t) + \pi_1 f_1(t)\} dt  =  \pi_0 \int^\infty_{-\infty}{} F_1(t) f_0(t) dt + \pi_1 \int^\infty_{-\infty}{} F_1(t) f_1(t) dt \\
	& = &  \pi_0  (1  - \auc) + \pi_1 \int^1_0{} F_1(t) dF_1(t) 
             =   \pi_0  (1  - \auc) + \frac{\pi_1}{2} 
\end{eqnarray*}
And now we can plug this into the expression for the expected loss:
\begin{eqnarray*}
\LcostrdU    & = & \pi_0  - \frac{2}{3}  + 2 \pi_1 (\pi_0  (1  - \auc) + \frac{\pi_1}{2} ) = \pi_0  - \frac{2}{3}  + 2 \pi_1\pi_0  (1  - \auc) + \pi_1 \pi_1  \\						
            & = & 2 \pi_1\pi_0  (1  - \auc) + \pi_1 (1- \pi_0)  + \pi_0  - \frac{2}{3}   							
             =  \pi_1\pi_0  (1  - 2\auc) + \frac{1}{3}   						
\end{eqnarray*}

\end{proof}

Now we can unveil and understand how we obtained the results for the expected loss in Table \ref{tab:models2} for the \ratedriven method. We just took the \auc of the models and applied the previous formula: $\pi_1\pi_0  (1  - 2\auc) + \frac{1}{3}$.


\begin{corollary} \label{cor:LskrdU}
Expected loss for uniform skews using the \ratedriven \thresholdchoicemethod is linearly related to \auc as follows: 
\[ \LskrdU
  =  (1-2\auc)/4 + 1/3 \]
\end{corollary}




If we compare Corollary \ref{th::LskiUUU} with Corollary \ref{cor:LskrdU},
we see that $\LskruU > \LskrdU$, more precisely:

\[ \LskruU = (1-2\auc)/4 + 1/2 =  \LskrdU + 1/6  \]

\begin{figure}
\centering
\includegraphics[width=0.4\textwidth]{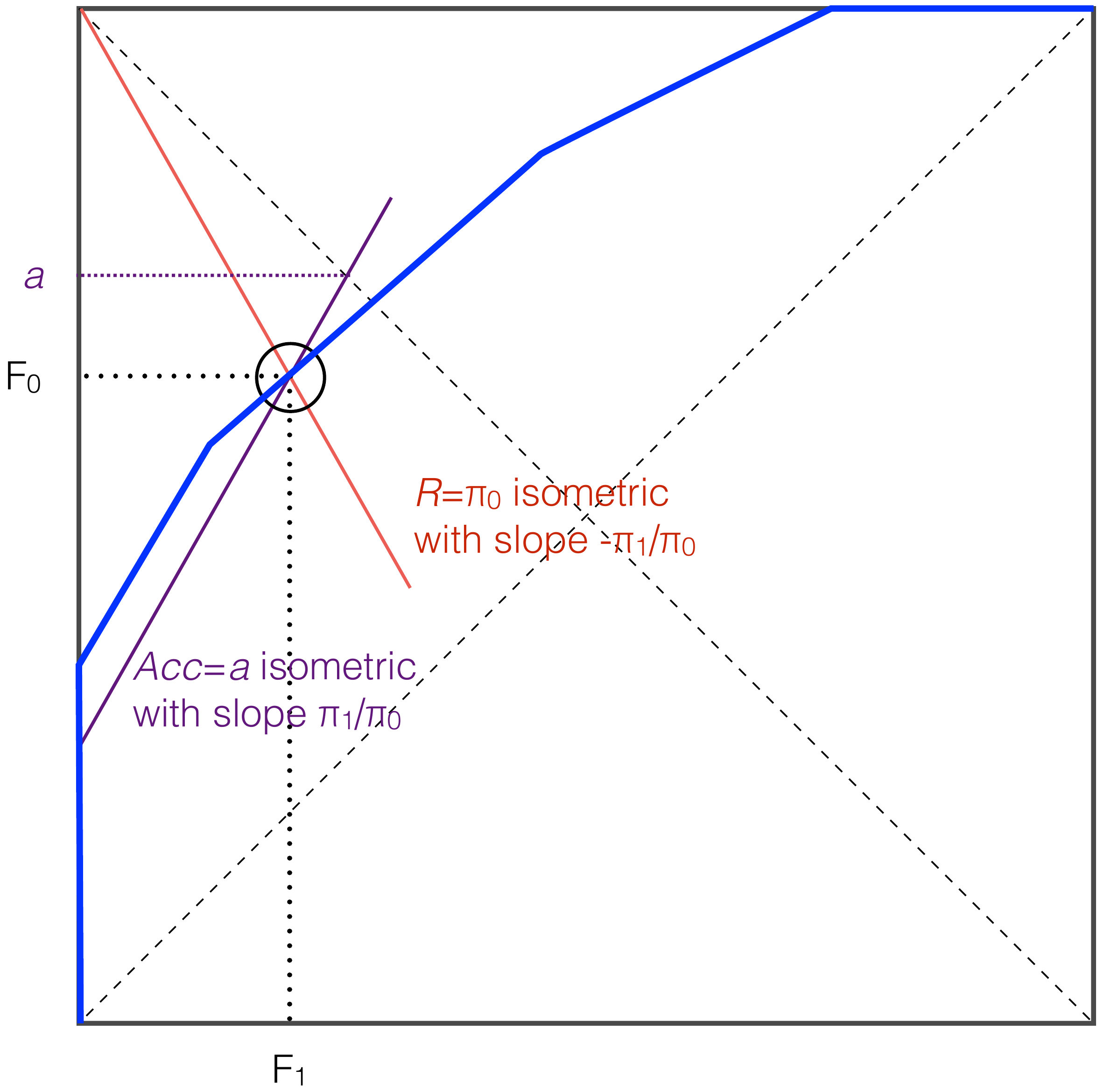} \hfill
\includegraphics[width=0.4\textwidth]{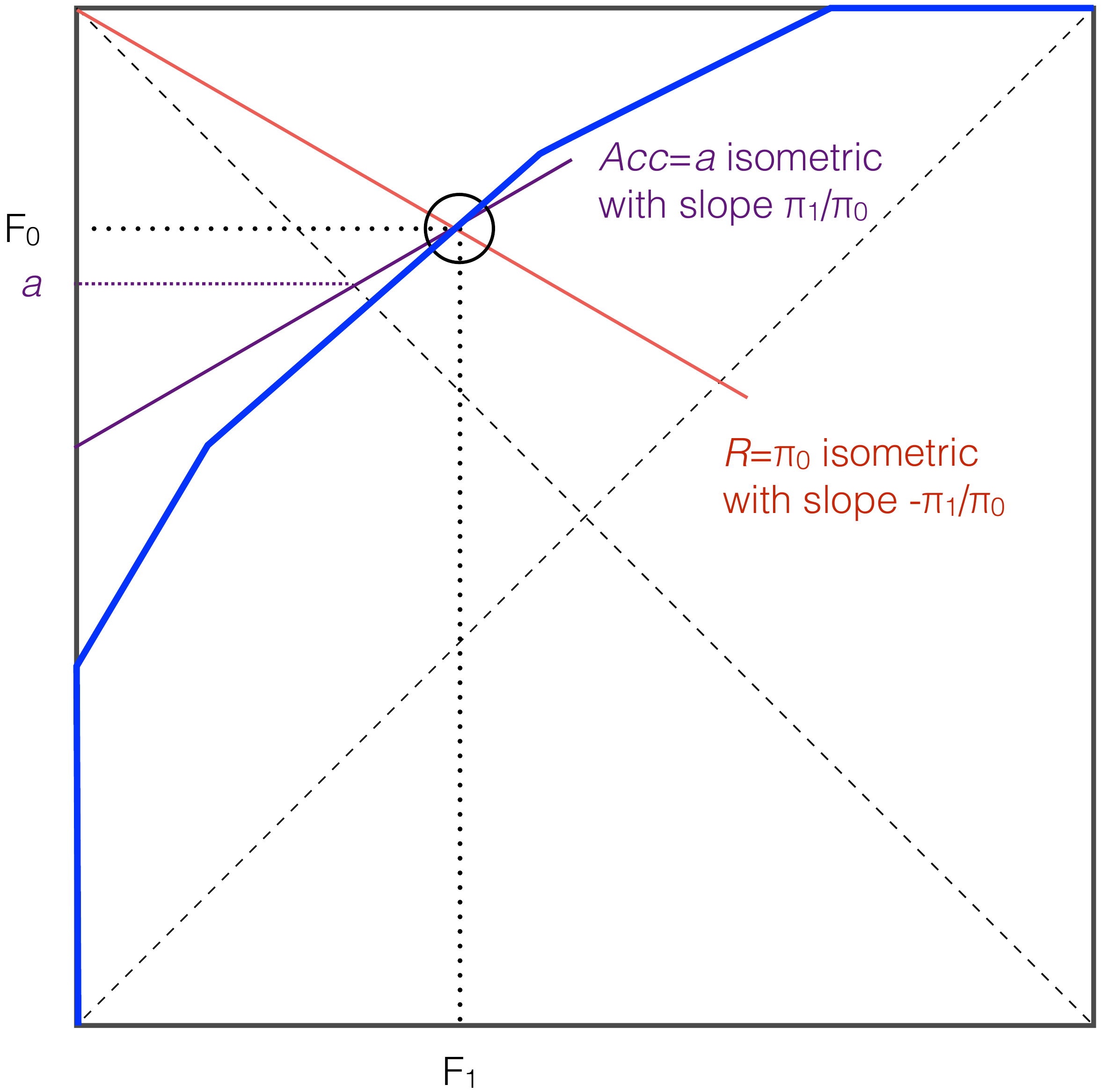}
\caption{Illustration of the \ratedriven \thresholdchoicemethod. We assume uniform misclassification costs ($c_{0} = c_{1} = 1$), and hence skew is equal to the proportion of positives ($z=\pi_{0}$). The majority class is class 1 on the left and class 0 on the right. Unlike the \rateuniform method, the \ratedriven method is able to take advantage of knowing the majority class, leading to a lower expected loss.}
\label{fig:RateDriven}
\end{figure}

So we see that taking the operating condition into account when choosing thresholds based on rates reduces the expected loss with $1/6$, {\em regardless of the quality of the \model} as measured by \auc. 
This term is clearly not negligible and demonstrates that our novel \ratedriven \thresholdchoicemethod is superior to the \rateuniform method. Figure \ref{fig:RateDriven} illustrates this. 
Logically, $\LcostrdU$ and $\LskrdU$ work upon information about the operating condition at deployment time, while $\LcostruU$ and $\LskruU$ may be suited when this information is unavailable or unreliable.

\section{The \optimal \thresholdchoicemethod}\label{sec:optimal}


The last \thresholdchoicemethod we investigate is based on the optimistic assumption
that (1) we are having complete information about the operating condition (class proportions and costs) at deployment time and (2) we are able to use that information (also at deployment time) to choose the threshold that will minimise the loss using the current model.
ROC analysis is precisely based on these two points since we can calculate the threshold which gives the smallest loss by using the skew and the convex hull.

This \thresholdchoicemethod, denoted by $\Tcosto$, is defined as follows:
\begin{definition} The \optimal \thresholdchoicemethod is defined as:
\begin{eqnarray}
\Tcosto(c) & \triangleq  & \argmin_{t}\ \{\Qcost(t; c)\} =  \argmin_{t}\ 2\{c\pi_0(1-F_0(t)) + (1-c)\pi_1 F_1(t)\}  \label{eqTcosto}
\end{eqnarray}
and similarly for skews:
%
\begin{align} \label{eqTsko}
\Tsko(\sk) & \triangleq \argmin_{t}\ \{\Qsk(t; \sk)\} \nonumber 
\end{align}
\end{definition}

Note that in both cases, the $\argmin$ will typically give a range (interval) of values which give the same optimal value. So these methods can be considered non-deterministic.
This \thresholdchoicemethod is analysed by \cite{Fawcett-and-Provost1997}, and used by \cite{DH00,drummond-and-Holte2006} for defining their cost curves and by \cite{hand2009measuring} to define a new performance metric.

If we plug Equations (\ref{eqTcosto}) and (\ref{eqQcost}) into Equation (\ref{eqLcost}) using a uniform distribution for cost proportions, we get:
%
\begin{eqnarray}
\LcostoU & = & \int^{1}_{0} \Qcost(\underset{t}\argmin\{\Qcost(t,c)\}; c) dc \nonumber = \int^{1}_{0}{} \underset{t}\min\{ \Qcost(t; c)\} dc \nonumber \\
				& = & \int^{1}_{0} \underset{t}\min\{ 2c\pi_0(1-F_0(t)) + 2(1-c)\pi_1 F_1(t) \} \label{eq:Loriginal} dc 
\end{eqnarray}

The connection with the convex hull of a ROC curve (ROCCH) is straightforward.
The convex hull  is a construction over the ROC curve in such a way that all the points on the convex hull have minimum loss for some choice of $c$ or $\sk$. 
This means that we restrict attention to the {\em optimal} threshold for a given cost proportion $c$, as derived from Equation (\ref{eqTcosto}).


\subsection{Convexification}

We can give a corresponding, and more formal, definition of the convex hull as derived from the score distributions. First, we need a more precise definition of a convex model. For that, we rely on the ROC curve, and we use the slope of the curve, defined as usual:
\newcommand{\slope}{\ensuremath{\mathit{slope}}}
\begin{equation}
\slope(T) = \frac{f_0(T)}{f_1(T)} \label{eq:slopedef}
\end{equation}
A related expression we will also use is:
\begin{equation}
c(T) = \frac{\pi_1 f_1(T)}{\pi_0 f_0(T) + \pi_1 f_1(T)} \label{eq:cT}
\end{equation}
Sometimes we will use subindices for $c(T)$ depending on the model we are using.
Clearly, $\frac{\pi_0}{\pi_1}{\slope(T)} = \frac{\pi_0 f_0(T)}{\pi_1 f_1(T)} = \frac{1}{c(T)}-1$.


\begin{definition}[Convex model]\label{def:Convex}
A model $m$ is convex, if for every threshold $T$, we have that $c(T)$ is non-decreasing (or, equivalently, $\slope(T)$ is non-increasing).
\end{definition}

In order to make any model convex, it is not sufficient to repair local concavities, we need to calculate the convex hull. A definition of convex hull for continuous distributions is given as follows:

\begin{definition}[Convexification]\label{def:Conv}
Let $m$ be any model with score distributions $f_{0}(T)$ and $f_{1}(T)$. 
Some values of $t$ will never minimise $\Qcost(t; c) =  2c\pi_0(1-F_0(t)) + 2(1-c)\pi_1 F_1(t) \} $ for any value of $c$. These values will be in one or more intervals of which only the end points will minimise $\Qcost(t; c) $ for some value of $c$. We will call these intervals {\em non-hull intervals}, and all the rest will be referred to as {\em hull intervals}. 
It clearly holds that hull intervals are convex. 
Non-hull intervals may contain convex and concave subintervals.

Define convexified score distributions $e_{0}(T)$ and $e_{1}(T)$ as follows. 
\begin{enumerate}
\item For every hull 
interval $t_{i-1}\leq s\leq t_{i}$:
$e_{0}(T)=f_{0}(T)$ and $e_{1}(T)=f_{1}(T)$.
\item For every non-hull 
 interval $t_{j-1}\leq s\leq t_{j}$:
\begin{align*}e_{0}(T)&=e_{0,j}=\frac{1}{t_{j}-t_{j-1}}\int_{t_{j-1}}^{t_{j}} f_{0}(T)dT\\
e_{1}(T)&=e_{1,j}=\frac{1}{t_{j}-t_{j-1}}\int_{t_{j-1}}^{t_{j}} f_{1}(T)dT\end{align*}
\end{enumerate}
The function \Conv returns the model $\Conv(m)$ defined by the score distributions $e_{0}(T)$ and $e_{1}(T)$.
\end{definition}

We can also define the cumulative distributions $E_x(t) = \int_{0}^{t} e_{x}(T)dT$, where $x$ represents either 0 or 1.
By construction we have that for every interval $[t_{j-1},t_{j}]$ identified above:
\begin{eqnarray}
[E_x(t)]_{t_{j-1}}^{t_{j}} & = & \int_{t_{j-1}}^{t_{j}} e_{x}(T)dT = (t_{j}-t_{j-1})e_{x,j}  \label{eq:constantdensity} = \int_{t_{j-1}}^{t_{j}} f_{x}(T)dT = [F_x(t)]_{t_{j-1}}^{t_{j}} 
\end{eqnarray}				
and so the convexified score distributions are proper distributions. Furthermore, since the new score distributions are constant in the convexified intervals -- and hence monotonically non-decreasing for the new $c(T)$, denoted by $c_{\Conv(m)}(T)$-- so is 
$$c_{\Conv(m)}(T) =c_{j}= \frac{\pi_1 e_{1,j}} {\pi_0 e_{0,j} + \pi_1 e_{1,j}}$$

It follows that $\Conv(m)$ is everywhere convex.
In addition,

\begin{theorem}\label{th:step2b}
Optimal loss is invariant under \Conv, i.e.:
$\LcostoU(\Conv(m)) = \LcostoU(m)$ for every $m$. 
\end{theorem}
\begin{proof}
By Equation (\ref{eq:Loriginal}) we have that optimal loss is:
\begin{eqnarray*}
\LcostoU(m) & = & \int^1_0{} \underset{t}\min\{ 2c\pi_0(1-F_0(t)) + 2(1-c)\pi_1 F_1(t) \} dc 
\end{eqnarray*}
By definition, the hull intervals have not been modified by $\Conv(m)$.
Only the non-hull intervals have been modified. A non-hull interval was defined as those where there is no $t$ which minimises $\Qcost(t; c) =  2c\pi_0(1-F_0(t)) + 2(1-c)\pi_1 F_1(t) \} $ for any value of $c$, and only the endpoints attained the minimum.
Consequently, we only need to show that the new $e_0(T)$ and $e_1(T)$ do not introduce any new minima.

We now focus on each non-hull segment $(t_{j-1}, t_j)$ using the definition of $\Conv$. We only need to check the expression for the minimum:
\begin{eqnarray*}
 \underset{{t_j} \leq t \leq {t_{j-1}}}\min\{ 2c\pi_0(1-E_0(t)) + 2(1-c)\pi_1 E_1(t) \} 
\end{eqnarray*}
From Equation (\ref{eq:constantdensity}) we derive that $E_x(t) = E_x(t_{j-1}) + (t_{j}-t_{j-1})e_{x,j}$ inside the interval (they are straight lines in the ROC curve), and we can see that the expression to be minimised is constant (it does not depend on $t$). Since the end points were the old minima and were equal, we see that this expression cannot find new minima.
\end{proof}

It is not difficult to see that if we plot $\Conv(m)$ in the cost space defined  
by \cite{drummond-and-Holte2006}  with $\Qsk(t; \sk)$ on the \yaxis against skew $\sk$ on the \xaxis, we have a cost curve. Its area is then 
the expected loss for the \optimal \thresholdchoicemethod.
%
%
%
%
%
%
%
%
%
%
%
%
%
In other words, this is the area under the (\optimal) cost curve.
%
Similarly, the new \performancemetric introduced by Hand ($H$) \cite{hand2009measuring} is simply a normalised version of the area under the \optimal cost curve using the $B_{2,2}$ distribution instead of the $B_{1,1}$ (i.e., uniform) distribution, and using cost proportions instead of skews (so being dependent to class priors).
This is further discussed in \cite{ICML11CoherentAUC}.


\subsection{The \optimal \thresholdchoicemethod leads to refinement loss}\label{sec:RL}

Once again, the question now must be stated clearly. Assume that the \optimal \thresholdchoicemethod is set as the method we will use for every application of our \model. Furthermore, assume that each and every application of the model is going to find the perfect threshold. Then, {\em if we must evaluate a \model before application for a wide range of skews and cost proportions, which \performancemetric should be used?}
In what follows, we will find the answer by relating this expected loss with a genuine \performancemetric: refinement loss. We will now introduce this performance metric.

The Brier score, being a sum of squared residuals, can be decomposed in various ways. The most common decomposition of the Brier score is due to Murphy \cite{murphy1973new} and decomposes  the Brier score into Reliability, Resolution and Uncertainty. Frequently, the two latter components are joined together and the decomposition gives two terms: calibration loss and refinement loss.

This decomposition is usually applied to empirical distributions, requiring a binning of the scores.  That is, the decomposition is based on a partition ${\cal{P}}_D = \{b_j\}_{j=1..B}$ where $D$ is the dataset, $B$ the number of bins, and each bin is denoted by $b_j \subset D$. Since it is a partition $\bigcup_{j=1}^{B}{b_j} = D$.
With this partition the decompoistion is:
%
\begin{equation}\label{eqBSdecomp}
\bs \approx  \cl^{{\cal{P}}_D} +  \rl^{{\cal{P}}_D} = {1 \over n} \sum_{j=1}^{B} |b_j| \left( s_{b_j} - y_{b_j} \right)^2 + {1 \over n} \sum_{j=1}^{B} |b_j| y_{b_j} \left( 1 - y_{b_j} \right) 
\end{equation}
Here we use the notation $s_{b_j} = {1 \over |b_j|}\sum_{i \in b_j} s_i$ and $y_{b_j}={1 \over |b_j|}\sum_{i \in b_j} y_i$ for the average predicted scores and the average actual classes respectively for bin $b_j$. 

%
For many partitions the empirical decomposition is not exact. It is only exact for partitions which are coarser than the partition induced by the ROC curve (i.e., ties cannot be spread over different partitions), as shown by \cite{flach-matsubara-ecml07}. We denote by $\cl^{ROC}$ and $\rl^{ROC}$ the calibration loss and the refinement loss, respectively, using the segments of the empirical ROC curve as bins. In this case, $\bs =  \cl^{ROC} +  \rl^{ROC}$.


In this paper we will use a variant of the above decomposition based on the ROC convex hull of a \model. In this decomposition, we take each bin as each segment in the convex hull. Naturally, the number of bins in this decomposition is lower or equal than the number of bins in the ROC decomposition. In fact, we may find different values of $s_i$ in the same bin. In some way, we can think about this decomposition as an optimistic/optimal version of the ROC decomposition, as Theorem 3 in \cite{flach-matsubara-ecml07} shows. We denote by $ \cl^{ROCCH}$ and $ \rl^{ROCCH}$ the calibration loss and the refinement loss, respectively, using the segments of the convex hull of the empirical ROC curve as bins \cite{ICML11CoherentAUC}. 




We can define the same decomposition in continuous terms considering Definition \ref{def:BS}.
%
%
%
%
%
%
%
%
%
%
We can see that in the continuous case, the partition is irrelevant. Any partition will give the same result, since the composition of consecutive integrals is the same as the whole integral. 

\begin{theorem}\label{eqBSContdecom} 

The continuous decomposition of the Brier Score, $BS=\cl+\rl$, is exact and gives \cl and \rl as follows.
%
\begin{align*}
\cl & =   \int_{0}^{1} \frac{	 \left( s(\pi_0 f_0(s) + \pi_1 f_1(s)) - \pi_1 f_1(s) \right)^2	} {\pi_0 f_0(s) + \pi_1 f_1(s)}     ds \\
\rl & =  \int_{0}^{1}  \frac{\pi_1 f_1(s) \pi_0 f_0(s)} {\pi_0 f_0(s) + \pi_1 f_1(s)}    ds
\end{align*}
%
\end{theorem}
\begin{proof}
\begin{align*}
\bs & = \int_{0}^{1} \left[s^2 \pi_0 f_0(s) + (1-s)^2 \pi_1 f_1(s)\right]  ds \\
   & =   \int_{0}^{1} \left[s^2(\pi_0 f_0(s) + \pi_1 f_1(s))  - 2s\pi_1 f_1(s) + \pi_1 f_1(s)\right] ds \\
   & =  \int_{0}^{1} \frac{s^2(\pi_0 f_0(s) + \pi_1 f_1(s))^2  - 2s(\pi_0 f_0(s) + \pi_1 f_1(s))\pi_1 f_1(s) + \pi_1 f_1(s)(\pi_1 f_1(s) + \pi_0 f_0(s))}{(\pi_0 f_0(s) + \pi_1 f_1(s))}  ds \\
& =   \int_{0}^{1} \frac{\left( s(\pi_0 f_0(s) + \pi_1 f_1(s)) - \pi_1 f_1(s) \right)^2	+ \pi_1 f_1(s) \pi_0 f_0(s)} {\pi_0 f_0(s) + \pi_1 f_1(s)}    ds \\
  & = \int_{0}^{1} \frac{\left( s(\pi_0 f_0(s) + \pi_1 f_1(s)) - \pi_1 f_1(s) \right)^2	} {\pi_0 f_0(s) + \pi_1 f_1(s)}  ds 
  + \int_{0}^{1}  \frac{\pi_1 f_1(s) \pi_0 f_0(s)} {\pi_0 f_0(s) + \pi_1 f_1(s)}    ds
\end{align*}
\end{proof}
This proof keeps the integral from start to end. That means that the decomposition is not only true for the integral as a whole, but also pointwise for every single score $s$.
Note that $y_{b_j}$ in the empirical case (see Definition \ref{eqBSdecomp}) corresponds to 
$c(s) = \frac{\pi_1 f_1(s)}{\pi_0 f_0(s) + \pi_1 f_1(s)}$ (as given by Equation (\ref{eq:cT})) in the continuous case above, and also note that $s_{b_j}$ corresponds to the cardinality $\pi_0 f_0(s) + \pi_1 f_1(s)$.  
The decomposition for empirical distributions as introduced by \cite{murphy1973new} is still predominant for any reference to the decomposition. To our knowledge this is the first explicit derivation of a continuous version of the decomposition.





And now we are ready for relating the \optimal \thresholdchoicemethod with a \performancemetric as follows:

\begin{theorem}\label{thm:RL}
For every convex model $m$, we have that:
\begin{equation*}
  \LcostoU(m) = \rl(m)
\end{equation*} 
\end{theorem}
The proof of this theorem is found in the appendix as Theorem \ref{th:planA}. 
This proof is accompanied with several examples that show that the above correspondence is not pointwise in general. This means that the $\rl$ is a genuinely different way of calculating $\LcostoU(m)$. In fact, in the appendix, we see that there is a third way of calculating $\LcostoU(m)$.  


\begin{corollary}
For every model $m$ the expected loss for the \optimal \thresholdchoicemethod $\LcostoU$ is equal to the refinement loss using the convex hull.
\begin{equation*}
  \LcostoU(m) = {\rl(\Conv(m))} \triangleq  \rl_{\Conv}(m)
\end{equation*} 
\end{corollary}
\begin{proof}
We have $\LcostoU(m) = \LcostoU(\Conv(m))$ by Theorem \ref{th:step2b}, and
$\LcostoU(\Conv(m)) = \rl(\Conv(m))$ by Theorem \ref{thm:RL} and the convexity of $\Conv(m)$. 
\end{proof}
It is possible to obtain a version of this theorem for empirical distributions which states that $\LcostoU = { \rl^{ROCCH}}$ where $ \rl^{ROCCH}$ is the refinement loss of the empirical distribution using the segments of the convex hull for the decomposition.

Before better analysing what the meaning of this \thresholdchoicemethod is and how it relates to the rest, we first have to consider whether this \thresholdchoicemethod is realistic or not.
%
%
%
In the beginning of this section we said that the optimal method assumes that  (1) we are having complete information about the operating condition at deployment time and (2) we are able to use that information to choose the threshold that will minimise the loss at deployment time.

While (1) is not always true, there are many occasions where we know the costs and distributions at application time. This is the base of the \scoredriven and \ratedriven methods. However, having this information does not mean that the \optimal threshold for a dataset (e.g. the training or validation dataset) ensures an optimal choice for a test set (2).
Drummond and Holte \cite{drummond-and-Holte2006} are conscious of this problem and they reluctantly rely on a \thresholdchoicemethod which is based on ``the ROC convex hull [...] only if this selection criterion happens to make cost-minimizing selections, which in general it will not do''. But even if these cost-minimising selections are done, as mentioned above, it is not clear how reliable they are for a test dataset. As Drummond and Holte \cite{drummond-and-Holte2006} page 122, recognise: ``there are few examples of the practical application of this technique. One example is \cite{Fawcett-and-Provost1997}, in which the decision threshold parameter was tuned to be optimal, empirically, for the test distribution''.
In the example shown in Table \ref{tab:models2} in Section \ref{sec:intro}, the evaluation technique was training and test. However, with cross-validation, the convex hull cannot be estimated reliably in general, and the thresholds derived from each fold might be inconsistent. Even with a big validation dataset, the decision threshold may be suboptimal. This is one of the reasons why the area under the convex hull has not been used as a performance metric. In any case, we can calculate the values as an optimistic limit, leading to $\LcostoU =  \rl^{ROCCH}= 0.0953$ for model $A$ and $0.2094$ for model $B$.

\section{Relating \performancemetrics}\label{sec:calibration}

So far, we have repeatedly answered the following question: ``If \thresholdchoicemethod $X$ is used, which is the corresponding performance metric?"
The answers are summarised in Table \ref{tab:intro}.
The seven \thresholdchoicemethods are shown in the first column (the two fixed methods are grouped in the same row).
The integrated view of \performancemetrics for classification is given by the next two columns. The expected loss of a \model for a uniform distribution of cost proportions or skews for each of these seven \thresholdchoicemethods produces most of the common \performancemetrics in classification: 0-1 loss (either macro-accuracy or micro-accuracy), the Mean Absolute Error (equivalent to Mean Probability Rate),
the Brier score, \auc (which equals the Wilcoxon-Mann-Whitney statistic and the Kendall tau distance of the \model to the perfect \model, and is linearly related to the Gini coefficient) and, finally, the {refinement loss} using the bins given by the convex hull.

\begin{table}[htdp]
	\centering
	  {\footnotesize
		\begin{tabular}{p{1.5cm}ccp{6.25cm}}
		  \hline\hline
			\Thresholdchoicemethod        & Cost proportions                 & Skews                               & Equivalent (or related) \performancemetrics     \\\hline\hline
			fixed                          & $\LcostsfU = 1-Acc$  \           & $\LsksfU= 1-{\macro}Acc$            & 0-1 loss: Accuracy and macro-accuracy. \\\hline
      \scoreuniform                 & $\LcostsuU=\mae$                 & $\LsksuU={\macro}\mae$  & Absolute error, Average score, $p\auc$ \cite{FFHS05} ,  Probability Rate \cite{PRL09}.        \\\hline
      \scoredriven                  & $\LcostsdU= \bs$                 & $\LsksdU= {\macro}\bs$  & Brier score \cite{brier1950verification}, Mean Squared Error ($MSE$).            \\\hline
      \rateuniform                  & $\LcostruU= \pi_0\pi_1(1-2\auc)+ \frac{1}{2}$           & $\LskruU= \frac{1-2\auc}{4}+ \frac{1}{2}$         & \auc \cite{SDM00} and variants ({\micro}\auc) \cite{Faw01,PRL09}, Kendall tau, WMW statistic, Gini coefficient. \\\hline
      \ratedriven                   & $\LcostrdU= \pi_0\pi_1(1-2\auc)+ \frac{1}{3}$           & $\LskrdU= \frac{1-2\auc}{4}+ \frac{1}{3}$         & \auc \cite{SDM00} and variants ({\micro}\auc) \cite{Faw01,PRL09}, Kendall tau, WMW statistic, Gini coefficient. \\\hline
     \optimal                   & $\LcostoU=  \rl_{\Conv}$       & $\LskoU= {\macro} \rl_{\Conv}$       & ROCCH Refinement loss \cite{flach-matsubara-ecml07}, Refinement Loss \cite{murphy1973new}, Area under the Cost Curve (`Total Expected Cost') \cite{drummond-and-Holte2006},  Hand's H  \cite{hand2009measuring}.     \\\hline
			\hline
		\end{tabular}
		}
	\caption{\Thresholdchoicemethods and their expected loss for cost proportions and skews. The $M$ in $\macc$, $\mmae$, ${\macro}\bs$ and ${\macro}\rl$ mean that these metrics are `macro-averaged', i.e., calculated as if $\pi_0 = \pi_1$.}
	\label{tab:intro}
\end{table}

All the \thresholdchoicemethods seen in this paper consider \model scores in different ways. Some of them disregard the score, since the threshold is fixed, some others consider the `magnitude' of the score as an (accurate) estimated probability, leading to the score-based methods, and others consider the `rank', `\rate' or `proportion' given by the scores, leading to the \rate-based methods. Since the \optimal \thresholdchoice is also based on the convex hull, it is apparently more related to the \rate-based methods. This is consistent with the taxonomy proposed in \cite{PRL09} based on correlations over more than a dozen performance metrics, where three families of metrics were recognised:
 \performancemetrics which account for the quality of classification (such as accuracy), \performancemetrics which account for a ranking quality (such as \auc), and \performancemetrics which evaluate the quality of scores or how well the \model does in terms of probability estimation (such as the Brier score or logloss).

This suggests that the way scores are distributed is crucial in understanding the differences and connections between these metrics. In addition, this may shed light on which threshold choice method is best.
We have already seen some relations, such as $\LcostsuU \geq \LcostsdU$, and $\LcostruU > \LcostrdU$, but what about $\LcostsdU$ and $\LcostrdU$? Are they comparable? And what about $\LcostoU$? It gives the minimum expected loss by definition over the training (or validation) dataset, but when does it become a good estimation of the expected loss for the test dataset?

In order to answer these questions we need to analyse transformations on the scores and see how these affect the expected loss given by each \thresholdchoicemethod. 
Given a \model, its scores establish a total order over the examples:
$\sigma = (s_1, s_2, ..., s_n)$ where $s_i \leq s_{i+1}$. Since there might be ties in the scores, this total order is not necessarily strict.
A monotonic transformation is any alteration of the scores, such that the order is kept.
%
%
%
We will consider two transformations: the evenly-spaced transformation and PAV calibration.




 \subsection{Evenly-spaced scores. Relating Brier score, \mae and \auc}\label{sec:evenly} 

If we are given a ranking or order, or we are given a set of scores but its reliability is low, a quite simple way to assign (or re-assign) the scores is to set them evenly-spaced (in the $[0,1]$ interval).
\begin{definition} A discrete evenly-spaced transformation is a procedure $\Evend(\sigma) \rightarrow \sigma'$ which converts any sequence of scores $\sigma = (s_1, s_2, ..., s_n)$ where $s_i < s_{i+1}$ into scores $\sigma' = (s'_1, s'_2, ..., s'_n)$ where $s'_i = \frac{i-1}{n-1}$.
\end{definition}
Notice that such a transformation does not affect the ranking and hence does not alter the \auc. 

The previous definition can be applied to continuous score distribution as follows:

\begin{definition} A continuous evenly-spaced transformation is a any strictly monotonic transformation function on the score distribution, denoted by
$\Even$,
 such that for the new scores $s'$ it holds that $P(s' \leq t)=t$.
\end{definition}

It is easy to see that $\Evend$ is idempotent, i.e., $\Evend(\Evend(\sigma)) = \Evend(\sigma)$.
So we say a set of scores $\sigma$ is evenly-spaced if $\Evend(\sigma) = \sigma$.

\begin{lemma}\label{lmm:evenly}
Given a model and dataset with set of scores $\sigma$, such that they evenly-spaced, when $n \rightarrow \infty$ then we have $R(t) = t$.
\end{lemma}

\begin{proof}
Remember that by definition the true positive rate $F_0(t) = P(s \leq t | 0)$ and the false positive rate $F_1(t) = P(s \leq t | 1)$. Consequently, from the definition of rate we have $R(t) = \pi_0F_0(t) + \pi_1F_1(t) = \pi_0P(s \leq t | 0) + \pi_1P(s \leq t | 1) = P(s \leq t)$. But, since the scores are evenly-spaced, the number of scores such that $s \leq t$ is  $\sum_{i=1}^n I(s_i \leq t) = \sum_{i=1}^n I(\frac{i-1}{n-1} \leq t)$ with $I$ being the indicator function (1 when true, 0 otherwise). This number of scores is $\sum_{i=1}^{tn} 1$ when $n \rightarrow \infty$, which clearly gives $tn$.
So the probability $P(s \leq t)$ is $tn/n=t$. Consequently $R(t) = t$.
\end{proof}

The following results connect the \scoredriven \thresholdchoicemethod with
the \ratedriven \thresholdchoicemethod:

\begin{theorem}\label{thm:evenly}
Given a model and dataset with set of scores $\sigma$, such that they are evenly-spaced, when $n \rightarrow \infty$:
\begin{equation}
\bs = \LcostsdU = \LcostrdU = \pi_0 \pi_1 (1 - 2\auc) + \frac{1}{3} 
\end{equation} 
\end{theorem}
\begin{proof}
By Lemma \ref{lmm:evenly} we have $R(t) = t$, and so the \ratedriven and \scoredriven \thresholdchoicemethods select the same thresholds. 
\end{proof}
%
%
\begin{corollary}\label{thm:evenlyskews}
Given a \model and dataset with set of scores $\sigma$ such that they are evenly-spaced, when $n \rightarrow \infty$:
\begin{equation}
\mbs = \LsksdU = \LskrdU = \frac{1 - 2\auc}{4} + \frac{1}{3} 
\end{equation} 
\end{corollary}
These straightforward results connect \auc and Brier score for evenly-spaced scores.
This connection is enlightening because it says that \auc and \bs are equivalent \performancemetrics (linearly related) when we set the scores in an evenly-spaced way. In other words, it says that \auc is like a Brier score which considers all the scores evenly-spaced.
Although the condition is strong, this is the first linear connection which, to our knowledge, has been established so far between \auc and the Brier score.

Similarly, we get the same results for the \scoreuniform \thresholdchoicemethod and
the \rateuniform \thresholdchoicemethod. 
\begin{theorem}\label{thm:evenlyuniform}
Given a \model and dataset with set of scores $\sigma$ such that they are evenly-spaced, when $n \rightarrow \infty$:
\begin{equation}
\mae = \LcostsuU = \LcostruU =  \pi_0 \pi_1 (1 - 2\auc) + \frac{1}{2} 
\end{equation} 
\end{theorem}
with similar results for skews.
This also connects \mae with \auc and clarifies when they are linearly related.

\subsection{Perfectly-calibrated scores. Relating \bs, \cl and \rl}\label{sec:perfectcal}

In this section we will work with a different condition on the scores. We will study what interesting connections can be established if we assume the scores to be perfectly calibrated. 

The informal definition of perfect calibration usually says that a model is calibrated when the estimated probabilities are close to the true probabilities. From this informal definition, we would derive that a model is perfectly calibrated if the estimated probability given by the scores (i.e., $\hat{p}(1|x)$) equals the true probability.
However, if this definition is applied to single instances, it implies not only perfect calibration but a perfect model. In order to give a more meaningful definition, the notion of calibration is then usually defined in terms of groups or bins of examples, as we did, for instance, with the Brier score decomposition.
So, we need to apply this correspondence between estimated and true (actual) probabilities over bins.
We say a bin partition is invariant on the scores if for any two examples with the same score they are in the same bin. In other words, two equal scores cannot be in different bins (equivalence classes cannot be broken). From here,
\begin{definition}[Perfectly-calibrated for empirical distribution \models]
We say that a \model is {\em perfectly calibrated} if for any invariant bin partition ${{\cal{P}}}$ we have that $y_{b_j} = s_{b_j}$ for all its bins: i.e., the average actual probability equals the average estimated probability, thus making $ \cl^{{\cal{P}}}=0$. Note that it is not sufficient to have $\cl=0$ for one partition, but for all the invariant partitions.
\end{definition}
Notice that the bins which are generated by a ROC curve are the minimal invariant partition on the scores (i.e., the quotient set). So, we can give an alternative definition of perfectly calibrated \model: a \model is perfectly calibrated if and only if $ \cl^{ROC}=0$.
For the continuous case, the partition is irrelevant and the definition is as follows: 


\begin{definition}[Perfectly-calibrated for continuous distribution \models]
We say a continuous \model is {\em perfectly calibrated} if $\cl=0$.
\end{definition}


\begin{lemma}\label{CLzero}
For a perfectly calibrated classifier $m$:
\[ \frac{1-s}{s} = \frac{ f_0(s)}{f_1(s)} \frac{\pi_0}{\pi_1} \]
and $c(s) = s$ as defined in Equation (\ref{eq:cT}), which means that $m$ is convex.
\end{lemma}

\begin{proof}
Consider the decomposition of  Theorem \ref{eqBSContdecom}.
%
%
%
%
A perfectly calibrated classifier must have  $\cl=0$ for every single continuous interval.
That means that:
\begin{eqnarray*}
\cl & = &  \int_{0}^{1} \frac{	 \left( s(\pi_0 f_0(s) + \pi_1 f_1(s)) - \pi_1 f_1(s) \right)^2	} {\pi_0 f_0(s) + \pi_1 f_1(s)}     ds \\
& = & \int_{0}^{1}  (\pi_0 f_0(s) + \pi_1 f_1(s))\left(s- \frac{\pi_1 f_1(s)} {\pi_0 f_0(s) + \pi_1 f_1(s)}  \right)^2  ds = 0
\end{eqnarray*}
which means
%
%
%
%
%
$s = \frac{\pi_1 f_1(s)} {\pi_0 f_0(s) + \pi_1 f_1(s)}$
which is exactly $c(s)$ given by Equation (\ref{eq:cT}) and the result and convexity follow (by definition \ref{def:Convex}). 
\end{proof}

Now that we have two proper and operational definitions of perfect calibration, we define a calibration transformation as follows. 
\begin{definition}
$\Cal$ is a monotonic function over the scores which converts any \model $m$ into another calibrated \model $m^*$ such that $\cl=0$ and $\rl$ is not modified. 
\end{definition}
$\Cal$ always produces a convex model, so $\Conv(\Cal(m))= \Cal(m)$, but a convex model is not always perfectly calibrated (e.g., a binormal model with same variances is always convex but it can be uncalibrated), so $\Cal(\Conv(m) \neq \Conv(m)$. This is summarised in Table \ref{tab:calconv}.
If the model is strictly convex, then $\Cal$ is strictly monotonic.
An instance of the function $\Cal$ is the transformation $T \mapsto s=c(T)$ where $c(T) = \frac{\pi_1 f_1(T)}{\pi_0 f_0(T) + \pi_1 f_1(T)}$ as given by Equation (\ref{eq:cT}). This transformation is shown to keep \rl unchanged in the appendix and makes $\cl=0$. 

The previous function is defined for continuous score distributions. The corresponding function for empirical distributions is known as the Pool Adjacent Algorithm (PAV)  \cite{PAV}. Following \cite{pav-rocch}, the $PAV$ function converts any \model $m$ into another calibrated \model $m^*$ such that the following property $s_{b_j} = y_{b_j}$ holds for every segment in its convex hull.

\begin{table}[htdp]
	\centering
	  {\small
		\begin{tabular}{cccc}
		  \hline\hline
			        & Evenly-spaced & Convexification & Perfect Calibration     \\\hline\hline
			Continuous distributions & \Even & \Conv & \Cal    \\\hline      		
			Empirical distributions & \Evend & ROCCH & PAV    \\\hline
      		\end{tabular}
		}
	\caption{Transformations on scores. Perfect calibration implies a convex model but not vice versa.}
	\label{tab:calconv}
\end{table}



It has been shown by \cite{pav-rocch} that isotonic-based calibration \cite{Robertson98}  is equivalent to the PAV algorithm, and closely related to ROCCH, since, for every $m$ and dataset, we have:
\begin{eqnarray}
\bs(PAV(m)) & = &  \cl^{ROC}(PAV(m)) +  \rl^{ROC}(PAV(m)) =  \cl^{ROCCH}(PAV(m)) +  \rl^{ROCCH}(PAV(m)) \nonumber \\ 
           & = &  \rl^{ROC}(PAV(m)) =  \rl^{ROCCH}(PAV(m)) \label{eq:PAV}
\end{eqnarray}

It is also insightful to see that isotonic regression (calibration) is the monotonic function defined as ${\arg\min}_f \sum{(y_i - f(s_i))^2}$, i.e., the monotonic function over the scores which minimises the Brier score. This leads to the same function if we use any other proper scoring function (such as logloss).

The similar expression for the continuous case is
\begin{eqnarray}
\bs(\Cal(m)) & = &  \cl(\Cal(m)) +  \rl(\Cal(m)) =  \rl(\Cal(m)) \label{eq:Conv2}
\end{eqnarray}

Now we analyse what happens with perfectly calibrated \models\ 
for the \scoredriven \thresholdchoice and the \scoreuniform \thresholdchoicemethods. This will help us understand the similarities and differences between the \thresholdchoices and their relation with the \optimal method. Along the way, we will obtain some straightforward, but interesting, results. We start with a basic result:

\begin{theorem}\label{thm:perfect}
If a \model is perfectly calibrated then we have:
\begin{eqnarray}\label{eq:perfect}
\pi_0 \bar{s}_0 & = & \pi_1 (1-\bar{s}_1)
\end{eqnarray}
or equivalently,
\begin{eqnarray}\label{eq:perfect2}
\pi_0 \mae_0 & = & \pi_1 \mae_1
\end{eqnarray}

\end{theorem}

\begin{proof}
For perfectly calibrated \models, we have that for every bin in an invariant partition on the scores we have that $y_{b_j} = s_{b_j}$.
Just taking a partition consisting of one single bin (which is an invariant partition), we have that this is the same as saying that $\pi_1 = \pi_1 \bar{s}_1 + \pi_0 \bar{s}_0$. This leads to $\pi_1 (1-\bar{s}_1) = \pi_0 \bar{s}_0$.
\end{proof}

This is an interesting equation in its own right. It gives a necessary condition for calibration: the extent to which the average score over all examples (which is the weighted mean of per-class averages $\pi_0\m{s}_0+\pi_1\m{s}_1$) deviates from $\pi_1$. 


We now give a first result which connects two \performancemetrics:

\begin{theorem}\label{thm:BS}
If a \model is perfectly calibrated then we have:
\begin{eqnarray*}
\bs = \pi_0 \bar{s}_0 = \pi_1 (1-\bar{s}_1) = \mae/2
\end{eqnarray*}

\end{theorem}

\begin{proof}
We use the continuous decomposition (Theorem \ref{eqBSContdecom}):
\[ \bs  =  \cl + \rl \]
Since it is perfectly calibrated, $\cl=0$.
Then we have:
\begin{eqnarray}
\bs & = &  \rl = \int_{0}^{1}  \frac{\pi_1 f_1(s) \pi_0 f_0(s)} {\pi_0 f_0(s) + \pi_1 f_1(s)}    ds = 
 \int_{0}^{1} \left({\pi_1 f_1(s)}  \right)  \left(1- \frac{\pi_1 f_1(s)} {\pi_0 f_0(s) + \pi_1 f_1(s)}  \right) ds \nonumber \\
& = & \int_{0}^{1} \left({\pi_1 f_1(s)}  - \frac{[\pi_1 f_1(s)]^2} {\pi_0 f_0(s) + \pi_1 f_1(s)}  \right) ds =   \int_{0}^{1} {\pi_1 f_1(s)} ds - \int_{0}^{1}  \frac{[\pi_1 f_1(s)]^2} {\pi_0 f_0(s) + \pi_1 f_1(s)}   ds \nonumber \\
& = &   \pi_1 - \int_{0}^{1}  \frac{\pi_1 f_1(s)} {\frac{\pi_0 f_0(s)}{\pi_1 f_1(s)} + 1}   ds \nonumber 
\end{eqnarray}
Since it is perfectly calibrated, we have:
\[ \frac{f_0(s)}{f_1(s)} = \frac{1-s}{s}\frac{\pi_1}{\pi_0} \]
So:
\begin{eqnarray}
\bs & = &  \pi_1 - \int_{0}^{1}  \frac{\pi_1 f_1(s)} {\frac{\pi_0}{\pi_1}\frac{1-s}{s}\frac{\pi_1}{\pi_0} + 1}   ds =  \pi_1 - \int_{0}^{1}  \frac{s\pi_1 f_1(s)} {(1-s) + s}    ds \nonumber \\
 & = &  \pi_1 - \pi_1 \int_{0}^{1}  s f_1(s)    ds =  \pi_1 (1 - \int_{0}^{1}  s f_1(s))    ds \nonumber 
  = \pi_1 (1 - \bar{s}_1)   \nonumber
\end{eqnarray}
%
\end{proof}

We will now use the expressions for expected loss to analyse where this result comes from exactly. 
In the following result, we see that for a calibrated \model the \optimal threshold $T$ for a given cost proportion $c$ is $T=c$, which is exactly the \scoredriven \thresholdchoicemethod. In other words:

\begin{theorem}\label{thresholdsequivalence}
For a perfectly calibrated \model: 
$\forall c : \Tcosto(c) = \Tcostsd(c) = c$
\end{theorem}

\begin{proof}
By Lemma \ref{CLzero} we have 
\begin{equation}
\frac{f_0(s)}{f_1(s)} = \frac{1-s}{s}\frac{\pi_1}{\pi_0} \label{eq:slope2}
\end{equation}
If we know $c$, we want to find the score $s=T$ where the slope is equal to the slope of a cost isometric \cite{Fla03}. 
The slope of 
a cost isometric is 
\begin{equation*}
\frac{c_1\pi_1}{c_0\pi_0} = \frac{1-c}{c}\frac{\pi_1}{\pi_0} \label{eq:slope1}
\end{equation*}
Setting the two slopes equal implies $T=c$. 
\end{proof}

%
%

And now we can express and relate many of the expressions for the expected loss seen so far. Starting with the expected loss for the \optimal \thresholdchoicemethod, i.e., $\Lcosto$ (which uses $\Tcosto$), we have, from Theorem \ref{thresholdsequivalence}, that $\Tcosto(c) = \Tcostsd(c) = c$ when the \model is perfectly calibrated. Consequently, we have the same as Equation (\ref{eqLcostp}), and since we know that $\bs= \pi_0 \bar{s}_0$ for perfectly calibrated \models, we have:
\begin{eqnarray*}
{\LcostoU} = {\bs} = \pi_0 \bar{s}_0 = \mae/2
\end{eqnarray*}

The following theorem
summarises all the previous results. 

\begin{theorem}\label{theo:perfcal}
For perfectly calibrated \models:
\begin{equation}
{\LcostsdU}= {\LcostoU} = \rl = \frac{{\LcostsuU}}{2} = \frac{\mae}{2} ={\bs} = {\pi_0 \bar{s}_0} = {\pi_1 (1 - \m{s}_1)} 
\end{equation}

\end{theorem}

\begin{proof}
Since $\LcostsdU= \bs$ it is clear that ${\LcostsdU}= {\pi_0 \bar{s}_0}$, as seen above for ${\LcostoU}$ as well.
Additionally, from Theorem \ref{thm:LcostsU}, we have that ${\LcostsuU} = {{\pi_0 \m{s}_0 + \pi_1 (1 - \m{s}_1)}}$, which reduces to $2{\LcostsuU} = 2\bs = 2\pi_0 \m{s}_0$.
We also use the result of Theorem \ref{thm:RL} which states that, in general (not just for perfectly calibrated models), ${\LcostoU(m)} = {\rl(\Conv(m))}$.
\end{proof}
%


All this gives an interpretation of the \optimal \thresholdchoice method as a method which calculates expected loss by assuming perfect calibration.
Note that this is clearly seen by the relation
${\LcostsdU}= {\LcostoU} = \frac{{\LcostsuU}}{2}$, since the loss drops to the half if we use scores to adjust to the operating condition. In this situation, we get the best possible result.

\subsection{Choosing a \thresholdchoicemethod}

It is enlightening to see that many of the most popular classification \performancemetrics are just expected losses by changing the \thresholdchoicemethod and the use of cost proportions or skews.
However, it is even more revealing to see how (and under which conditions) these \performancemetrics can be related (in some cases with inequalities and in some other cases with equalities).
The notion of score transformation is the key idea for these connections, and is more important that it might seem at first sight. Some \thresholdchoicemethods can be seen as a score transformation followed by the \scoredriven \thresholdchoicemethod. Even the fixed \thresholdchoicemethod can be seen as a crisp transformation where scores are set to $1$ if $s_i > t$ and 0 otherwise. 
Another interesting point of view is to see the values of extreme models, such as a model with perfect ranking ($\auc=1$, $ \rl^{ROCCH}=0$) and a random model ($\auc=0.5$, $ \rl^{ROCCH}=0.25$ when $\pi_0=\pi_1$).
Figure \ref{fig:comparison} summarises all the relations found so far and these extreme cases.

%
\begin{figure*}
\centering
%
\begin{fmpage}{0.8\linewidth}
\centering
General relations:
$$\LcostruU = \pi_0\pi_1(1-2\auc)+ \frac{1}{2}  > \LcostrdU = \pi_0\pi_1(1-2\auc)+ \frac{1}{3} \geq \LcostoU =  \rl_{\Conv} $$
$$\LcostsuU = \mae \geq \LcostsdU = \bs \geq \LcostoU =  \rl_{\Conv}$$

\end{fmpage}
\begin{fmpage}{0.8\linewidth}
\centering

If scores are evenly-spaced:
$$\LcostruU = \pi_0\pi_1(1-2\auc)+ \frac{1}{2} = \LcostsuU = \mae = \pi_0 {\bar{s}_0} + {\pi_1 (1 - \m{s}_1)} $$
$$\LcostrdU = \pi_0\pi_1(1-2\auc)+ \frac{1}{3} = \LcostsdU = \bs $$

\end{fmpage}
\begin{fmpage}{0.8\linewidth}
\centering

If scores are perfectly calibrated:
$$ {\LcostsdU}= {\LcostoU} = \rl = \frac{{\LcostsuU}}{2} = \frac{\mae}{2} ={\bs} = {\pi_0 \bar{s}_0} = {\pi_1 (1 - \m{s}_1)} $$
\end{fmpage}
\begin{fmpage}{0.8\linewidth}
\centering

If the model has perfect ranking:
$$ {\LcostruU}= \frac{1}{4} > {\LcostrdU} = \frac{1}{12} > {\LcostoU} = 0 $$
\end{fmpage}
\begin{fmpage}{0.8\linewidth}
\centering

If the model is random (and $\pi_0 = \pi_1$):
$$ {\LcostsuU}= {\LcostsdU}= {\LcostruU}= \frac{1}{2} > {\LcostrdU} = \frac{1}{3} > {\LcostoU} =  \frac{1}{4} $$
\end{fmpage}

\caption{Comparison of losses and \performancemetrics, in general and under several score conditions.}
\label{fig:comparison}
\end{figure*}

The first apparent observation is that $\LcostoU$ seems the best loss, since it derives from the \optimal \thresholdchoicemethod. We already argued in Section \ref{sec:optimal} that this is unrealistic. The result given by Theorem \ref{thm:RL} is a clear indication of this, since this makes expected loss equal to ${ \rl_{\Conv}}$.
Hence, this \thresholdchoicemethod assumes that the calibration which is performed with the convex hull over the training (or a validation dataset) is going to be perfect and hold for the test set.
Figure \ref{fig:comparison} also gives the impression that $\LcostsuU$ and $\LcostruU$ are so bad that their corresponding \thresholdchoicemethods and metrics are useless.
In order to refute this simplistic view, we must realise (again) that not every \thresholdchoicemethod can be applied in every situation. Some require more information or more assumptions than others. Table \ref{tab:times} completes Table \ref{tab:thresholdchoicemethods} to illustrate the point.
If we know the deployment operating condition at evaluation time, then we can fix the threshold and get the expected loss.
If we do not know this information at evaluation time, but we expect to be able to have it and use it at deployment time, then the \scoredriven, \ratedriven and \optimal \thresholdchoicemethods seem the appropriate ones. Finally, if no information about the operating condition is going to be available at any time then the \scoreuniform and the \rateuniform are the (only reasonable) option.

\begin{table}[htdp]
	\centering
		\begin{tabular}{lccc}	 
			\Thresholdchoicemethod               & Fixed & Driven by o.c. & Chosen uniformly   \\\hline
			
			Using scores                          & \scorefixed ($\Tsf$)                                       & \scoredriven ($\Tsd$)     & \scoreuniform ($\Tsu$)     \\
      Using \rates                          & \ratefixed ($\Trf$)                                       & \ratedriven ($\Trd$)       & \rateuniform ($\Tru$)      \\		
      Using \optimal thresholds             &                                                          & \optimal ($\To$)       &      \\
			\hline
      Required information                  & o.c. at evaluation time                                  & o.c. at deployment time       &  no information    \\
			\hline
		\end{tabular}
	\caption{Information which is required (and when) for the seven \thresholdchoicemethods so that they become reasonable. Operating condition is denoted by o.c. }
	\label{tab:times} 
\end{table}

From the cases shown in \ref{tab:times}, the methods driven by the operating condition require further discussion.
The relations shown in Figure \ref{fig:comparison} illustrate that, in addition to the \optimal \thresholdchoicemethod, the other two methods that seem more competitive are the \scoredriven and the \ratedriven. 
One can argue that the \ratedriven \thresholdchoice has an expected loss which is always greater than $1/12$ (if $\auc=1$, we get $-1/4 + 1/3$), while the others can be 0.
But things are not so clear-cut. 

\begin{itemize}
\item The \scoredriven \thresholdchoicemethod considers that the scores are estimated probabilities and that they are reliable, in the tradition of proper scoring rules. So it just uses these probabilities to set the thresholds.
\item The \ratedriven \thresholdchoicemethod completely ignores the scores and only considers their order. It assumes that the ranking is reliable while the scores are not accurate probabilities. It derives the thresholds using the predictive positive rate. It can be seen as the \scoredriven \thresholdchoicemethod where the scores have been set evenly-spaced by a transformation.
\item The \optimal \thresholdchoicemethod also ignores the scores completely and only considers their order. It assumes that the ranking is reliable while the scores are not accurate probabilities. However, this method derives the thresholds by keeping the order and using the slopes of the segments of the convex hull (typically constructed over the training dataset or a validation dataset). It can be seen as the \scoredriven \thresholdchoicemethod where the scores have been calibrated by the PAV method. 
\end{itemize}

\noindent Now that we better understand the meaning of the \thresholdchoicemethods we may state the difficult question more clearly: given a model, which \thresholdchoicemethod should we use to make classifications? The answer is closely related to the calibration problem. Some theoretical and experimental results \cite{Robertson98, PAV, Platt, ZE01, Zadrozny, NC05, niculescu2005predicting, Handbook, gebel2009multivariate} have shown that the PAV method (also known as isotonic regression) is not frequently the best calibration method.
Some other calibration methods could do better, such as Platt's calibration or binning averaging. 
In particular, it has been shown that ``isotonic regression is more prone to overfitting, and thus performs worse than Platt scaling, when data is scarce" \cite{niculescu2005predicting}.
Even with a large validation dataset which allows the construction of an accurate ROC curve and an accurate convex hull, the resulting choices are not necessarily optimal for the test set, since there might be problems with outliers \cite{ruping2006robust}. 
In fact, if the validation dataset is much smaller (or biased) than the training set, the resulting probabilities can be even worse than the original probabilities, as it may happen with cross-validation. So, 
we have to feel free to use other (possibly better) calibration methods instead and do not stick to the PAV method just because it is linked to the \optimal \thresholdchoicemethod.

So the question of whether we keep the scores or not (and how we replace them in case) depends on our expectations on how well-calibrated the \model is, and whether we have tools (calibration methods and validation datasets) to calibrate the scores.

But we can turn the previous question into a much more intelligent procedure. Calculating the three expected losses discussed above (and perhaps the other \thresholdchoicemethods as well) provides a rich source of information about how our models behave. This is what performance metrics are all about. It is only after the comparison of all the results and the availability of (validation) datasets when we can make a decision about which \thresholdchoicemethod to use.

This is what we did with the example shown in Table \ref{tab:models2} in Section \ref{sec:intro}.
We evaluated the model for several \thresholdchoicemethods 
and from there we clearly saw which models were better calibrated and we finally made a decision about which model to use and with which \thresholdchoicemethods.

In any case, note that the results and comparisons shown in Figure \ref{fig:comparison} are for expected loss; the actual loss does not necessarily follow these inequalities. In fact, the expected loss calculated over a validation dataset may not hold over the test dataset, and even some \thresholdchoicemethods we have discarded from the discussion above (the fixed ones or the \scoreuniform and \rateuniform, if probabilities or rankings are very bad respectively) could be better in some particular situations.

\section{Discussion}\label{sec:conclusion}

This paper builds upon the notion of \thresholdchoicemethod and the expected loss we can obtain for a range of cost proportions (or skews) for each of the \thresholdchoicemethods we have investigated. The links between \thresholdchoicemethods, between \performancemetrics, in general and for specific score arrangements, have provided us with a much broader (and more elaborate) view of classification \performancemetrics and the way thresholds can be chosen.
In this last section we link our results to the extensive bulk of work on classification evaluation and analyse the most important contributions and open questions which are derived from this paper.

\subsection{Related work}

One decade ago there was a scattered view of classification evaluation. Many \performancemetrics existed and it was not clear what their relationships were.
One first step in understanding some of these \performancemetrics in terms of costs was the notion of cost isometrics \cite{Fla03}.
With cost isometrics, many classification metrics (and decision tree splitting criteria) are characterised by 
its skew landscape, i.e., the slope of its isometric at any point in the ROC space.
Another comprehensive view was the empirical evaluation made in \cite{PRL09}.
The analysis of Pearson and Spearman correlations between 18 different \performancemetrics
shows the pairs of metrics for which the differences are significant. 

In addition to these, there have been three lines of research in this area which provide further pieces to understand the whole picture.

\begin{itemize}
\item First, the notion of `proper scoring rules' (which was introduced in the sixties, see e.g. \cite{murphy1970scoring}), has been developed to a degree \cite{BSS05} in which it has been shown that the Brier score (MSE loss), logloss, boosting loss and error rate (0-1 loss) are all special cases of an integral over a Beta density, and that all these \performancemetrics can be understood as averages (or integrals), at least theoretically, over a range of cost proportions (see e.g. \cite{gneiting2007strictly,reid2010composite,brummer2010thesis}), so generalising the early works by Murphy on probabilistic predictions when cost-loss ratio is unknown (\cite{murphy1966note} and \cite{murphy1969measures}). Additionally, further connections have been found between proper scoring rules and distribution divergences ($f$-divergences and Bregman divergences) \cite{reid2011information}.

\item Second, the translation of the Brier decomposition using ROC curves \cite{flach-matsubara-ecml07} suggests a connection between the Brier score and ROC curves, and most specially between refinement loss and \auc, since both are \performancemetrics which do not require the magnitude of the scores of the \model.

\item Third, an important coup d'effet has been given by Hand \cite{hand2009measuring}, stating that the \auc cannot be used as a \performancemetric for evaluating \models (for a range of cost proportions) because the distribution for these cost proportions depends on the \model. This seemed to suggest a definitive rupture between ranking quality and classification performance over a range of cost proportions. 

\end{itemize}

\noindent Each of the three lines mentioned above provides a partial view of the problem of classifier evaluation, and suggests that some important connections between \performancemetrics were waiting to be unveiled.
The starting point of this unifying view is that all the previous works above worked with only two \thresholdchoicemethods, which we have called the \scoredriven \thresholdchoicemethod and the \optimal \thresholdchoicemethod. Only a few works mention these two \thresholdchoicemethods together. For instance, Drummond and Holte \cite{drummond-and-Holte2006} talk about `selection criteria' (instead of `\thresholdchoicemethods') and they distinguish between `performance-independent' selection criteria and `cost-minimizing' selection criteria. Hand (personal communication) says that `\cite{hand2009measuring} (top of page 122) points out
that there are situations where one might choose thresholds independently of cost, and go into
more detail in \cite{hand2010evaluating}'. This is related to the fixed \thresholdchoicemethod, or the \rateuniform and \scoreuniform \thresholdchoicemethods used here.
Finally, in \cite{ICML11CoherentAUC} we explore a new \rateuniform \thresholdchoicemethod while in \cite{ICML11Brier} we explore the \scoredriven \thresholdchoicemethod. 

The notion of proper scoring rule works with the \scoredriven \thresholdchoicemethod.
This implies that this notion cannot be applied to \auc ---\cite{reid2011information} connects the area under the convex hull (\auch) with other proper scoring rules but not \auc--- and to RL. 
As a consequence, the Brier score, log-loss, boosting loss and error rate would only be minor choices depending
on the information about the distribution of costs. 

David Hand \cite{hand2009measuring} takes a similar view of the cost distribution, as a choice that depends on the information we may have about the problem, but makes an important change over the tradition in proper scoring rules tradition. He considers `\optimal thresholds' (see Equation (\ref{eqTcosto})) instead of the \scoredriven choice.
With this \thresholdchoicemethod, David Hand is able to derive \auc (or yet again \auch) as a measure of aggregated classification performance, but the distribution he uses (and criticises) depends on the \model itself. Then he defines a new \performancemetric which is proportional to the area under the \optimal cost curve.

\subsection{Conclusions and future work}

As a conclusion, if we want to evaluate a \model for a wide range of operating conditions (i.e., cost proportion or skews), we have to determine first which \thresholdchoicemethod is to be used. If it is fixed because we have a non-probabilistic classifier or we are given the actual operating condition at evaluation time, then we get accuracy (and macro-accuracy) as a good \performancemetric. 
If we have no access to the operating condition at evaluation time but neither do we at deployment time, then the \scoreuniform and the \rateuniform may be considered, with \mae and \auc as corresponding performance metrics. Finally, in the common situation when we do not know the operating condition at evaluation time but we expect that it will be known and used at deployment time, then we have more options.
If a \model has no reliable scores or probability estimations, we recommend the refinement loss ($ \rl_{\Conv}$, which is equivalent to area under the \optimal cost curve) if thresholds are being chosen using the convex hull of a reliable ROC curve, or, alternatively, we recommend the area under the ROC curve ($\auc$) if the estimation of this convex hull is not reliable enough to choose thresholds confidently. More readily, if a \model has reliable scores because it is a good probability estimator or it has been processed by a calibration method, then we recommend to choose the thresholds according to scores. In this case, the corresponding \performancemetric is the Brier score. 


From this paper, now we have a much better understanding on the relation between the Brier score, the \auc and refinement loss. We also know much better what is happening when models are not convex and/or not calibrated.
In addition, we find that using evenly-spaced scores, we get that the Brier score and the \auc are linearly related. Furthermore, we see that if the \model is perfectly calibrated, the expected loss using the \scoredriven \thresholdchoicemethod equals the \optimal \thresholdchoicemethod.


As said in the introduction, this paper works on a different dimension, because, instead of changing the cost or skew distribution as the work on proper scoring rules has done, we change the \thresholdchoicemethod. This suggests that some other combinations could be explored, such as Hand did with his measure $H$ \cite{hand2009measuring}, when using the $B_{2,2}$ distribution for the \optimal \thresholdchoicemethod instead of the uniform ($B_{1,1}$) distribution. We think that the same thing could be done with the \ratedriven \thresholdchoicemethod, possibly leading to new variants of the \auc.


The collection of new findings introduced in this paper leads to many other avenues to follow and some questions ahead. For instance, the duality between cost proportions and skews suggests that we could work with loglikelihood ratios as well. Also, there is always the problem of multiclass evaluation. This is as challenging as interesting, since there are many more \thresholdchoicemethods in the multiclass case and the corresponding expected losses could be connected to some multiclass extensions of the binary \performancemetrics.
Finally, more work is needed on the relation between the ROC space and the cost space, and the representation of all these expected losses in the latter space. The notion of Brier curve \cite{ICML11Brier} is a first step in this direction, but all the new \thresholdchoicemethods also lead to other curves.


\vskip 0.2in
\bibliographystyle{plain}

\bibliography{biblio}


\newpage

\appendix

\section{Appendix: proof of Theorem \ref{thm:RL} and examples relating $\LcostoU$ and $\rl$ }

In this appendix, we give the proof for Theorem \ref{thm:RL} in the paper, along with some examples which show how the correspondence between $\LcostoU(m)$ and $\rl(m)$ goes.
The theorem works with convex models as given by definition \ref{def:Convex}. 

In this appendix, we will use:

$$c(T) = \frac{\pi_1 f_1(T)}{\pi_0 f_0(T) + \pi_1 f_1(T)}$$

\noindent
throughout, as introduced by eq. (\ref{eq:cT}). Sometimes we will use subindices for $c(T)$ depending on the model we are using.
We will also use ${slope(T)} = {\pi_1 f_1(T)} = \frac{\pi_1}{\pi_0}\left(\frac{1}{c(T)}-1\right)$.
A convex model is the same as saying that $c(T)$ is non-decreasing or that $slope(T)$ is non-increasing.

We use $c^{-1}(s)$ for the inverse of $c(T)$ (wherever it is well defined).
We will use the following transformation $T \mapsto s=c(T)$ and the resulting model will be denoted by $m^{(c)}$. We will use $s$, $c$ or $\sigma$ for elements in the codomain of this transformation (cost proportions or scores between 0 and 1) and we will use $T$ or $\tau$ for elements in the domain.

For continuous and strictly convex models for which $c(0)=0$ and $c(1)=1$, the proof is significantly simpler. In general, for any convex model, including discontinuities and straight segments, things become a little bit more elaborate, as we see below.







\subsection{Analysing pointwise equivalence}

One way of showing that two aggregated measures are equal is to show that these measures are pointwise equal. However, this is not the right way here, since this is not the case in general (for every possible convex model $m$). In order to understand this better, we start by drawing the loss (in terms of cost proportion $c$, as in cost curves) and, using the same \xaxis,  we also draw measures on the thresholds, especially when these thresholds go from 0 to 1.
In particular, we use the same idea that was introduced in \cite{ICML11Brier} for \emph{Brier curves}. In this appendix, we will only show the \cl term of the Brier curve, so we can call them \emph{refinement curves}.
Finally, a third curve will also be used, which can be understood as the {\rl} of the calibrated model using the transformation $c(T)$ on the scores. As we will see, calibration keeps the $\rl$ as well, but this is again achieved in a non-pointwise manner. This leads to three different curves in the figures which follow:

\begin{enumerate}
\item The loss of the original convex model $m$, $\LcostoU(m)$, which will be shown in brown.
\item The loss of the $m^{(c)}$, the model calibrated with $c(T)$, i.e., $\LcostoU(m^{(c)})$, denoted by $\Lambda$, which will be shown in purple.   
\item The refinement loss of model $m$, $\rl(m)$, which will be shown in green.
\end{enumerate}

\noindent In this appendix, we develop these three items and their correspondences, which will be the key to find the proof for the theorem. In fact, we go from (1) to (2), and then, from (2) to (3).

We start with figure \ref{fig:perfectlycalibrated}, where the model is perfectly calibrated (and hence strictly convex), and the three curves match pointwise.


\begin{figure}
\centering
\includegraphics[width=0.9\textwidth]{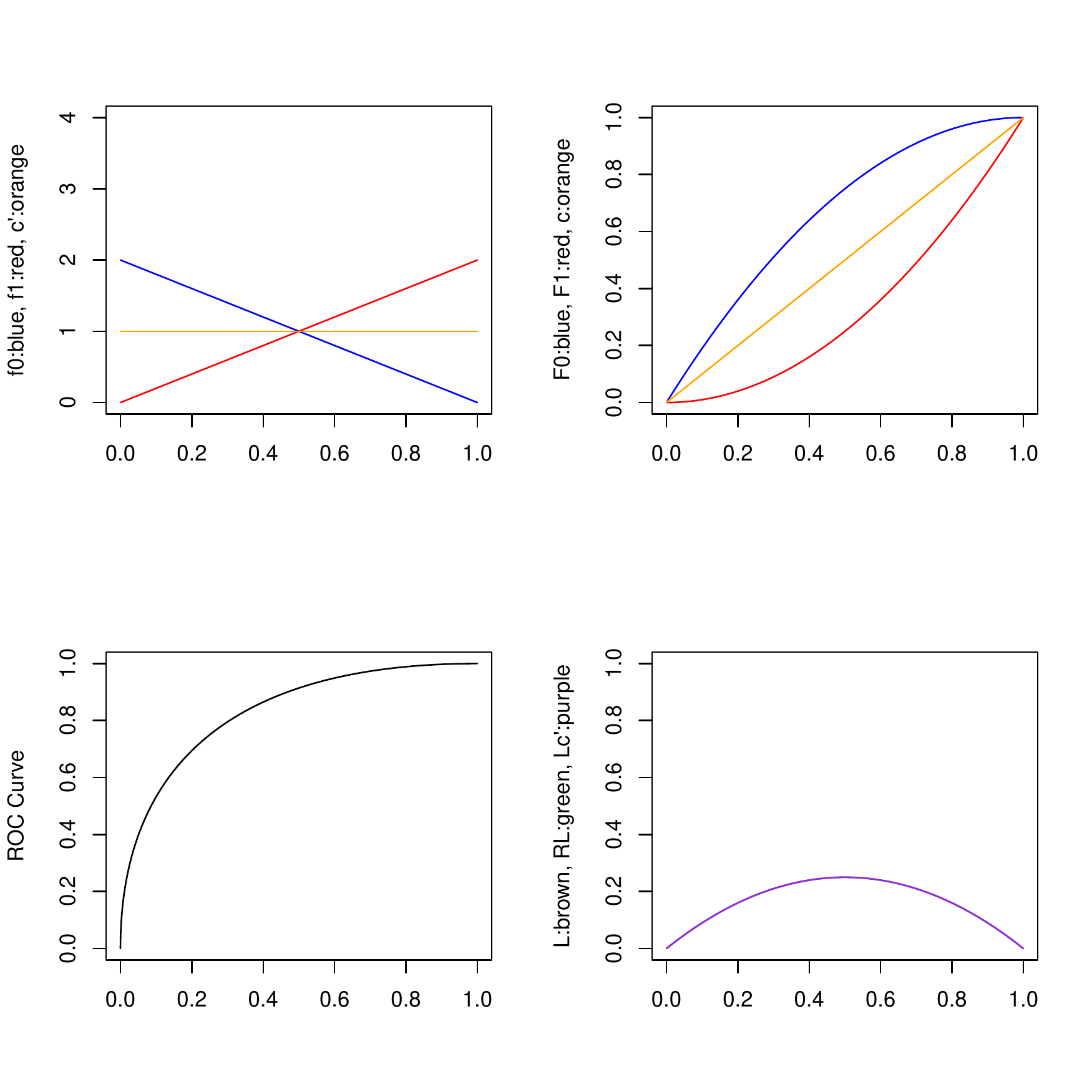} 
\caption{Here we have a perfectly calibrated model using two triangular distributions.
We see that $c(T)$ (in orange) in this case is $c(T)=T$ and its density function $c'(T)=1$. The ROC curve is shown at the bottom left plot, which is strictly convex.
The plot on the bottom right shows that the three curves are pointwise identical (this is trivial from Lemma \ref{lemma:prime} for the two expressions of loss, and it is also easy to see for \rl using Theorem \ref{thm:LcostpUequalsBS2}). 
The intervals are $I_\sigma = \{(0, 1) \}$, $I_\tau = \{(0, 1) \}$, where the only interval is bijective.}
\label{fig:perfectlycalibrated}
\end{figure}

Figure \ref{fig:noncalibrated} shows a model which has exactly the same ranking (and ROC curve) as in figure \ref{fig:perfectlycalibrated}. However, here, the model is not perfectly calibrated. We see now that two of the three curves match pointwise, namely the green one (\rl) and the purple ($\Lambda$). Yet again, the three curves still match in total area.


\begin{figure}
\centering
\includegraphics[width=0.9\textwidth]{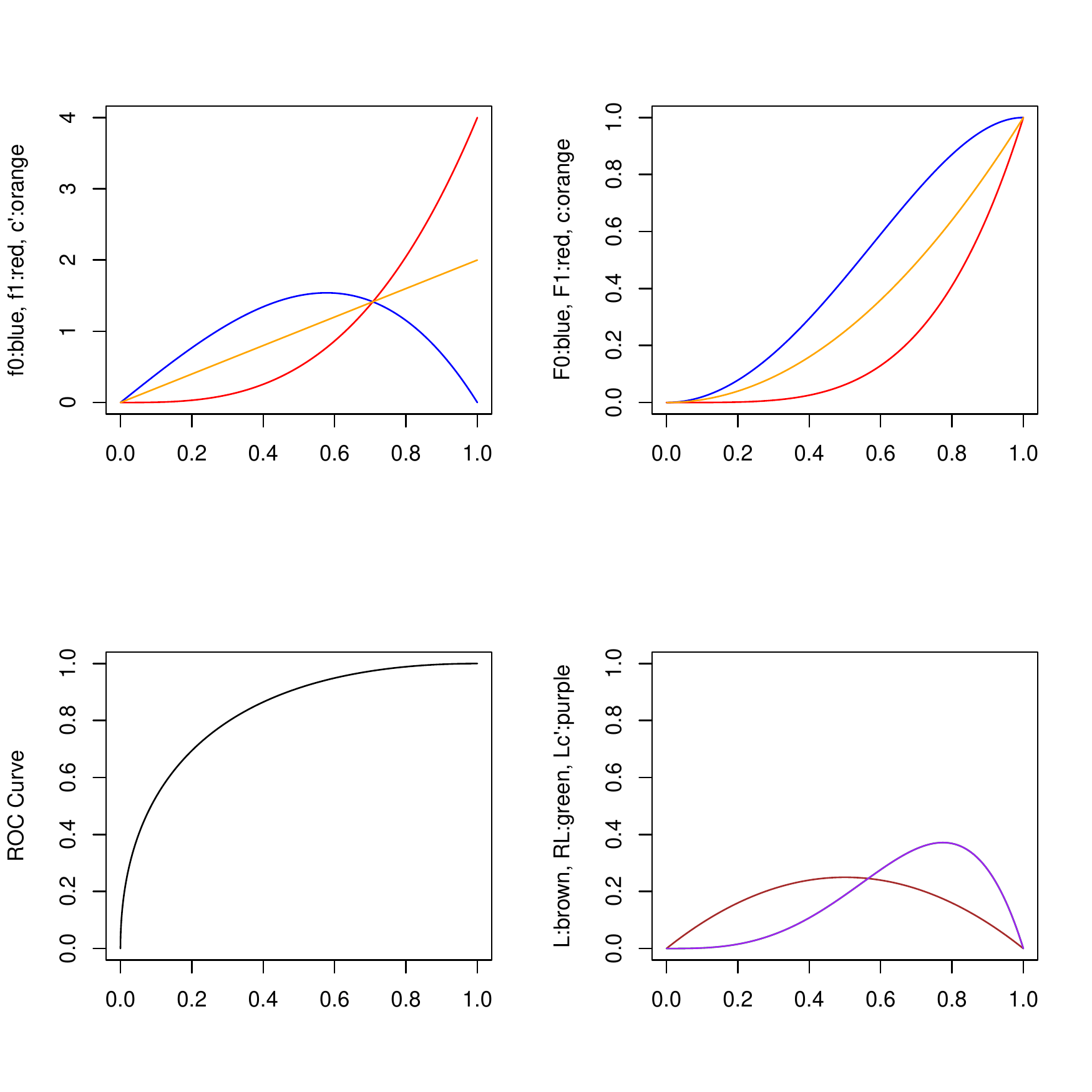} 
\caption{Here we have a non-calibrated model with the same ranking as Figure \ref{fig:perfectlycalibrated}.
We see that $c'(T)$ is linear. The ROC curve is identical to the previous case.
The plot on the bottom right shows that the three curves match in the areas but only two in their shapes as well: the green and purple curves are equal, since \rl is pointwise equal to $\acute{\Lambda}$. 
The intervals are $I_\sigma = \{(0, 1) \}$, $I_\tau = \{(0, 1) \}$, where the only interval is bijective.}
\label{fig:noncalibrated}
\end{figure}

Figure \ref{fig:noncalibrated2} shows another model which has exactly the same ranking (and ROC curve) as in figure \ref{fig:perfectlycalibrated}, but again the model is not perfectly calibrated. We see now, none of the three curves match pointwise (while the three still match in total area).


\begin{figure}
\centering
\includegraphics[width=0.9\textwidth]{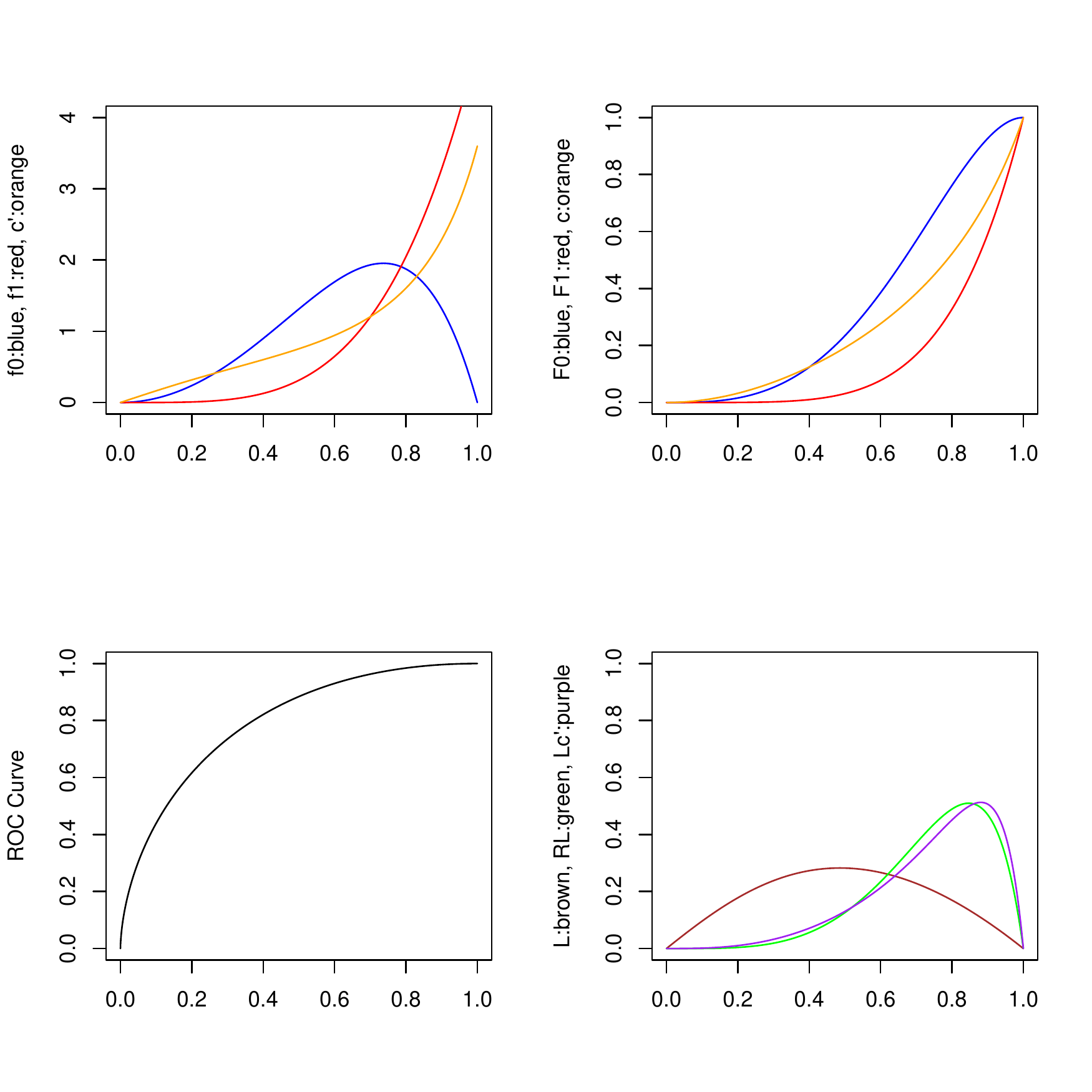} 
\caption{Here we have a non-calibrated model with the same ranking as Figure \ref{fig:perfectlycalibrated}.
Now we see that $c'(T)$ is not linear. The ROC curve is identical to the previous two cases.
The plot on the bottom right shows that the three curves match in the areas but not point wise.  
The intervals are $I_\sigma = \{(0, 1) \}$, $I_\tau = \{(0, 1) \}$, where the only interval is bijective.}
\label{fig:noncalibrated2}
\end{figure}

\subsection{Intervals}

Since the model is convex, we know that $c(T)$ is monotone, more precisely, non-decreasing. We can split the codomain and domain of this function into intervals. Intervals in the codomain of thresholds will be represented with the letter $\tau$ and intervals in the domain of cost proportions or scores between 0 and 1 will be denoted by letter $\sigma$. The series of intervals are denoted as follows:
\begin{eqnarray*}
 I_\sigma & = & (\sigma_0, \sigma_1), (\sigma_1, \sigma_2) \dots (\sigma_i, \sigma_{i+1}) \dots (\sigma_{n-1}, \sigma_n)  \\
& & \:\:\:\:\:\:\:\:\:\:\:\:\:\:\:\:\:\:\:\:\:\:\:\: \Uparrow c(\tau) \:\:\:\:\:\:\:\:\:\:\:\:\:\:\:\:\:\:\: \Downarrow c^{-1}(\sigma) \\
 I_\tau & = & (\tau_0, \tau_1), (\tau_1, \tau_2) \dots (\tau_i, \tau_{i+1}) \dots (\tau_{n-1}, \tau_n)  
\end{eqnarray*}

\noindent where $\sigma_0 = 0$, $\sigma_n = 1$, $\tau_0 = -\infty$ and $\tau_n = \infty$.
Even though we cannot make a bijective mapping for every point, we can construct a bijective mapping between $I_\sigma$ and $I_\tau$.
Because of this bijection, we may occasionally drop the subindex for $I_\sigma$ and $I_\tau$.

We need to distinguish three kinds of intervals:

\begin{itemize}
\item Intervals where $c(T)$ is strictly increasing, denoted by $\acute{I}$. We call these intervals {\em bijective}, since $c(T)$ is invertible.
These correspond to non-straight parts of the ROC curve. Each point inside these segments is optimal for one specific cost proportion.
\item Intervals where $c(T)$ is constant, denoted by $\bar{I}$. We call these non-injective intervals {\em constant}. 
These correspond to straight parts of the ROC curve. All the points inside these segments are optimal for just one cost proportion, and we only need to consider any of them (e.g., the extremes).
\item Intervals in the codomain where no value $T$ for $c(T)$ has an image, denoted by $\dot{I}$. We call these `intervals' {\em singular}, and address non-surjectiveness. In the codomain they may usually correspond to one single point, but also can correspond to an actual interval when the density functions are 0 for some intervals in the codomain. 
In the end, these correspond to discontinuous points of the ROC curve. The points at $(0,0)$ and $(1,1)$ are generally (but not always) discontinuous. These points are optimal for many cost proportions. 
\end{itemize}

Table \ref{tab:intervals} shows how these three kinds of intervals work.

\begin{table}[htdp]
\begin{center}
\begin{tabular}{ccc}\hline 
bijective & constant & singular \\ 
$\acute{I}$ & $\bar{I}$ & $\dot{I}$ \\ \hline 
$]\sigma_i, \sigma_{i+1}[$ & $[\sigma_i, \sigma_{i+1}]$ & $]\sigma_i, \sigma_{i+1}[$  \\
$\uparrow\uparrow\uparrow\uparrow\uparrow$ & $\nearrow\nwarrow$ &  $\nwarrow\nearrow$ \\
$]\tau_i, \tau_{i+1}[$ & $]\tau_i, \tau_{i+1}[$ & $[\tau_i, \tau_{i+1}]$  \\
\hline\end{tabular}
\end{center}
	\caption{Illustration for the three types of intervals.}
	\label{tab:intervals}
\end{table}%

Figures \ref{fig:strictlyconvex} to \ref{fig:discontinuous} show several examples where the number and type of intervals vary, as they are detailed in the captions of the figures.

Now we are ready to get some results:

\begin{lemma}\label{lemma:prime}
If the model $m$ is convex, we have that minimal expected loss can be expressed as:

\begin{equation}
\LcostoU(m)  = \acute{\Lambda}(m) + \dot{\Lambda}(m) \label{eq:LcT}
\end{equation}

\noindent where:

\begin{equation}
\acute{\Lambda}(m) = \sum_{]\tau_i,\tau_{i+1}[ \in \acute{I}_\tau} \int_{\tau_i}^{\tau_{i+1}} 2c(T) \pi_0(1-F_0(T)) + 2(1-c(T))\pi_1 F_1(T) \} c'(T) dT  \label{eq:LcTbijective}
\end{equation}

\noindent where $c'(T)$ is the derivative of $c(T)$ and:

\begin{eqnarray}
 \dot{\Lambda}(m) & = & \sum_{]\sigma_i,\sigma_{i+1}[ \in \dot{I}_\sigma} \int_{\sigma_i}^{\sigma_{i+1}} 2c \pi_0(1-F_0(\tau_i)) + 2(1-c)\pi_1 F_1(\tau_i) \} dc \label{eq:LcTsingular1} \\
 & = & \sum_{]\sigma_i,\sigma_{i+1}[ \in \dot{I}_\sigma} \left\{ \pi_0(1-F_0(\tau_i))   ({{\sigma_{i+1}}^2}-{{\sigma_{i}}^2}) + \pi_1 F_1(\tau_i)  (2{\sigma_{i+1}}-{{\sigma_{i+1}}^2}-2{\sigma_i}+{{\sigma_i}^2})    \right\}  \label{eq:LcTsingular}
\end{eqnarray}

\noindent Note that the constant intervals in $\bar{I}_\sigma$ are not considered (their loss is 0). 
\end{lemma}

\label{eq:Tcosto}

\begin{proof}
We take the expression for optimal loss from Equation (\ref{eq:Loriginal}):
%
\begin{eqnarray}
\LcostoU 				& = & \int^{1}_{0} \underset{t}\min\{ 2c\pi_0(1-F_0(t)) + 2(1-c)\pi_1 F_1(t) \} dc \label{eq:Loriginal2}  
\end{eqnarray}

\noindent In order to calculate the minimum, we make the derivative of the min expression equal to 0:

\begin{eqnarray*}
2c\pi_0(0-f_0(t)) + 2(1-c)\pi_1 f_1(t) = 0\\
-2c \cdot \frac{\pi_0}{\pi_1}{slope(T)} + 2(1-c) = 0\\
\frac{\pi_0}{\pi_1}{slope(T)}  = \frac{1-c}{c}\\
\frac{1}{c(T)}-1  = \frac{1-c}{c}\\ 
c(T) = c
\end{eqnarray*}

\noindent We now check the sign of the second derivative, which is:

\begin{eqnarray*}
-2c \cdot \frac{\pi_0}{\pi_1}slope'(t) =
-2c \times (\frac{1}{c(T)}-1)' =
-2c \frac{-c'(T)}{{c(T)}^2} =
2c  \frac{c'(T)}{{c(T)}^2} 
\end{eqnarray*}

For the bijective intervals $\acute{I}_\sigma$, where the model is strictly convex and $c(T)$ is strictly decreasing, its derivative is $> 0$.
Also, $c$ is always between 0 and 1, so the above expression is positive, and it is a minimum.
And this cannot be a `local' minimum, since the model is convex.

For the constant intervals $\bar{I}_\sigma$ where the model is convex (but not strictly), this means that $c(T)$ is constant, and its derivative is 0. That means that the minimum can be found at any point $T$ in these intervals $]\tau_i, \tau_{i+1}[$ for the same $[\sigma_i = \sigma_{i+1}]$. But their contribution to the loss will be 0, as can be seen since $c'(T)$ equals $0$.

For the singular intervals $\dot{I}_\sigma$, on the contrary, all the values in each interval $]\sigma_i, \sigma_{i+1}[$ will give a minimum for the same $[\tau_i = \tau_{i+1}]$.

So we decompose the loss with the bijective and singular intervals only:

\begin{equation}
\LcostoU(m)  = \acute{\Lambda}(m) + \dot{\Lambda}(m)
\end{equation}

For the strictly convex (bijective) intervals, we now know that the minimum is at $c(T) = c$, and $c(T)$ is invertible. We can use exactly this change of variable over Equation (\ref{eq:Loriginal2}) and express this for the series of intervals $]\tau_i,\tau_{i+1}[$. 

\begin{eqnarray*}
\acute{\Lambda}(m)   & = & \sum_{]\tau_i,\tau_{i+1}[ \in \acute{I}_\sigma} \int_{\tau_i}^{\tau_{i+1}} 2c(T) \pi_0(1-F_0(T)) + 2(1-c(T))\pi_1 F_1(T) \} c'(T) dT  
\end{eqnarray*}

\noindent which corresponds to Equation (\ref{eq:LcTbijective}). Note that when there is only one bijective interval (the model is continuous and strictly convex), we have that there is only one integral in the sum and its limits go from $c^{-1}(0)$ to $c^{-1}(1)$, which in some cases can go from $-\infty$ to $\infty$, if the scores are not understood as probabilities.

For the singular intervals, we can work from Equation (\ref{eq:Loriginal2}):

\begin{eqnarray*}
 \dot{\Lambda}(m) & = & \sum_{]\sigma_i,\sigma_{i+1}[ \in \dot{I}_\sigma} \int_{\sigma_i}^{\sigma_{i+1}} \min\{ 2c\pi_0(1-F_0(t)) + 2(1-c)\pi_1 F_1(t) \}  dc  
\end{eqnarray*}

As said, all the values in each interval $]\sigma_i, \sigma_{i+1}[$ will give a minimum for the same $[\tau_i = \tau_{i+1}]$, so this reduces to:

\begin{eqnarray*}
 \dot{\Lambda}(m) & = &  \sum_{]\sigma_i,\sigma_{i+1}[ \in \dot{I}_\sigma} \int_{\sigma_i}^{\sigma_{i+1}} \{ 2c \pi_0(1-F_0(\tau_i)) + 2(1-c)\pi_1 F_1(\tau_i) \} dc \\
                  & = &  2 \sum_{]\sigma_i,\sigma_{i+1}[ \in \dot{I}_\sigma} \left\{ \pi_0(1-F_0(\tau_i)) \int_{\sigma_i}^{\sigma_{i+1}}  c dc + \pi_1 F_1(\tau_i) \int_{\sigma_i}^{\sigma_{i+1}} (1-c) dc \right\} \\ 
                  & = &  2 \sum_{]\sigma_i,\sigma_{i+1}[ \in \dot{I}_\sigma} \left\{ \pi_0(1-F_0(\tau_i))   \left[\frac{c^2}{2}\right]_{\sigma_i}^{\sigma_{i+1}} + \pi_1 F_1(\tau_i)  \left[c-\frac{c^2}{2}\right]_{\sigma_i}^{\sigma_{i+1}}   \right\} \\ 
									& = &  \sum_{]\sigma_i,\sigma_{i+1}[ \in \dot{I}_\sigma} \left\{ \pi_0(1-F_0(\tau_i))   ({{\sigma_{i+1}}^2}-{{\sigma_{i}}^2}) + \pi_1 F_1(\tau_i)  (2{\sigma_{i+1}}-{{\sigma_{i+1}}^2}-2{\sigma_i}+{{\sigma_i}^2})    \right\}  
\end{eqnarray*}

\noindent which corresponds to Equation (\ref{eq:LcTsingular}).
\end{proof}

This gives us the connection between the brown curves (the loss) and the purple curves (the loss of the transformed model). In the figures, we only show the term $\acute{\Lambda}(m)$, and not the other term $\dot{\Lambda}(m)$.

\subsection{More examples}

Now we will see examples with different kinds of intervals, and where the correspondence between $\rl$ and loss is more subtle.

We can illustrate what happens with Lemma \ref{lemma:prime} in figure \ref{fig:strictlyconvex}.
As we see here, the model is strictly convex (as shown by its ROC curve and the strictly increasing  blue $c(T)$ in the top right plot). We can apply the lemma and we see that the brown curve (the loss) and the purple curve (corresponding to  $\acute{\Lambda}(m)$) in the bottom right plot have the same area. Note that the limits are here between 0 and 1 for cost proportions $c$ and thresholds $T$. In fact, we have that $c^{-1}(0)=0$ and $c^{-1}(1)=1$. This explains why $\dot{\Lambda}(m)=0$ in this case.
There is another curve, and it is shown in green, and it is the refinement loss of the model. Showing that the area under this curve is also equal is the purpose of the rest of this appendix.


\begin{figure}
\centering
\includegraphics[width=0.9\textwidth]{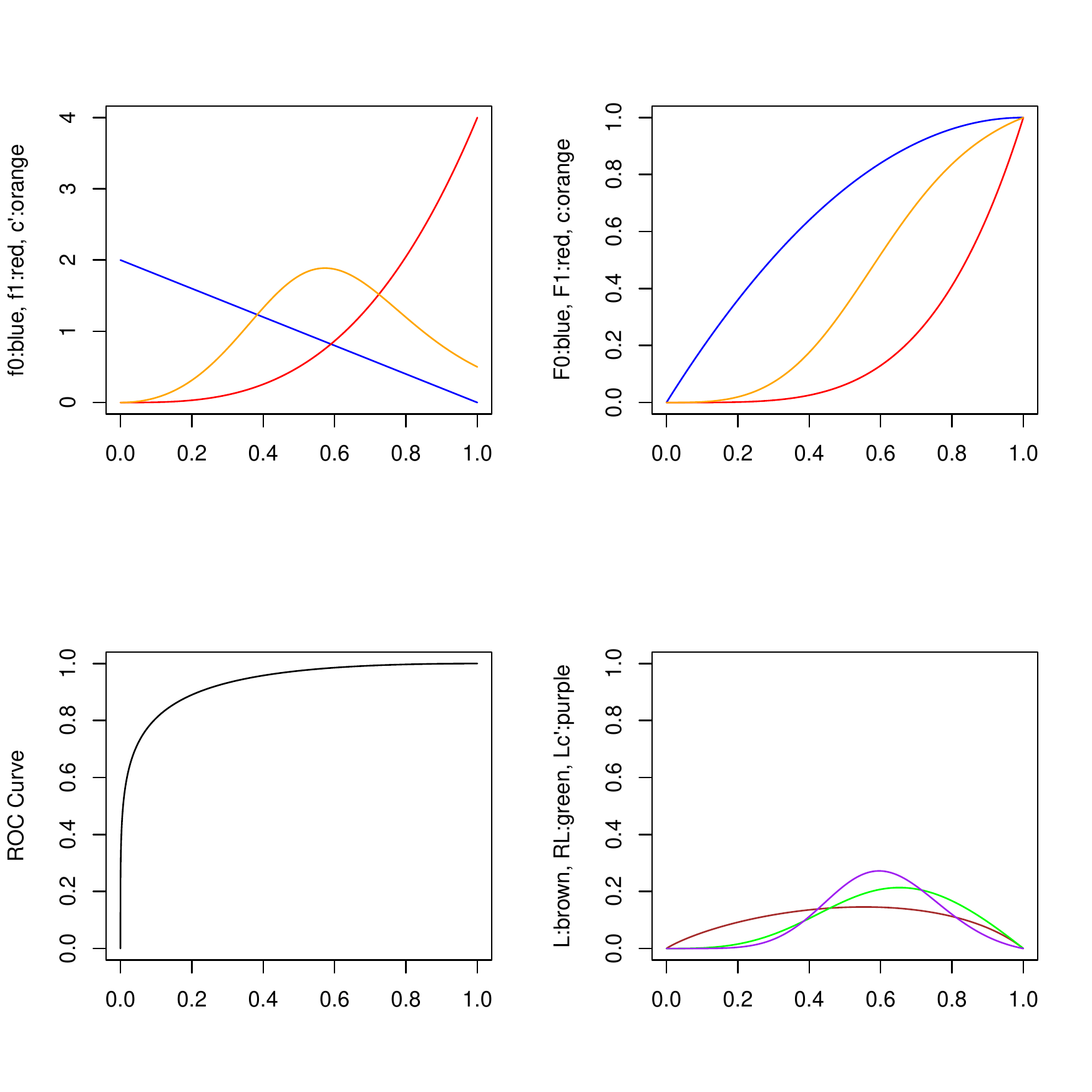} 
\caption{Here we have a model with probability density functions $f_0(x) = 2(1-x)$ and $f_1(x)= 4x^3$, with corresponding cumulative distribution functions $F_0(x)= 1 - (1-x)^2$ and $F_1(x) = x^4$. We also show the functions $c(x) = 4x^3 / (2(1-x) + 4x^3)$ and $c'(x) = (2 (3 - 2x )x^2)/(-1 + x - 2x^3)^2$. These six functions are shown on the the top left plot (densities) and top right plot (cumulative). We see that $c(T)$ (in orange) in this case is strictly increasing. The ROC curve is shown at the bottom left plot, which is strictly convex. The interesting plot is on the bottom right.
Here we see that the overall loss (which is $0.10245$ in this case) can be calculated in three different ways. The original curve (in brown) is given by Equation (\ref{eq:Loriginal2}).
A different curve (in purple) is given by $\acute{\Lambda}$ in Equation (\ref{eq:LcTbijective}). The equivalence of the area under these two curves when the model is convex is what Lemma \ref{lemma:prime} shows, because in this case $\dot{\Lambda} = 0$.
But there is yet another possible curve for convex models, and it is shown in green, and it is the refinement loss of the model.
The intervals are $I_\sigma = \{(0, 1) \}$, $I_\tau = \{(0, 1) \}$, where the only interval is bijective. }
\label{fig:strictlyconvex}
\end{figure}

We can see a similar picture for a model which is also convex (but not strictly) in  figure \ref{fig:nonstrictlyconvex}.
However, here we see plateaus in some of the curves and we see that the warping of the curves is much more interesting.


\begin{figure}
\centering
\includegraphics[width=0.9\textwidth]{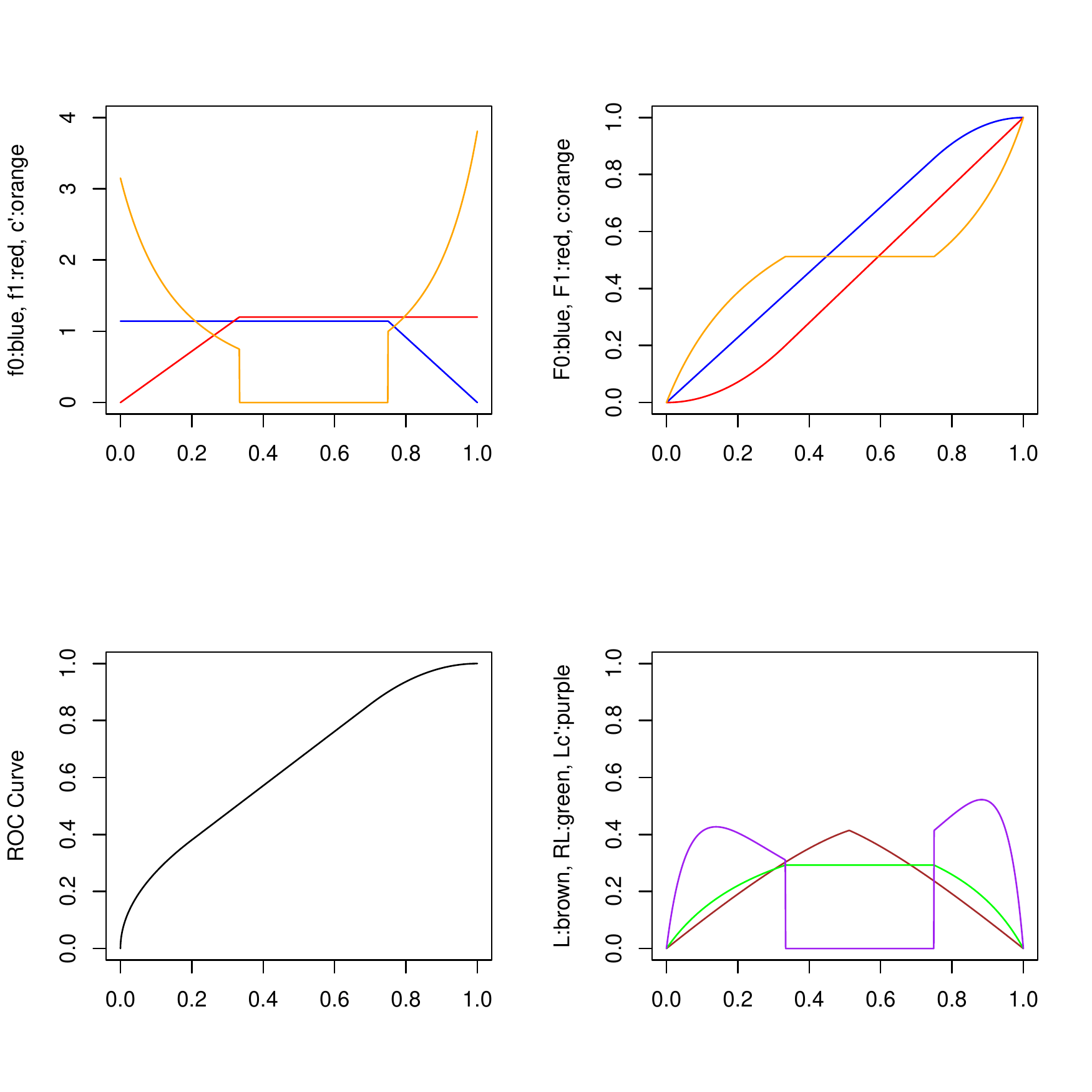} 
\caption{Here we have a convex model, which is not strictly convex.
$f_0$ is defined as a triangular distributions upto 1/3 and then constant with height $8/7$, and a $f_1$ is defined constant with height $6/5$ until 3/4 and then triangular. All this leads to three different segments, one being straight.
We see that $c(T)$ (in orange) in this case is non-decreasing. The ROC curve is shown at the bottom left plot, which is convex (but not strictly, since it has a straight segment). The interesting plot is on the bottom right.
Here we also see that the overall loss (which is $0.2266$ in this case) can be calculated in three different ways. The original curve (in brown) is given by Equation (\ref{eq:Loriginal2}).
A different curve (in purple) is given by $\acute{\Lambda}$ in Equation (\ref{eq:LcTbijective}). Again, the equivalence of the area under these two curves when the model is convex is what Lemma \ref{lemma:prime} shows, because in this case $\dot{\Lambda} = 0$.
The curve for the refinement loss is shown in green, with also the same area. 
One important thing is what happens to the brown curve against the purple curve. While the three curves treat the segments which are straight on the convex hull, they do this in a very different way. The brown curve (original loss expression) eliminates them on the \xaxis, the purple curve (loss expression using $c(T)$) eliminates them on the \yaxis and the green curve ($\rl$) treats them by using a constant value.
The intervals are $I_\sigma = \{(0, 0.51), (0.51, 0.51), (0.51, 1) \}$, $I_\tau = \{(0, 0.33), (0.33, 0.75), (0.75, 1) \}$, where the first interval is bijective, the second is constant and the third is bijective. }
\label{fig:nonstrictlyconvex}
\end{figure}

As already shown before, figure \ref{fig:perfectlycalibrated} is a case of a model is perfectly calibrated, where the three curves match. Figure \ref{fig:noncalibrated} shows a model which has exactly the same ranking (and ROC curve) than in figure \ref{fig:perfectlycalibrated} but only two of the three curves match pointwise (the three still matching in total area). Figure \ref{fig:noncalibrated2} shows another model which has exactly the same ranking (and ROC curve) than in figure \ref{fig:perfectlycalibrated}, but again the model is not perfectly calibrated. We see now, none of the three curves match pointwise (the three still matching in total area). In these three cases we only have one bijective interval.

In  figure \ref{fig:severalconcavities}, we show the picture for a non-convex model.


\begin{figure}
\centering
\includegraphics[width=0.9\textwidth]{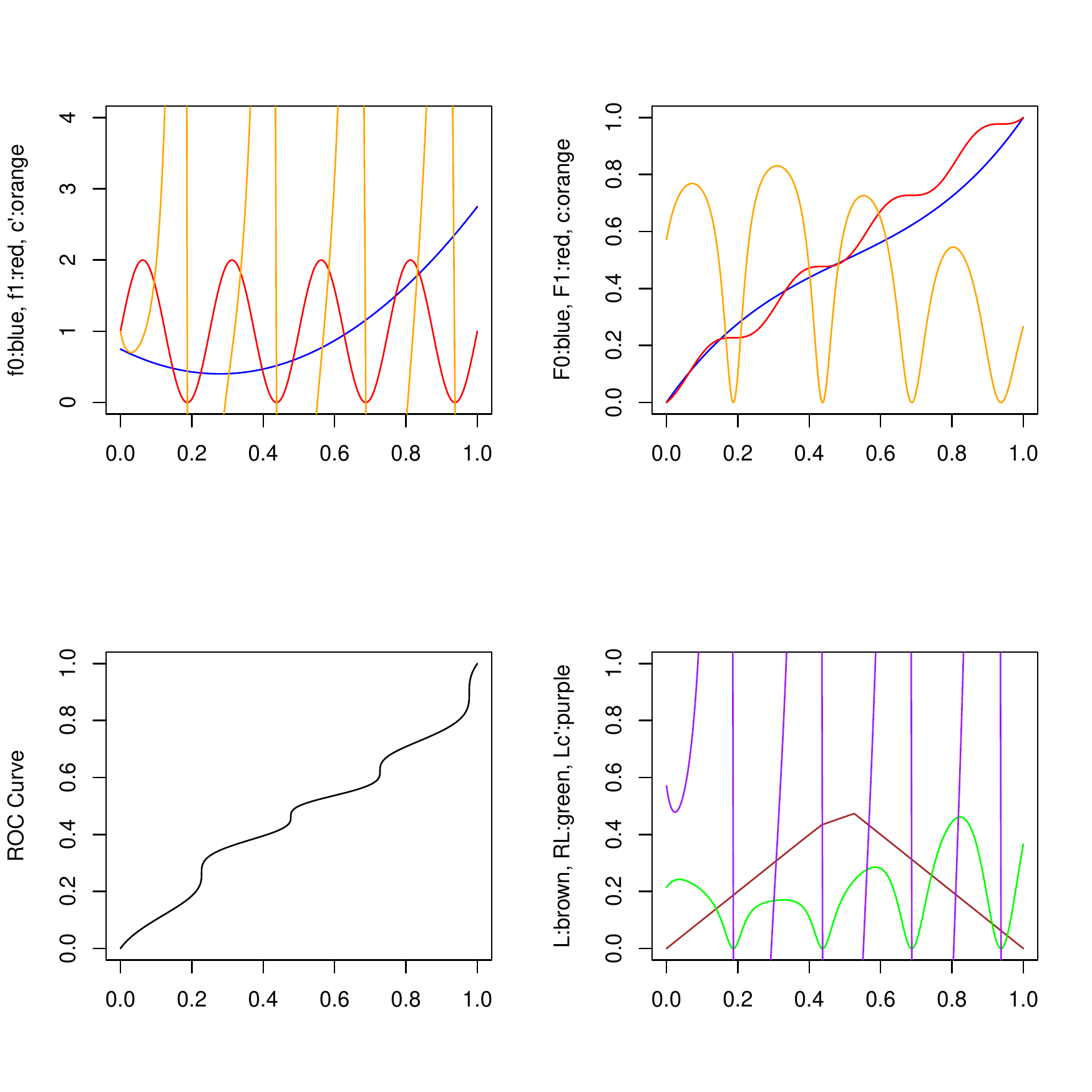} 
\caption{Here we have a non-convex model.
$f_0$ is a parabola and $f_1$ is a sinoidal function. 
We see that $c(T)$ (in orange) in this case is not non-decreasing. In fact, it is not a cumulative distribution and $c'(T)$ is clearly not a density function. The ROC curve is shown at the bottom left plot, which has many concavities.
Note that we cannot use $c(T)$ to determine the concavities in the ROC curve, and we cannot use $c(T)$ to calculate the convex hull either (locally).
 The plot on the bottom right shows that the three curves do not match in areas.
Because of non-convexity we cannot give a series of intervals. }
\label{fig:severalconcavities}
\end{figure}

Figure \ref{fig:diagonal} shows a diagonal classifier represented by $f_0$ and $f_1$ being the density functions for uniform distribution. The result is a convex model.


\begin{figure}
\centering
\includegraphics[width=0.9\textwidth]{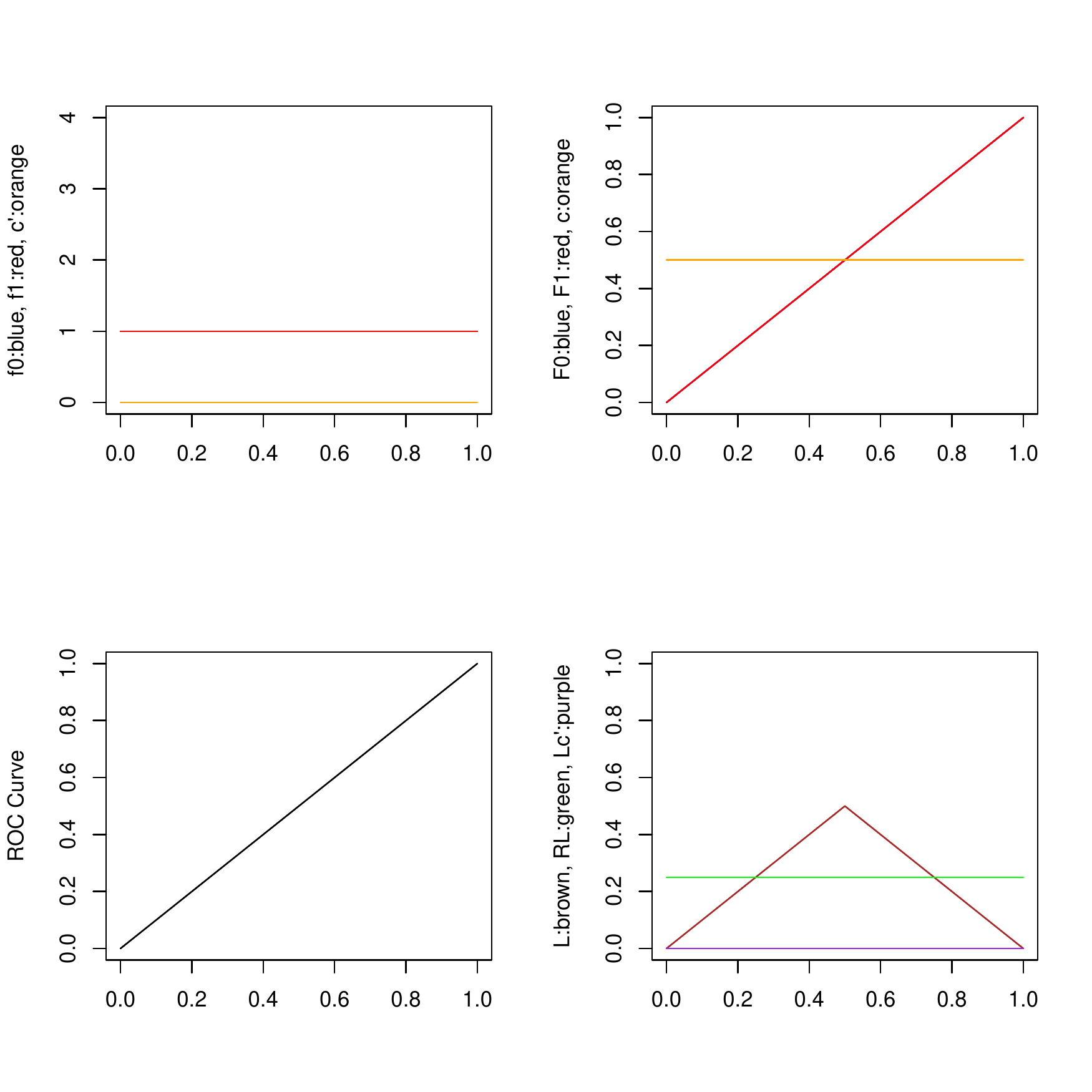} 
\caption{A diagonal model where $f_0=f_1=1$ everywhere. This model is convex, but it shows that the purple curve ($\acute{\Lambda}$ in Equation (\ref{eq:LcTbijective})) is 0. In this case, however, $\bar{\Lambda}$ in Equation (\ref{eq:LcTbijective}) is equal to $\pi_0 [c^2]^{\pi_1}_0 + \pi_1 [2c-c^2]^{1}_{\pi_1}$, which leads to $\pi_0\pi_1$, which equals the \rl for a diagonal classifier. In the figure, since $\pi_0=\pi_1$ we have that the area is 0.25.
The intervals are $I_\sigma = \{(0, 0.5), (0.5, 0.5), (0.5, 1) \}$, $I_\tau = \{(0, 0), (0, 1), (1, 1) \}$, where the first interval is singular, the second is constant and the third is singular.}		
\label{fig:diagonal}
\end{figure}

Figure \ref{fig:diagonal2} shows another diagonal classifier, which is also convex. Here thera are some regions where the density functions are zero, and all the mass is concentrated around 0.5.


\begin{figure}
\centering
\includegraphics[width=0.9\textwidth]{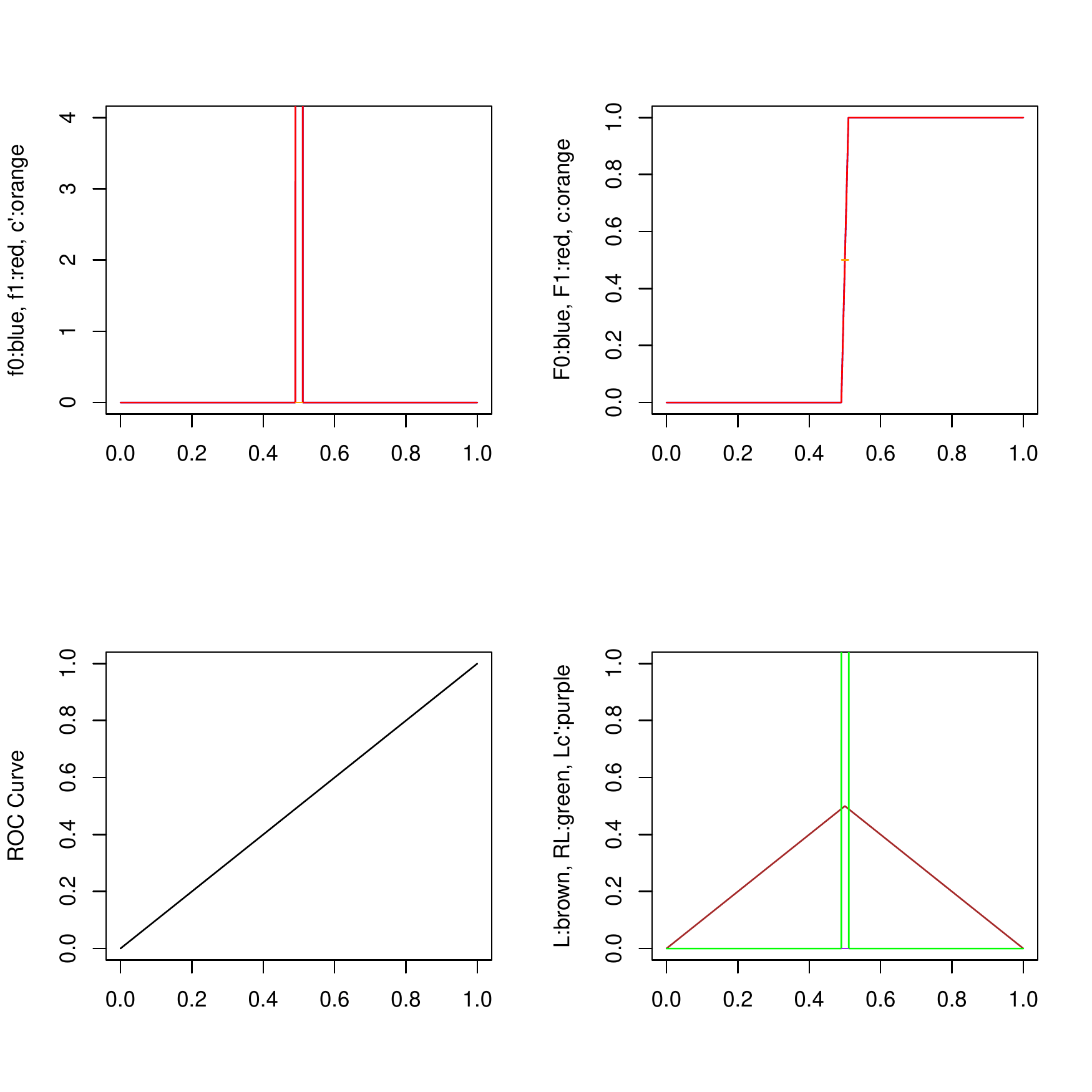} 
\caption{A diagonal model where $f_0=f_1=0$ almost everywhere, except in the interval $[0.495,0.505]$ where we have $f_0=f_1=1/0.01$. This model is convex, but it shows that the purple curve ($\acute{\Lambda}$ in Equation (\ref{eq:LcTbijective})) is 0. In this case, however, $\bar{\Lambda}$ in Equation (\ref{eq:LcTbijective}) is equal to $\pi_0 [c^2]^{\pi_1}_0 + \pi_1 [2c-c^2]^{1}_{\pi_1}$, which leads to $\pi_0\pi_1$, which equals the \rl for a diagonal classifier. In the figure, since $\pi_0=\pi_1$ we have that the area is 0.25. However, the \rl is concentrated in the interval $[0.495,0.505]$.
The intervals are $I_\sigma = \{(0, 0.5), (0.5, 0.5), (0.5, 1) \}$, $I_\tau = \{(0, 0.495), (0.495, 0.505), (0.505,1) \}$, where the first interval is singular, the second is constant and the third is singular.}		
\label{fig:diagonal2}
\end{figure}

Figures \ref{fig:notcumulative-simple} and \ref{fig:notcumulative2} show cases where the limits of the integral are not between 0 and 1, and also because there are some `singular' intervals.


\begin{figure}
\centering
\includegraphics[width=0.9\textwidth]{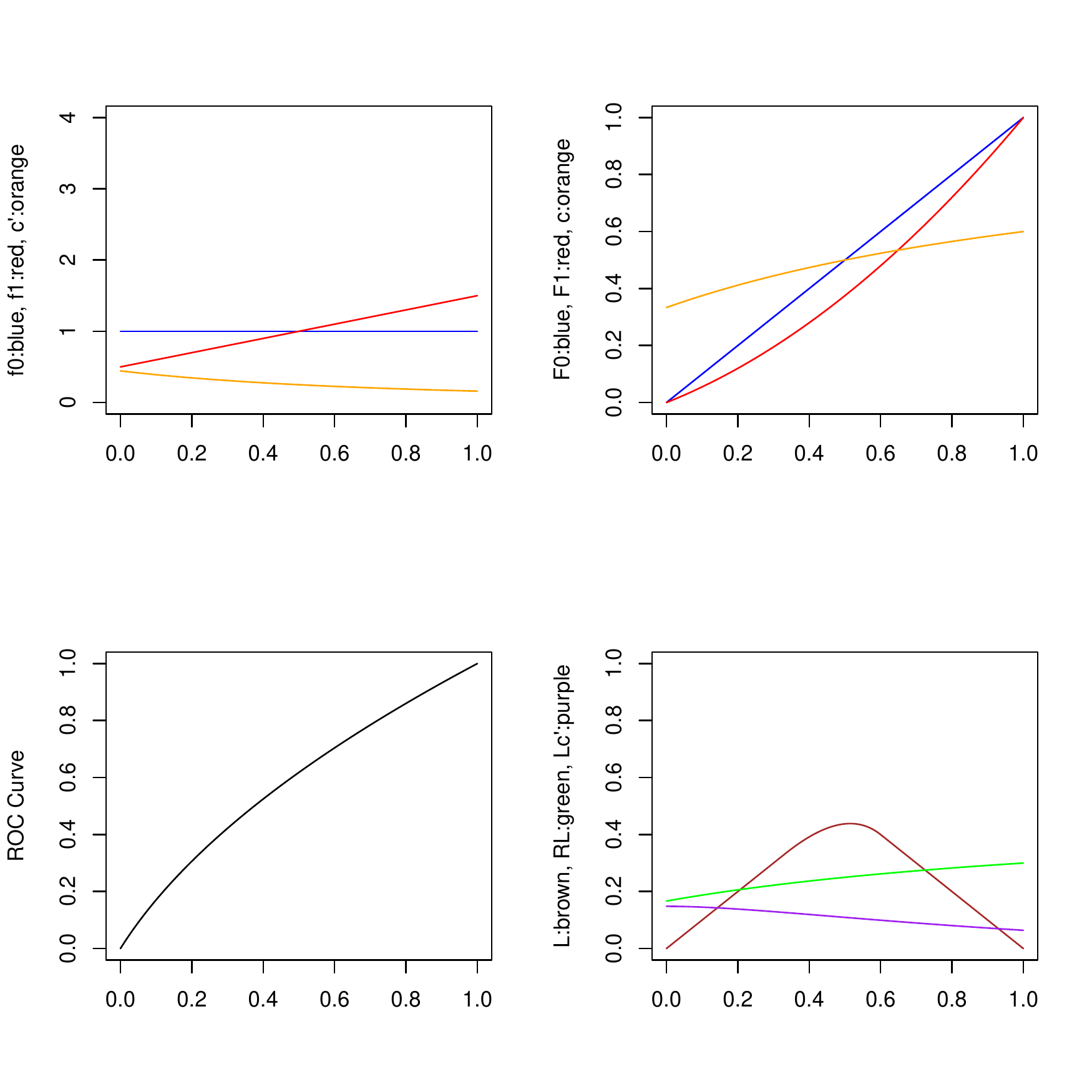} 
\caption{A model with a constant $f_0=1$ and a triangular $f_1$ which does not start at 0. This model is strictly convex, but it shows that the area of the purple curve does not match the other two curves. In fact, $c(T)$ in this case is not a cumulative distribution, since it does not go from 0 to 1.
The explanation here is found because $\dot{\Lambda} \neq 0$ and also the limits of integration, so they are $c^{-1}(0) \neq 0$ and $c^{-1}(1) \neq 1$.
The intervals are $I_\sigma = \{(0, 0.33), (0.33, 0.6), (0.6, 1) \}$, $I_\tau = \{(0, 0), (0, 1), (1, 1) \}$, where the first interval is singular, the second is bijective and the third is singular.}
\label{fig:notcumulative-simple}
\end{figure}

\begin{figure}
\centering
\includegraphics[width=0.9\textwidth]{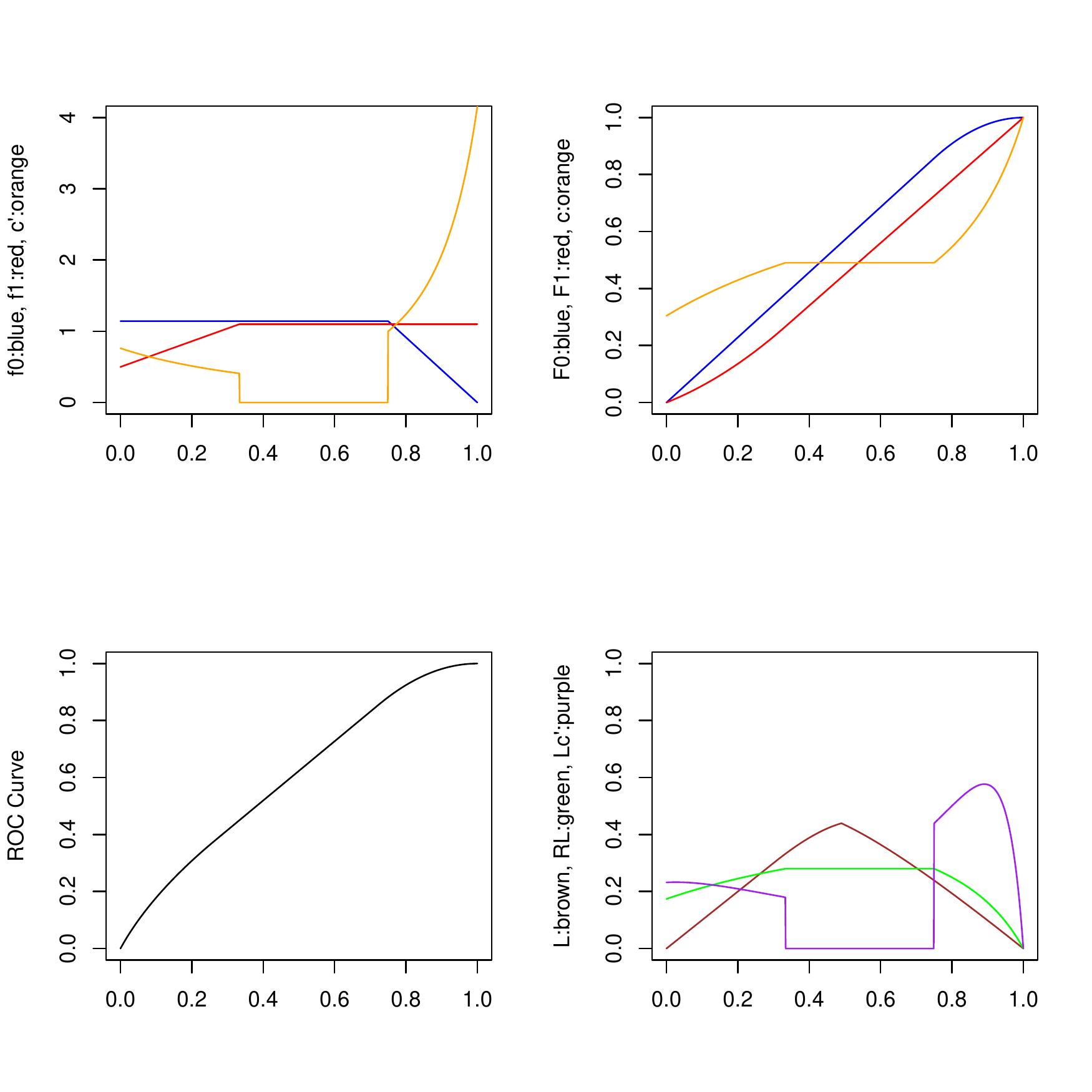} 
\caption{Another example where the area of the purple curve does not match the other two curves. In fact, again, $c(T)$ in this case is not a cumulative distribution, since it does not go from 0 to 1.
The explanation here is again found because $\dot{\Lambda} \neq 0$ and also the limits of integration, so they are $c^{-1}(0) \neq 0$ and $c^{-1}(1) \neq 1$.
The intervals are $I_\sigma = \{(0, 0.3), (0.3, 0.51), (0.51, 0.51), (0.51, 1) \}$, $I_\tau = \{(0, 0), (0, 0.33), (0.33, 0.75), (0.75, 1) \}$, where the first interval is singular, the second is bijective, the third is constant and the fourth is bijective.
}
\label{fig:notcumulative2}
\end{figure}

Finally, Figure \ref{fig:discontinuous} shows a case where the model is discontinuous, and they start with $c(0) \neq 0$ and $c(1)\neq 1$.


\begin{figure}
\centering
\includegraphics[width=0.9\textwidth]{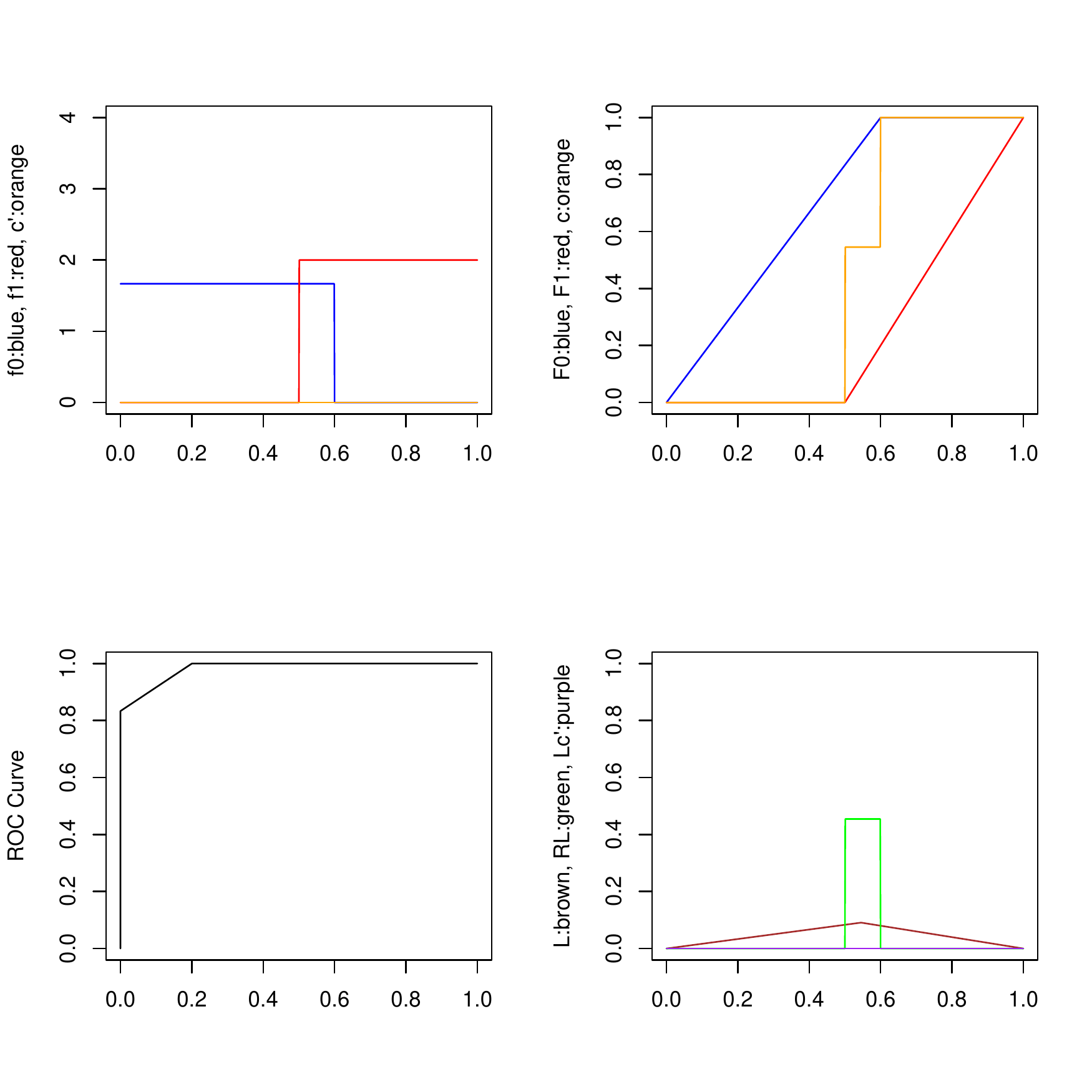} 
\caption{A model with a constant $f_0=1/0.6$ from $0$ to $0.6$ and a constant $f_1=2$ from $0$ to $0.5$. This model is convex, but it shows that the area of the purple curve does not match the other two curves. We have that $c(T)$ in this case is a cumulative distribution, since it goes from 0 to 1. However, we have discontinuities, and then we have some mass which is not included by the purple curve, which appears in $\dot{\Lambda} \neq 0$.
The intervals are $I_\sigma = \{(0, 0), (0, 0.55), (0.55, 0.55), (0.55, 1), (1,1) \}$, $I_\tau = \{(0, 0.5), (0.5, 0.5), (0.5, 0.6), (0.6, 0.6), (0.6, 1) \}$, where the first interval is constant, the second is singular, the third is constant, the fourth is singular and the fifth is constant.
}
\label{fig:discontinuous}
\end{figure}

\subsection{$c(T)$ is idempotent}

Now we work with the transformation $T \mapsto s=c(T)$. The resulting model using this transformation will be denoted by $m^{(c)}$.
We will use $H_0(s)$ and $H_1(s)$ for the cumulative distributions, which are defined as follows.
Since $s=c(T)$ by definition we have that $F_0(T) = H_0(c(T)) = H_0(s)$ and similarly $F_1(T) = H_1(c(T)) = H_1(s)$.

For the intervals $]\tau_i,\tau_{i+1}[$ in $\acute{I}_\tau$, we have $c(T)$ is strictly convex we just use $c^{-1}(s)$ to derive $H_0$ and $H_1$. This may imply discontinuities at $\tau_i$ or $\tau_{i+1}$ for those values of $s$ for which constant intervals have been mapped, namely $\sigma_i$ and $\sigma_{i+1}$.
So, we need to define the density functions as follows. For the bijective intervals we just use $h_0(s)ds = f_0(T)dT$ and $h_1(s)ds = f_1(T)dT$ as a shorthand for a change of variable, and we can clear $h_0$ and $h_1$ using $c^{-1}(s)$. We do that using open intervals $]\tau_i,\tau_{i+1}[$  in $T$.
These correspond to $]c(\tau_i), c(\tau_{i+1})[$ = $]\sigma_i, \sigma_{i+1}[$.

The constant intervals are $[\tau_{i}, \tau_{i+1}]$ in $\bar{I}_\tau$. 
There is probability mass 
for every constant interval $[\tau_{i}, \tau_{i+1}]$ mapping to a point $s_i = c(\tau_i) = c(\tau_{i+1}) = \sigma_i = \sigma_{i+1}$, as follows:
%
%
\begin{eqnarray}
{[H_0(T)]}_{\sigma_{i}}^{\sigma_{i+1}} = \int_{\tau_{i}}^{\tau_{i+1}} f_0(T) dt = [F_0(T)]_{\tau_{i}}^{\tau_{i+1}} = F_0(\tau_{i+1})-F_0(\tau_{i}) \label{eq:h0} \\
{[H_1(T)]}_{\sigma_{i}}^{\sigma_{i+1}} = \int_{\tau_{i}}^{\tau_{i+1}} f_1(T) dt = [F_1(T)]_{\tau_{i}}^{\tau_{i+1}} = F_1(\tau_{i+1})-F_1(\tau_{i}) \label{eq:h1}  
\end{eqnarray}
%


Finally, we just define $h_0(s)=h_1(s)=0$ for those $s \in [\sigma_{i}, \sigma_{i+1}] \in \dot{I}_\sigma$, since for the singular intervals there is only one point $\tau_{i}$ and the mass to share is 0.

This makes $m^{(c)}$ well-defined for convex models (not necessarily continuous and strictly convex).

\begin{lemma}\label{idempotent}
For model $m^{(c)}$ we have that, for the non-singular intervals, $c_{m^{(c)}}(s)=\frac{\pi_1 h_1(s)}{\pi_0 h_0(s) + \pi_1 h_1(s)}$ is idempotent, i.e.:
\begin{equation}
c_{m^{(c)}}(s)  =  s
\end{equation}
\end{lemma}

\begin{proof}
For the bijective (strictly convex) intervals $]\tau_i, \tau_{i+1}[$ mapped into $]c(\tau_i), c(\tau_{i+1})[$, i.e. $]\sigma_i, \sigma_{i+1}[$: 
\begin{eqnarray*}
c_{m^{(c)}}(s) & = & \frac{\pi_1 h_1(s)}{\pi_0 h_0(s) + \pi_1 h_1(s)} = \frac{\pi_1 h_1(s)ds}{\pi_0 h_0(s)ds + \pi_1 h_1(s)ds}  \\
               & = & \frac{\pi_1 f_1(T)dT}{\pi_0 f_0(T)dT + \pi_1 f_1(T)dT} = \frac{\pi_1 f_1(T)}{\pi_0 f_0(T) + \pi_1 f_1(T)}  = c(T) = s  
\end{eqnarray*}
For the points $s_i= c(\tau_{i}) = c(\tau_{i+1})$ corresponding to constant intervals, we have that using (\ref{eq:h0}) and (\ref{eq:h1}): 
\begin{eqnarray*}
c_{m^{(c)}}(s_i) & = & \frac{\pi_1 h_1(s_i)}{\pi_0 h_0(s_i) + \pi_1 h_1(s_i)} = \frac{\pi_1 [F_1(T)]_{\tau_{i}}^{\tau_{i+1}}}{\pi_0 [F_0(T)]_{\tau_{i+1}}^{a_{i+1}} + \pi_1 [F_1(T)]_{\tau_{i}}^{\tau_{i+1}}} 
\end{eqnarray*}

Since $c(T)$ is constant in the interval $]\tau_{i}, \tau_{i+1}[$, we have:
\begin{eqnarray*}
c_{m^{(c)}}(s_i) & = & \frac{\pi_1 f_1(T)}{\pi_0 f_0(T) + \pi_1 f_1(T)}  = c(T) = s_i  
\end{eqnarray*}
%
%
%
\end{proof}

\subsection{Main result}

Finally, we are ready to prove the theorem treating the three kinds of intervals.

\begin{theorem}\label{th:planA} (Theorem \ref{thm:RL} in the paper)
For every convex model $m$, we have that:
\begin{equation*}
\LcostoU(m)  =  \rl(m)
\end{equation*}
\end{theorem}

\begin{proof}

Let us start from Lemma \ref{lemma:prime}:

\begin{equation*}
\LcostoU(m)  = \acute{\Lambda}(m) + \dot{\Lambda}(m) 
\end{equation*}

\noindent working with Equation (\ref{eq:LcTbijective}) first for the bijective intervals:

\begin{equation*}
\acute{\Lambda}(m) = \sum_{]\tau_i,\tau_{i+1}[ \in \acute{I}_\tau} \int_{\tau_i}^{\tau_{i+1}} 2c(T) \pi_0(1-F_0(T)) + 2(1-c(T))\pi_1 F_1(T) \} c'(T) dT  
\end{equation*}

Since this only includes the bijective intervals, we can use the correspondence between the $H$ and the $F$, and making the change $s=c(T)$.
\begin{eqnarray*}
\acute{\Lambda}(m) & = & \sum_{]\tau_i,\tau_{i+1}[ \in \acute{I}_\tau} \int_{\tau_i}^{\tau_{i+1}}  2c(T) \pi_0(1-H_0(c(T))) + 2(1-c(T))\pi_1 H_1(c(T)) \} c'(T) dT \\
                  & =  & \sum_{]c(\tau_i),c(\tau_{i+1})[ \in \acute{I}_\sigma} \int_{c(\tau_i)}^{c(\tau_{i+1})} 2s \pi_0(1-H_0(s)) + 2(1-s)\pi_1 H_1(s) \} ds \\
                  & =  & \sum_{]\sigma_i,\sigma_{i+1}[ \in \acute{I}_\sigma} \int_{\sigma_i}^{\sigma_{i+1}} 2s \pi_0(1-H_0(s)) + 2(1-s)\pi_1 H_1(s) \} ds 
\end{eqnarray*}

\noindent and now working with Equation (\ref{eq:LcTsingular1}) for the singular intervals and also using the correspondence between the $H$ and the $F$:
\begin{eqnarray*}
 \dot{\Lambda}(m) & = & \sum_{]\sigma_i,\sigma_{i+1}[ \in \dot{I}_\sigma} \int_{\sigma_i}^{\sigma_{i+1}} 2c \pi_0(1-F_0(\tau_i)) + 2(1-c)\pi_1 F_1(\tau_i) \} dc \\
                   & = & \sum_{]\sigma_i,\sigma_{i+1}[ \in \dot{I}_\sigma} \int_{\sigma_i}^{\sigma_{i+1}} 2c \pi_0(1-H_0(c(\tau_i))) + 2(1-c)\pi_1 H_1(c(\tau_i)) \} dc \\
                   & = & \sum_{]\sigma_i,\sigma_{i+1}[ \in \dot{I}_\sigma} \int_{\sigma_i}^{\sigma_{i+1}} 2s \pi_0(1-H_0(\sigma_i)) + 2(1-s)\pi_1 H_1(\sigma_i) \} ds 
\end{eqnarray*}
The last step also uses the renaming of the variable. But since $h_0(s)=h_1(s)=0$ for the singular intervals, we have that $H_0(s)$ and $H_1(s)$ are constant in these intervals, so this can be rewritten as:
\begin{eqnarray*}
 \dot{\Lambda}(m) & = & \sum_{]\sigma_i,\sigma_{i+1}[ \in \dot{I}_\sigma} \int_{\sigma_i}^{\sigma_{i+1}} 2s \pi_0(1-H_0(s)) + 2(1-s)\pi_1 H_1(s) \} ds 
\end{eqnarray*}
Putting $\acute{\Lambda}(m)$ and $\dot{\Lambda}(m)$ together, because the constant intervals ($\bar{I}_\sigma$) have length 0 (and loss 0), we have:
\begin{eqnarray*}
\LcostoU(m)   & = & \sum_{]\sigma_i,\sigma_{i+1}[ \in I_\sigma} \int_{\sigma_i}^{\sigma_{i+1}} 2s \pi_0(1-H_0(s)) + 2(1-s)\pi_1 H_1(s) \} ds 
\end{eqnarray*}
We can join the integrals into a single one, even though the whole integral has to be calculated by intervals if it is discontinuous:
\begin{eqnarray*}
\LcostoU(m) & =  & \int_{\sigma_0}^{\sigma_{n}} \{  2s \pi_0(1-H_0(s)) + 2(1-s)\pi_1 H_1(s) \} ds \\
            & =  & \int^{1}_{0} \{ 2s \pi_0(1-H_0(s)) + 2(1-s)\pi_1 H_1(s) \} ds 
\end{eqnarray*}

By Theorem \ref{thm:LcostpUequalsBS2} (Equation (\ref{eq:bsloss})) in the paper (and also because this theorem holds pointwise)  we have that the last expression equals the Brier score, so this leads to:
\begin{equation}
\LcostoU(m)  =  \bs(m^{(c)})
\end{equation}

\noindent And now we have that using Definition \ref{def:BS} for the \bs:

\begin{equation}
\bs(m^{(c)}) = \int^{1}_{0} \{  \pi_0s^2h_0(s) + \pi_1(1-s)^2 h_1(s) \} ds
\end{equation}

\noindent This is 0 when $h_0(s)=h_1(s)=0$, so we can ignore the singular intervals for the rest of the proof.
The calibration loss for model $m^{(c)}$ can be expanded as follows, and using Lemma \ref{idempotent} (which is applicable except for non-singular intervals) we have:
\begin{eqnarray*}\label{eq:CL}
\cl(m^{(c)}) & = &\int_{0}^{1} \left(	 s - \frac{\pi_1 h_1(s)}{\pi_0 h_0(s) + \pi_1 h_1(s)} \right)^{2} \left( \pi_0 h_0(s) + \pi_1 h_1(s) \right)  ds \\
& = &\int_{0}^{1} \left(	 s - s \right)^{2} \left( \pi_0 h_0(s) + \pi_1 h_1(s) \right)  ds	= 0
\end{eqnarray*}
So, we have that:
\begin{equation}
\LcostoU(m)  =  \rl(m^{(c)})  \label{eq:ap54}
\end{equation}
And now we need to work with $\rl$: 
\begin{eqnarray*}
  \rl(m^{(c)})	 & =  & \int_{0}^{1}  \frac{\pi_1 h_1(s) \pi_0 h_0(s)} {\pi_0 h_0(s) + \pi_1 h_1(s)} ds = \int_{0}^{1}   \pi_0 h_0(s) {\frac{\pi_1 h_1(s)}{\pi_0 h_0(s) + \pi_1 h_1(s)}} ds \\
               	 & =  & \int_{0}^{1}   \pi_0 h_0(s) c_{m^{(c)}}(s)  ds = \int_{0}^{1}   \pi_0 h_0(s) s  ds 							
\end{eqnarray*}
The last step applies Lemma \ref{idempotent} again.

We now need to treat the bijective and the constant intervals separately, otherwise the integral cannot be calculated when $h_0$ and $h_1$ are discontinuous.
\begin{eqnarray*}
  \rl(m^{(c)})	 & =  & \sum_{]\sigma_i,\sigma_{i+1}[ \in \acute{I}_\sigma } \int^{\sigma_{i+1}}_{\sigma_i}   \pi_0 h_0(s) s  ds  +  \sum_{]\sigma_i,\sigma_{i+1}[ \in \bar{I}_\sigma }   \pi_0 h_0(\sigma_{i}) \sigma_i  							
\end{eqnarray*}
We apply the variable change $s=c(T)$ for the expression on the left:
\begin{eqnarray*}
\sum_{]c(\tau_i),c(\tau_{i+1})[ \in \acute{I}_\sigma } \int^{c(\tau_{i+1})}_{c(\tau_i)}   \pi_0 h_0(s) s  ds          & = &  \sum_{]\tau_i,\tau_{i+1}[ \in \acute{I}_\tau} \int^{\tau_{i+1}}_{\tau_i}  \pi_0 h_0(c(T)) c(T) \frac{d c(T)}{dT} dT  \\ 
  	       & = & \sum_{]\tau_i,\tau_{i+1}[   \in \acute{I}_\tau } \int^{\tau_{i+1}}_{\tau_i}   \pi_0 h_0(c(T)) \frac{\pi_1f_1(T)}{\pi_1f_1(T)+ \pi_0f_0(T)} \frac{d c(T)}{dT} dT    \\
			     & = &  \sum_{]\tau_i,\tau_{i+1}[  \in \acute{I}_\tau }\int^{\tau_{i+1}}_{\tau_i}  \pi_0 h_0(c(T))\frac{d c(T)}{dT} \frac{\pi_1f_1(T)}{\pi_1f_1(T)+ \pi_0f_0(T)}  dT   \\
           & = &  \sum_{]\tau_i,\tau_{i+1}[ \in \acute{I}_\tau} \int^{\tau_{i+1}}_{\tau_i}   \pi_0 f_0(T) \frac{\pi_1f_1(T)}{\pi_1f_1(T)+ \pi_0f_0(T)}  dT   	
\end{eqnarray*}
We now  work with the expression on the right using Equation (\ref{eq:h0}):
\begin{eqnarray}
\sum_{]c(\tau_i),c(\tau_{i+1})[ \in \bar{I}_\sigma }    \pi_0 h_0(c(\tau_{i})) c(\tau_{i}) 	 & = & \sum_{]\tau_i,\tau_{i+1}[  \in \bar{I}_\tau }    \pi_0   [F_0(T)]_{\tau_{i}}^{\tau_{i+1}}  c(\tau_{i}) \label{eq:constant} 
			     =     \sum_{]\tau_i,\tau_{i+1}[  \in \bar{I}_\tau }   \pi_0   \int_{\tau_{i}}^{\tau_{i+1}} f_0(T) dT c(\tau_{i}) \nonumber \\					
           & = &     \sum_{]\tau_i,\tau_{i+1}[  \in \bar{I}_\tau }    \int_{\tau_{i}}^{\tau_{i+1}} \pi_0  f_0(T) c(T) dT   \nonumber \\
           & = &     \sum_{]\tau_i,\tau_{i+1}[  \in \bar{I}_\tau }     \int_{\tau_{i}}^{\tau_{i+1}} \pi_0  f_0(T) \frac{\pi_1f_1(T)}{\pi_1f_1(T)+ \pi_0f_0(T)} dT   	\nonumber
\end{eqnarray}
The change from $c(\tau_{i})$ to $c(T)$ inside the integral can be performed since $c(T)$ is constant, because here we are working with the constant intervals.

Putting everything together again:				
\begin{eqnarray*}			
	\rl(m^{(c)})          & = &   \sum_{]\tau_i,\tau_{i+1}[ \in \acute{I}_\tau } \int^{\tau_{i+1}}_{\tau_i}   \pi_0 f_0(T) \frac{\pi_1f_1(T)}{\pi_1f_1(T)+ \pi_0f_0(T)}  dT +    \sum_{]\tau_i,\tau_{i+1}[  \in \bar{I}_\tau }    \int_{\tau_{i}}^{\tau_{i+1}} \pi_0  f_0(T) \frac{\pi_1f_1(T)}{\pi_1f_1(T)+ \pi_0f_0(T)} dT   \\
		       & = &  \int_{\tau_0}^{{\tau_n}}   \frac{\pi_0 f_0(T)\pi_1f_1(T)}{\pi_1f_1(T)+ \pi_0f_0(T)}  dT  =
		        \int_{-\infty}^{\infty}    \frac{\pi_0 f_0(T)\pi_1f_1(T)}{\pi_1f_1(T)+ \pi_0f_0(T)}  dT = \rl(m) 		
\end{eqnarray*}
This and Equation (\ref{eq:ap54}) complete the proof.
\end{proof}

We have a nice example in Figure \ref{fig:nonstrictlyconvex}, where it shows that the refinement loss is not 0 in the constant segments. In fact, the area in the constant segments can be calculated from Equation (\ref{eq:constant}). In this case, $c(1/3) = c(3/4) = 0.512$, and we have that $F_0(1/3) = 0.38$ and $F_0(3/4) = 0.85$. Since $\pi_1=0.5$ we have that the area of this constant segment for the \rl is $0.512 \cdot 0.5 \cdot (0.857- 0.381) = 0.1220$, which equals the area as calculated using the width times the height (i.e., $(3/4-1/3) \cdot 0.2927 = 0.1220$). The height is given by \rl in this interval, which is $\frac{\pi_0 f_0(T)\pi_1f_1(T)}{\pi_1f_1(T)+ \pi_0f_0(T)} = \frac{0.5 \cdot 1.143 \cdot 0.5 \cdot 1.2}{0.5 \cdot 1.143 + 0.5 \cdot 1.2} = 0.2927$.

 \end{document}